\pdfoutput=1
\documentclass{article}

\PassOptionsToPackage{dvipsnames}{xcolor}

\usepackage{eof_preprint,times}

\usepackage{graphicx}
\usepackage{caption}
\usepackage{subcaption}

\usepackage[utf8]{inputenc}
\usepackage[T1]{fontenc}
\usepackage[hidelinks]{hyperref}
\usepackage{url}
\usepackage{booktabs}
\usepackage{microtype}

\usepackage{amsmath}
\usepackage{amssymb}
\usepackage{amsfonts}
\usepackage{mathtools}
\usepackage{amsthm}
\usepackage{accents}
\usepackage{dsfont}

\usepackage{nicefrac}
\usepackage{diagbox}
\usepackage{enumitem}
\usepackage{wrapfig}

\usepackage{listings}
\newlength{\dhatheight}

\newlength{\dwidehatheight}

\usepackage{amsmath,amsfonts,bm}

\def\eqref#1{equation~\ref{#1}}

\def\1{\bm{1}}

\def\eps{{\epsilon}}

\DeclareMathAlphabet{\mathsfit}{\encodingdefault}{\sfdefault}{m}{sl}
\SetMathAlphabet{\mathsfit}{bold}{\encodingdefault}{\sfdefault}{bx}{n}

\newcommand{\E}{\mathbb{E}}

\newcommand{\R}{\mathbb{R}}

\newcommand{\Var}{\mathrm{Var}}

\renewcommand{\eps}{{\boldsymbol{\epsilon}}}

\newcommand{\ML}{{\text{ML}}}

\newcommand{\rec}{{\text{rec}}}

\newcommand{\ND}{\text{ND}}

\newcommand{\noisesig}{\sigma_{\eps}}\newcommand{\xnoisy}{\x_\eps}

\newcommand{\e}{{\boldsymbol{e}}}

\newcommand{\A}{\boldsymbol{A}}

\newcommand{\I}{\boldsymbol{I}}

\newcommand{\set}[1]{{\mathbb{#1}}}
\newcommand{\nset}[1]{{\overline{\mathbb{#1}}}}

\DeclareMathOperator{\Expectation}{\mathbb{E}}
\newcommand{\Expt}[2]{\Expectation\displaylimits_{#1}\left[#2\right]}

\NewDocumentCommand\pz{O{} O{}}{{p_{\set{#1}\nset{#2}}}}
\NewDocumentCommand\z{O{} O{}}{{\boldsymbol{z}_{\set{#1}\nset{#2}}}}
\NewDocumentCommand\Z{O{} O{}}{{\boldsymbol{Z}_{\set{#1}\nset{#2}}}}

\RenewDocumentCommand\u{O{} O{}}{{\boldsymbol{u}_{\set{#1}\nset{#2}}}}
\NewDocumentCommand\U{O{} O{}}{{\boldsymbol{U}_{\set{#1}\nset{#2}}}}

\NewDocumentCommand\f{O{} O{}}{{\boldsymbol{f}_{\set{#1}\nset{#2}}}}
\NewDocumentCommand\g{O{} O{}}{{\boldsymbol{g}_{\set{#1}\nset{#2}}}}
\NewDocumentCommand\x{O{} O{}}{{\boldsymbol{x}_{\set{#1}\nset{#2}}}}
\NewDocumentCommand\X{O{} O{}}{{\boldsymbol{X}_{\set{#1}\nset{#2}}}}
\NewDocumentCommand\J{O{} O{}}{{\boldsymbol{J}_{\set{#1}\nset{#2}}}}
\NewDocumentCommand\p{O{} O{}}{{{p}_{\set{#1}\nset{#2}}}}
\NewDocumentCommand\q{O{} O{}}{{{q}_{\set{#1}\nset{#2}}}}
\NewDocumentCommand\mix{O{} O{}}{{{C}_{\set{#1}\nset{#2}}}}

\NewDocumentCommand\zfl{O{} O{}}{{\boldsymbol{z}_{\set{#1}\nset{#2}}}}
\NewDocumentCommand\Zfl{O{} O{}}{{\boldsymbol{Z}_{\set{#1}\nset{#2}}}}

\NewDocumentCommand\fgt{O{} O{}}{{\boldsymbol{f}^*_{\set{#1}\nset{#2}}}}
\NewDocumentCommand\ggt{O{} O{}}{{\boldsymbol{g}^*_{\set{#1}\nset{#2}}}}
\NewDocumentCommand\xgt{O{} O{}}{{\boldsymbol{x}^*_{\set{#1}\nset{#2}}}}
\NewDocumentCommand\Xgt{O{} O{}}{{\boldsymbol{X}^*_{\set{#1}\nset{#2}}}}
\NewDocumentCommand\Jgt{O{} O{}}{{\boldsymbol{J}^*_{\set{#1}\nset{#2}}}}
\NewDocumentCommand\pgt{O{} O{}}{{p^*_{\set{#1}\nset{#2}}}}
\NewDocumentCommand\Mangt{O{} O{}}{{\mathcal{M}^*_{\set{#1}\nset{#2}}}}

\NewDocumentCommand\zgt{O{} O{}}{{\boldsymbol{z}^*_{\set{#1}\nset{#2}}}}
\NewDocumentCommand\Zgt{O{} O{}}{{\boldsymbol{Z}^*_{\set{#1}\nset{#2}}}}

\NewDocumentCommand\ffl{O{} O{}}{{\boldsymbol{f}^{\theta}_{\set{#1}\nset{#2}}}}
\NewDocumentCommand\gfl{O{} O{}}{{\boldsymbol{g}^{\theta}_{\set{#1}\nset{#2}}}}
\NewDocumentCommand\xfl{O{} O{}}{{\boldsymbol{x}^{\theta}_{\set{#1}\nset{#2}}}}
\NewDocumentCommand\Xfl{O{} O{}}{{\boldsymbol{X}^{\theta}_{\set{#1}\nset{#2}}}}
\NewDocumentCommand\Jfl{O{} O{}}{{\boldsymbol{J}^{\theta}_{\set{#1}\nset{#2}}}}
\NewDocumentCommand\qfl{O{} O{}}{{q^\theta_{\set{#1}\nset{#2}}}}
\NewDocumentCommand\Manfl{O{} O{}}{{\mathcal{M}^\theta_{\set{#1}\nset{#2}}}}
\NewDocumentCommand\Man{O{} O{}}{{\mathcal{M}_{\set{#1}\nset{#2}}}}

\NewDocumentCommand\MI{O{} O{}}{{\mathcal{I}_{\set{#1}\nset{#2}}}}
\NewDocumentCommand\lam{O{} O{}}{{\lambda_{\set{#1}\nset{#2}}}}

\RenewDocumentCommand\L{O{} O{}}{{\mathcal{L}_{\set{#1}\nset{#2}}}}

\let\oldLambda\Lambda
\renewcommand{\Lambda}{\boldsymbol{\oldLambda}}

\newcommand{\tr}[1]{\text{tr}\left(#1\right)}

\newcommand{\C}{\mathcal{U}}
\newcommand{\met}{\mathbf{\gamma}}
\newcommand{\eye}{\mathbb{I}}
\DeclareMathOperator{\diag}{diag}
\newcommand{\EO}{\text{EO}}

\theoremstyle{plain}
\newtheorem{theorem}{Theorem}[section]

\newtheorem{corollary}[theorem]{Corollary}
\theoremstyle{definition}

\theoremstyle{remark}
\newtheorem{remark}[theorem]{Remark}

\title{From Core to Detail: Unsupervised Disentanglement with Entropy-Ordered Flows}

\author{Daniel Galperin, Ullrich Köthe\\
  Heidelberg University\\
  \texttt{\{daniel, ullrich\}.\{galperin, koethe\}@iwr.uni-heidelberg.de}
}
\runningtitle{Unsupervised Disentanglement with Entropy-Ordered Flows}
\firstpagenotice{Preprint. Under review.}

\begin{document}

\maketitle

\begin{abstract}

The unsupervised discovery of features that are both semantically meaningful and stable across runs remains a central challenge in representation learning.
We introduce entropy-ordered flows (EOFlows), a normalizing flow (NF) framework that augments standard maximum likelihood training with an orthogonality regularizer on the decoder Jacobian.
The regularizer is rooted in Independent Mechanism Analysis and encourages geometric disentanglement, and a stochastic estimator makes it tractable at image scale (CelebA at $D = 2352$ and $12288$).
Learned features form near-orthogonal curvilinear coordinates and can be ordered by their \textit{explained (manifold) entropy} after training, analogous to the ranking by explained variance in PCA.
This turns EOFlows into a non-linear generalization of PCA.
EOFlows identify an order of magnitude more stable features than existing methods, and these features emerge in distinguishable categories (global, local, and generic) and support tentative semantic interpretations.
The local features have strikingly sparse support in pixel space, although our method never enforces this.
Retaining only the most important, i.e. highest entropy, features turns the bijective flow into an autoencoder with adjustable bottleneck, rivaling the rate-distortion performance of dedicated autoencoders.
\end{abstract}

\section{Introduction}

A model is considered {\em disentangled} when its latent features represent distinct, independent, and informative factors of variation and remain stable under repeated training \citep{bengio2014representationlearningreviewnew,wang2024disentangledrepresentationlearning}.
\cite{HYVARINEN1999429} and \cite{locatello2019challengingcommonassumptionsunsupervised} showed that this cannot be achieved in a purely unsupervised manner without additional assumptions.
A simple counter-example is provided by generative models with a standard normal latent prior:
Since this prior is rotationally invariant, any latent rotation preserves the generated distribution -- no orientation is special.
Thus, the data's intrinsic factors have no reason to align with the latent coordinates, and stable disentanglement is impossible.
In our experiments, ordinary normalizing flows clearly exhibit this behavior.

We study {\em orthogonal (coordinate) maps} as an inductive bias to break this symmetry.
Orthogonality means that the decoder transforms the latent Cartesian coordinate system into a curvilinear, but still orthogonal coordinate system in data space.
Polar coordinates are a classic example for a curvilinear orthogonal system, and EOFlows generalize this concept to learned transformations.
If the latent coordinate system is rotated now, the induced coordinates in data space will no longer remain orthogonal.
Conversely, the requirement of orthogonal induced coordinates tends to single out a specific latent rotation and leads to reproducible latent features in our experiments.

Our method rests on a key insight from Independent Mechanism Analysis \citep[IMA,][]{gresele2021independent, ghosh2023independent}: 
A decoder defines an orthogonal map if its Jacobian has pairwise orthogonal columns almost everywhere.
This can be achieved by adding the {\em IMA contrast} to the training objective as a regularizer.
Their theoretical arguments and low-dimensional results indicate that orthogonality can indeed resolve the rotational ambiguities that otherwise hinder stable feature discovery.
In the pioneering work on {\em principal component flows} (PCF), \citet{pmlr-v162-cunningham22a} brought the IMA contrast to normalizing flows, but their experiments were still restricted to low-dimensional synthetic data.
EOFlows push the limit of what is practically possible to higher dimensions ($D=2352$ and $12288$), and the discovered features are informative and semantically nearly independent (Fig.~\ref{fig: archetypes overview}).

Orthogonal data-space coordinates have various desirable implications.
First, we can estimate the {\em explained entropy} \citep{galperin2025analyzing} of each latent feature and order the feature space according to this importance measure -- hence the name EOFlow.
This is analogous to the ordering by explained variance in classical Principal Component Analysis (PCA) and suggests that our method can be interpreted as a nonlinear generalization of PCA.
Second, the ordering allows us to define a threshold $C < \dim(Z)$ on the feature indices to separate the most important features (the ``core'' of the representation) from the rest (the ``details'').
Setting the details to zero before data generation or reconstruction turns EOFlows into {\em injective flows}.
Crucially, our bottleneck is {\em adaptive}, because suitable values of $C$ can be determined after training at {\em inference} time.
This sets our concept apart from $\beta$-VAEs \citep{burgess2018understandingdisentanglingbetavae}, rectangular flows \citep{NEURIPS2021_fde9264c}, or $\mathcal{M}$-flows \citep{NEURIPS2020_05192834}, which must fix $C$ as a training hyperparameter.
The entropy spectra of EOFlows typically decrease until they saturate at the data's noise floor, which defines a natural bound for the choice of $C$.
A model's intrinsic dimension is thus no longer fixed and pre-determined, but limited by the measurement accuracy: geometric structure below the noise level is statistically indistinguishable from noise, regardless of the true generative dimensionality.
Geometrically, the core dimensions span the $C$-dimensional {\em reconstruction manifold}, an approximation of the true data manifold.
The details represent the residual degrees of freedom and generalize the null space of linear maps.

\begin{figure}[!t]
    \centering
    \includegraphics[width=\linewidth]{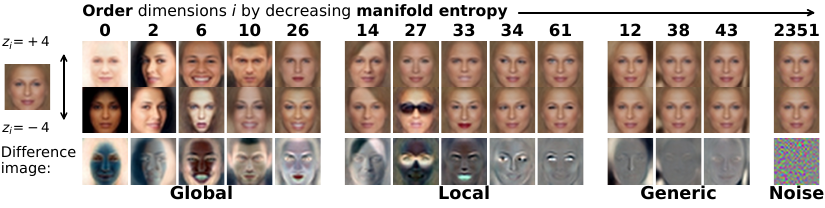}
    \caption{Archetypes of the features learned by EOFlows on CelebA, sorted by decreasing entropy.
    Global features involve the entire image, e.g. head pose.
    Local features affect only a small number of pixels, e.g. eye gaze.
    Generic features resemble windowed Fourier transforms when a subregion (e.g. the background) lacks semantic structure.
    Finally, features with low entropy contain only noise.}
    \label{fig: archetypes overview}\label{fig: top stable features EOFlows}
\end{figure}

\noindent \textbf{Contributions.}
\begin{itemize}[leftmargin=*, itemsep=1pt, topsep=2pt]
    \item \textbf{Tractable orthogonality.} EOFlows make the orthogonality bias of IMA tractable in Normalizing Flows and achieve rich unsupervised feature discovery on CelebA. The ingredients are a stochastic estimate of the Hadamard regularizer ($2.5\times$ the cost of a plain NF), a fully connected flow with dense learnable mixings (App.~\ref{app: architecture ablation}), and noise inflation against density overfitting (Sec.~\ref{sec:scaling}, \ref{sec:noise}).
    \item \textbf{Entropy ordering and an inference-time bottleneck.} Manifold entropies rank all $D$ features from $\textit{core}$ to $\textit{detail}$ at inference time, like explained variance ranks principal components. Linearized, our objective provably recovers PCA (App.~\ref{app: Linearized Normalizing Flows and PCA}). Keeping the top $C$ features yields an autoencoder without a reconstruction loss or a fixed bottleneck.
    \item \textbf{Stability.} Comparing learned features across training runs, we find that on CelebA EOFlows reproduce up to $150$ features across seeds, where tuned $\beta$-VAE variants reach at most $16$, nested-dropout flows $9$ and plain NFs none (Fig.~\ref{fig: stability all models}).
    \item \textbf{Sparsity.} Many EOFlow features are strikingly sparse in the pixel domain (local), although neither the architecture nor the objective favors sparsity; the sparsest touches only $3\,\%$ of the pixels, outperforming competitors (Table~\ref{tab: summary}).
\end{itemize}

\section{Related Work}
\label{sec:related-work}

\textbf{Unsupervised disentanglement with $\beta$-VAEs and GANs.}
The quest for semantically meaningful, independent factors of variation has been largely dominated by the $\beta$-VAE framework and its derivatives \citep{higgins2017betavae, burgess2018understandingdisentanglingbetavae, pmlr-v80-kim18b, NEURIPS2018_1ee3dfcd, zhao2018infovaeinformationmaximizingvariational, NEURIPS2018_b9228e09}. 
These models rely on a bottleneck to enforce the discovery of informative features, often resulting in an inherent trade-off between disentanglement and reconstruction quality. 
In our experiments they discover at most $\approx 20$ meaningful features, regardless of the bottleneck size, and \citet{locatello2019challengingcommonassumptionsunsupervised} found across six objectives that the seed and the regularization strength matter more than the objective itself.
Recent studies suggest that the success of $\beta$-VAEs in disentangling features is closely linked to Independent Mechanism Analysis (IMA), because training is implicitly biased towards decoders with orthogonal Jacobians \citep{Rolinek_2019_CVPR, NEURIPS2022_4eb91efe}.
Related regularizers disentangle GAN generators by penalizing off-diagonal entries of their Hessian \citep{peebles2020hessian} or non-orthogonal columns of their Jacobian \citep[OroJaR,][]{wei2021orojar}.
EOFlows make the orthogonality bias explicit in a bijective model: no bottleneck size is fixed before training and the exact likelihood yields the entropy ordering.

\textbf{Normalizing flows for manifold learning.}
Standard bijective Normalizing Flows (NFs) struggle with data residing on low-dimensional manifolds. 
This is addressed by injective NF variants like $\mathcal{M}$-flows \citep{NEURIPS2020_05192834}, Rectangular Flows \citep{NEURIPS2021_fde9264c} and Injective Free-Form Flows \citep{sorrenson2024lifting}, and by denoising NFs \citep{NEURIPS2021_4c07fe24}.
However, they require the intrinsic dimensionality $C$ to be fixed as a training hyperparameter and a reconstruction error to find the manifold.
The GIN method \citep{Sorrenson2020Disentanglement} can identify the data manifold with full-dimensional, volume-preserving flows, but needs an auxiliary class label for each instance.

\textbf{Independent mechanism analysis and principal component flows.}
The insight that geometrically independent, i.e.\ \textit{disentangled}, features correspond to \textit{orthogonal} columns in the decoder Jacobian is a central result of IMA \citep{gresele2021independent, ghosh2023independent}. 
Principal Component Flows \citep[PCF,][]{pmlr-v162-cunningham22a} brought this contrast to normalizing flows but faced limited empirical success in scaling to high-dimensional data.
PCF also proposed a stochastic estimate with randomized Jacobian-vector products; ours differs in using decoder columns rather than encoder rows and in being scaled up to $D=2352$--$12288$ (App.~\ref{app: Related work}).

\textbf{Importance ordering of the latent features.}
Inducing a preference to channel important information through specific latent dimensions has been explored through Nested Dropout (ND) in both autoencoders \citep{pmlr-v32-rippel14} and flows \citep{bekasov2020orderingdimensionsnesteddropout}. 
Neural PCA \citep{li2022neuralpcaflowbasedrepresentation} obtains an ordering by fitting a PCA to the latents of a trained flow.
These methods encourage importance ordering but lack an information-theoretic grounding: the order is imposed by the training scheme, not measured from the learned geometry.

\textbf{Inflation-deflation and generative quality.}
TarFlow and its successors \citep{zhai2025normalizing, chen2026normalizingflowsiterativedenoising}, a recently introduced improvement of normalizing flows, have highlighted the importance of adding noise during training, without giving an explanation.
Noise inflation-deflation was formally analyzed by \citet{NEURIPS2022_4f918fa3,horvat2023density} and later adopted as a way to mitigate manifold overfitting \citep{loaiza-ganem2022diagnosing} and improve generative quality \citep{pmlr-v187-loaiza-ganem23a}. 
EOFlows can also remove the noise through their entropy order (Sec.~\ref{sec:noise}).

\section{Method}

\subsection{Normalizing flows and induced curvilinear coordinates}\label{sec:coordinates}

NFs learn encoder/decoder pairs $\z = \f(\x)$, $\x = \g(\z) \coloneq \f^{-1}(\z)$ that define a bijective mapping between data and latent space ($\x,\z\in\mathbb{R}^D$).
We use invertible neural networks where the decoder is realized by running the encoder backwards, so that $\f$ and $\g$ share the learnable parameters $\theta$.
As usual, the latent prior is a standard normal $p(\Z=\z) = \mathcal{N}(\z\,|\, 0, \I_D)$.
The decoder induces a push-forward distribution $q(\X)$ in the data space, whose PDF is given by the change-of-variables formula
\begin{equation}
    \!\!q\big(\X\! =\! \g(\z)\big)
    = p(\Z\! =\!\z) \cdot \big| \J(\z) \big|^{-1}\;\text{ with }\z\sim p(\Z)
\end{equation}
where $\J(\z) \coloneqq \frac{\partial \g(\z')}{\partial \z'}\big|_{\z'=\z}$ is the decoder Jacobian at $\z$ and $|\boldsymbol{A}| \coloneqq \det(\boldsymbol{A}^T\boldsymbol{A})^\frac{1}{2}$ the volume of a square or rectangular matrix. 
NFs are trained by the maximum likelihood objective, i.e.\ by minimizing the negative log-likelihood of $q(\X)$ on samples $\x\sim p^*(\X)$ of the data distribution:
\begin{equation}
\begin{split}
    \L_\ML(\x) \coloneq - \log(q(\X\!=\!\x)) = \frac{1}{2} \big\lVert\f(\x)\big\rVert_2^2 + \log\left|\J\big(\z\!=\!\f(\x)\big)\right| + \text{const.}
\end{split}
\end{equation}
In addition to the push-forward of the density, the decoder also defines a push-forward of the latent Cartesian coordinates into curvilinear coordinates in data space.
The individual coordinate lines of this system are obtained by varying a single latent coordinate $z_i$, while the others $\z_{\ne i}$ remain fixed
\begin{equation}
    \C_i(\z_{\ne i})=\big\{\x = \g(z_i, \Z_{\ne i}=\z_{\ne i}):\,z_i\in\R, \z_{\ne i}=\text{const} \big\}
\end{equation}
The properties of this curvilinear coordinate system are determined by its {\em pull-back metric} $\met(\z) = \J(\z)^T \J(\z)$.
The columns of the Jacobian $\J_i(\z) = \partial \g(\z) / \partial z_i$ are the tangent vectors of $\C_i$ at location $\x=g(\z)$.
We illustrate the push-forward density of decoders and their induced coordinate systems in Fig.~\ref{fig: push-forward pdf factorization}.
On CelebA, we visualize the feature $i$ by its {\em archetype}: the images at the two extreme ends $z_i=\pm4$ of the mean coordinate line $\C_i(\z_{\ne i}\!=\!\mathbf{0})$ and their difference image (Fig.~\ref{fig: archetypes overview}; App.~\ref{app: CelebA archetypes} for all models).

\subsection{Disentanglement, identifiability, and orthogonality}\label{sec: disentanglement}

The term ``disentanglement'' intuitively refers to representations where each feature has a distinct semantic meaning that is largely independent of the values of the other features.
Among its possible formalizations (App.~\ref{app: stochastic and geometric disentanglement}), we use {\em geometric disentanglement}, which requires the flow's induced coordinate system to be orthogonal almost everywhere
\begin{equation}
    \text{orthogonal map:} \quad \forall i\ne j:\, \J_i(\z)^T \J_j(\z) = 0\quad\Longleftrightarrow\quad\met(\z) = \diag\big(\met(\z)_{11}, ..., \met(\z)_{DD}\big)\,\,
\end{equation}

\begin{figure*}[t]
    \centering
    \includegraphics[width=0.9\linewidth]{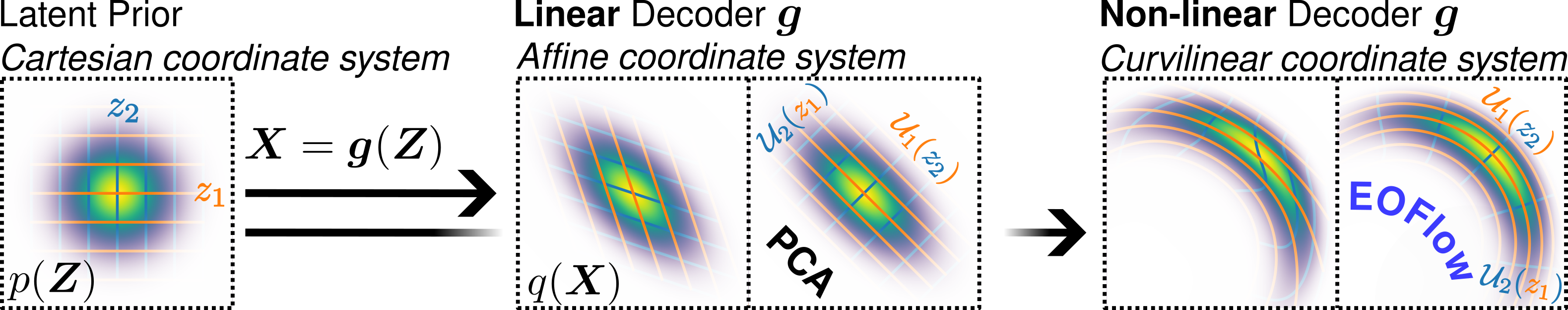}\caption{\textbf{Induced coordinates.} A standard normal latent (left) is pushed forward by linear (center) and non-linear (right) decoders, which induce affine and curvilinear coordinates in the data space. Both decoders of each pair fit the same distribution, but only the ones on the right, PCA and EOFlows, have orthogonal coordinate lines $\C_1(z_2)$ and $\C_2(z_1)$ that align with the data geometry.}
    \label{fig: push-forward pdf factorization}
\end{figure*}

Identifiability concerns our ability to learn such representations reliably.
\citet{nelsonstatistical} call a representation {\em structurally identifiable} if the learned features reproduce the independent factors of the true generative process, and {\em statistically identifiable} if repeated training always converges to the same features.
We target the latter, since the true process of natural data is unknown.
Identifiability must hold up to {\em tolerable} ambiguities, usually taken as index permutations and element-wise reparameterizations of the latents. 
We follow this view and disambiguate the disentanglement by additional conditions: 
Indices are fixed by sorting the features by importance, and the parameterization is fixed by enforcing a standard normal prior (instead of an arbitrary separable prior). 
These conditions do not affect feature identification, leaving sign flips $z_i \rightarrow-z_i$ as the only remaining ambiguity.

Orthogonality is not an arbitrary choice: 
It is the only geometric constraint that is stable under precisely the tolerable ambiguities discussed above \citep{gresele2021independent}.
If the metric $\met(\z)$ is non-orthogonal, i.e.\ has non-zero off-diagonal elements, one can construct metric-preserving latent rotations, which are not tolerable.
Conversely, more restrictive {\em conformal maps} with $\met(\z)=\kappa(\z)\eye$ lose the tolerance against element-wise reparameterizations \citep{NEURIPS2022_6c5da478}.
Figure \ref{fig: push-forward pdf factorization} illustrates how orthogonal maps achieve identifiability:
If the latent space is rotated (a non-tolerable ambiguity), the orthogonality between the induced coordinates $\C_i$ is lost, despite the generated distribution remaining unchanged due to the rotational symmetry of the standard normal prior.

On the other hand, we conjecture that orthogonal maps with separable latent priors are not universal approximators (a formal proof is beyond the scope of this paper).
Thus, orthogonality should only be enforced as a {\em soft} constraint in order to maintain control of the inherent trade-off between generative accuracy and geometric disentanglement.

\subsection{Training objective}\label{sec:objective}

We employ a regularized maximum likelihood objective for EOFlows with $\lam$ as a tuning parameter
\begin{equation}\label{eq: def EOFlow objective}
    \L_\EO(\x) = \L_\ML(\x) + \lam \cdot \L_\text{H}(\x) ,\quad \lam > 0
\end{equation}
$\L_\text{H}(\x)$ is the {\em Hadamard regularizer}, which ensures {\em soft} orthogonality between all $D$ Jacobian columns $\J_i$ simultaneously:
\begin{equation}\label{eq: def Total Disentanglement MAIN}
    \L_\text{H}(\x)  := \sum\nolimits_{i=1}^D \log\big\| \J_i(\f(\x))\big\|_2 - \log\big| \J(\f(\x))\big| \ge 0
\end{equation}
\citet{gresele2021independent} called this the {\em IMA contrast}, though it stems from {\em Hadamard's inequality}:
\begin{equation}
    \big|\A\big|\le\prod\nolimits_{i=1}^D \big\|\A_i\big\|_2
\end{equation}
It states that the volume of any matrix $\A$ is bounded by the product of its column norms, and equality is achieved if and only if all columns $\A_i$ are pairwise orthogonal.

Eq. \ref{eq: def Total Disentanglement MAIN} can be efficiently evaluated in practice: 
$| \J(\f(\x))|$ is already available from the maximum likelihood loss, and $\J_i(\z)$ can be calculated via Jacobian-vector product primitives provided by all major autodiff frameworks (e.g. PyTorch, JAX).

We prove in App.~\ref{app: Linearized Normalizing Flows and PCA} that for a linear decoder the optimum of $\L_\EO$ is the PCA solution. Linear EOFlows perform PCA by gradient descent, and non-linear ones generalize it to curvilinear coordinates.
Like PCA, EOFlows are equivariant to orthogonal changes of the data basis: if $\g$ minimizes $\L_\EO$ on $\x$, then $\boldsymbol{R}\circ\g$ minimizes it on $\boldsymbol{R}\x$ for any orthogonal $\boldsymbol{R}$, since $|\boldsymbol{R}\J|=|\J|$ and $\|\boldsymbol{R}\J_i\|_2=\|\J_i\|_2$, and the isotropic noise inflation (Sec.~\ref{sec:noise}) is invariant under $\boldsymbol{R}$ as well. The objective thus singles out no pixel basis; Sec.~\ref{sec: equivariance}, {\em Equivariance of EOFlows}, tests whether a trained network inherits this property.

\subsection{Entropy ordering, manifold learning, and compression}\label{sec:ordering}

Each coordinate line $\C_i(\z_{\ne i})$ is a 1-dimensional manifold in data space, on which the prior of $z_i$ induces the {\em manifold PDF} $q_i\big(\X\!=\!\g(\z)\big) = p(z_i)\cdot\|\J_i(\z)\|_2^{-1}$ by the injective change-of-variables formula.
Its expected negative log-density over the entire latent space is the {\em manifold entropy} \citep{galperin2025analyzing} of feature $i$\begin{equation}
    H_i \coloneqq H(q_i) = \E_{\z\sim p(\z)}\left[\frac{1}{2}(1 + \log 2\pi) + \log \big\|\J_i(\z)\big\|_2\right]
\end{equation}
A large $H_i$ means that the data vary strongly along $\C_i$. Under the manifold hypothesis, this is the case for coordinates aligned with the data manifold, whereas coordinates perpendicular to it carry only noise and have small entropies.
\citet{galperin2025analyzing} found that the model entropy decomposes into a sum over manifold entropies minus a non-negative term $\mathcal{I}$, which is precisely the residual expected Hadamard loss (\ref{eq: def Total Disentanglement MAIN}) under the model distribution:
\begin{equation}
    H \coloneqq H(q) = \sum\nolimits_{i=1}^D H_i - \mathcal{I}, \qquad \mathcal{I} = \E_{\x\sim q(\X)}\big[\L_\text{H}(\x)\big] \geq 0 .
\end{equation}
As Maximum Likelihood makes $q(\X) \approx p^*(\X)$, $H(q)$ converges to the data entropy $H^*=H(p^*)$, a constant. 
The Hadamard regularizer keeps $\mathcal{I}$ small, so that $\sum_i H_i \approx H^*$.
Thus, each $H_i$ measures only the variability along its own coordinate line without significant cross-talk. 
This is analogous to PCA, where the variances of the individual principal components sum to the total variance.
We can thus meaningfully sort the latent coordinates, and this inspired our method's name ``EOFlows''.

In addition, we can split the latent space into a {\em core} subspace of important features (high $H_i$, aka. ``content''), and a {\em detail} subspace of the rest (low $H_i$, aka. ``noise'').
This allows us to introduce a {\em dynamic bottleneck} of size $C<D$, which assigns features $\z_{\le C}$ (after importance ordering) to the core subspace.
This defines an autoencoder-like dimension reduction, but with the important advantage that we can decide about the optimal $C$ {\em after} we have trained the geometry of $\g(\z)$ and calculated the $H_i$.
An autoencoder-like reconstruction is defined as
\begin{equation}\label{eq:autoencoder-like-reconstruction}
    \widehat{\x}_{\downarrow C} = \g\Big(\big[\z_{\le C}=\f_{\le C}(\x), \z_{>C}=0\big]\Big)
\end{equation}
If the detail subspace contains only independent isotropic noise, the $\C_i$ for $i>C$ tend to be straight lines (App.~\ref{app: line straightness}).
Since these $\C_i$ are perpendicular to the data manifold by way of orthogonality regularization, the reconstructions $\widehat{\x}_{\downarrow C}$ of off-manifold points $\x$ will approximately be orthogonal projections onto the manifold, i.e. (\ref{eq: def Total Disentanglement MAIN}) implicitly reduces the squared reconstruction error $\|\x -\widehat{\x}_{\downarrow C}\|_2^2$.
We show in Figs.~\ref{fig: noise denoise} and \ref{fig: CelebA rate-distortion SSIM test}b that this leads to powerful denoising and to compression with favorable rate-distortion behavior.

The dynamic bottleneck also reveals a close relation of EOFlows to {\em nested dropout flows} \citep[NDFlows,][]{bekasov2020orderingdimensionsnesteddropout}, which minimize the reconstruction error of (\ref{eq:autoencoder-like-reconstruction}) {\em during training} for a random bottleneck $C\sim p_C$ drawn from a geometric distribution per image, weighted by $\nu$ against the maximum likelihood loss (App.~\ref{app: Nested Dropout implementation}).
This forces the core information into the first dimensions and recovers PCA in the linear case \citep{pmlr-v32-rippel14}. 
Importantly, NDFlows enforce the order {\em during training} through their reconstruction loss, whereas EOFlows arrive at an equally {\em compressed} representation (Fig.~\ref{fig: CelebA Manifold Entropy Spectra}a) without any such term and order it only after training.

\subsection{Noise inflation and denoising}\label{sec:noise}
We train on data inflated with isotropic Gaussian noise that is redrawn at every step:
\begin{equation}
    \xnoisy \coloneq \x + \noisesig \cdot \eps \ , \quad \eps \sim \mathcal{N}(\mathbf{0}, \I_D)
\end{equation}
\begin{wrapfigure}{r}{0.5\linewidth}
    \centering
    \vspace{-10pt}
    \includegraphics[width=\linewidth]{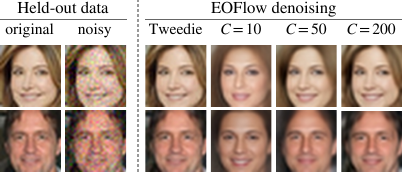}
    \caption{\textbf{Denoising} of held-out images by one Tweedie step and through the top-$C$ bottleneck of an EOFlow ($\lam=1$).}
    \label{fig: noise denoise}
    \vspace{-10pt}
\end{wrapfigure}
\leavevmode Noise may look like a necessary evil, but it cures two problems of likelihood-based models.
First, bijective flows cannot learn lower-dimensional manifolds: their likelihood would grow without bound if they contracted the normal directions \citep[manifold overfitting,][]{loaiza-ganem2022diagnosing}.
Second, a finite training set is a zero-dimensional manifold, so memorizing it would give unbounded likelihood as well \citep{ross2024a}.
With noise redrawn at every step, the flow instead fits a smooth and bounded kernel density estimate $\frac1N\sum_n\mathcal{N}(\x_n,\noisesig^2\I_D)$ of the training set.
We use $\noisesig=0.05$ throughout all experiments, selected by sweeping $\noisesig$ from $0.01$ to $0.5$ on smaller networks (App.~\ref{app: noise grid}).
Values $\noisesig<0.01$ (and specifically $\noisesig=0.002$, the quantization noise level of 8-bit images) do not work, as two out of three training runs diverge.
Noise $\noisesig>0.5$ starts to mask important information, since structures with variance below $\noisesig^2$ become invisible.
The entropy of the noise $H_{\noisesig}=\tfrac12\log(2\pi e\,\noisesig^2)$ is clearly visible in the entropy spectrum (Fig.~\ref{fig: CelebA Manifold Entropy Spectra}a) as a lower bound for the spectrum.
EOFlows handle this floor gracefully:
The associated coordinate lines $\C_i$ become straight lines orthogonal to the data manifold (App.~\ref{app: line straightness}) and can be easily removed during reconstruction by setting all $\z_i=0$ whose $H_i \approx H_{\noisesig}$, see Fig.~\ref{fig: noise denoise}.
This denoising is explained in detail under {\em Compression and denoising} in Sec.~\ref{sec:experiments}.

\section{Experiments}\label{sec:experiments}

\label{sec: setup}\label{sec:scaling}
\textbf{Stochastic estimation of $\L_\text{H}$ at scale:}
Instead of calculating all column norms $\|\J_i\|_2$ in (\ref{eq: def Total Disentanglement MAIN}) for each data instance, we follow \citep{pmlr-v162-cunningham22a} and select a single index $i$ uniformly at random.
This requires only one Jacobian-vector product per instance instead of $D$.
Since the index is uniform, $D\log\|\J_i\|_2$ has expectation $\sum_{k=1}^D\log\|\J_k\|_2$, so the estimate of $\L_\text{H}$ is unbiased for any batch size $B$.
A training step thus takes only $2.5\times$ as long as for a plain NF (App.~\ref{app: stochastic estimate of L_TC}), regardless of $D$ and $B$.

\textbf{Data, baselines and training:}
We train on CelebA \citep{liu2015faceattributes}, center-cropped and downscaled to $28\times28\times3$ ($D=2352$) and inflated with noise of $\noisesig=0.05$.
We vary the Hadamard regularization strength of EOFlows in $\lam\in\{0,0.01,0.1,1,3,10\}$, where $\lam=0$ is the unconstrained NF, and compare with PCA, with NDFlows ($\nu\in\{0.1,1,10\}$) of the same architecture, and with fully connected $\beta$-VAE \citep{higgins2017betavae}, $\beta$-TCVAE \citep{NEURIPS2018_1ee3dfcd}, FactorVAE \citep{pmlr-v80-kim18b} and AnnealedVAE \citep{burgess2018understandingdisentanglingbetavae} with $32$ latents, each tuned for the largest number of stable features (defined below).
See App.~\ref{app: Training setup} and App.~\ref{app: VAE baselines} for more details.

\textbf{EOFlow architecture:}
All $28\times28$ flows share one fully connected architecture of $12$ Real NVP affine coupling blocks \citep{dinh2017density} with two-hidden-layer MLPs of width $1024$, each followed by a random permutation and ActNorm \citep{NEURIPS2018_d139db6a}, and learnable affine $D\times D$ mixing matrices between blocks and after the last one (App.~\ref{app: Normalizing Flow architecture}).
Our convolutional EOFlow for $D=12288$ is specified below under \emph{Scaling to higher resolution} and in App.~\ref{app: conv64}.

\textbf{Entropy spectra:}
Fig.~\ref{fig: CelebA Manifold Entropy Spectra}a shows the manifold entropy spectra of all trained model classes together with the noise entropy $H_{\noisesig}$, which forms an approximate floor. As EOFlows and NDFlows concentrate the entropy in a few hundred coordinates, a small bottleneck $C$ loses little information.
For noisy PCA, the manifold entropies directly measure the variance along each principal direction (App.~\ref{app: Rate-distortion plots}).
The best VAE, FactorVAE, uses only $20$ of its $32$ latents and leaves the rest inactive, whereas the bijective flows must account for all $D$ dimensions.
The plain NF spreads its entropy uniformly over all $D$ coordinates, so it admits no meaningful ordering, and a bottleneck would keep arbitrary latents.

\textbf{Compression and denoising:} Fig.~\ref{fig: CelebA rate-distortion SSIM test}b turns the entropy ordering into a rate--distortion curve: held-out images are inflated, encoded, truncated to their $C$ highest-entropy latents, decoded and compared to the clean image, which tests compression and denoising at once.
That EOFlows compete here at all is remarkable, since their objective (\ref{eq: def EOFlow objective}) contains no reconstruction error. Together with maximum likelihood, their \emph{local} bias towards orthogonal coordinates aligns them \emph{globally} with the principal factors of the data, so that thousands of features can be sorted by entropy in a meaningful way.

On unseen images, EOFlows reconstruct better than NDFlows for $C\ge20$ and better than denoising autoencoders for $C\ge100$, although dedicated autoencoders have been trained separately for each $C$ for exactly this task.
At large bottlenecks, PCA catches up with EOFlows, since the detail coordinates model Gaussian noise, for which the EOFlow converges to a PCA-like linear projection.
A single correction by Tweedie's formula $\widehat{\x}= \x+\sigma^2\cdot \nabla_{\x} \log q(\x)$ \citep{efron2011tweedie,pmlr-v187-loaiza-ganem23a,zhai2025normalizing} on the flow outputs (dotted lines) denoises better than any bottleneck (best for EOFlows). 
However, this acts on all $D$ coordinates and cannot compress.
See App.~\ref{app: prior samples} and~\ref{app: Reconstruction} for prior samples and reconstructions through the bottleneck.

\begin{figure}[t]
    \centering
    \includegraphics[width=\linewidth]{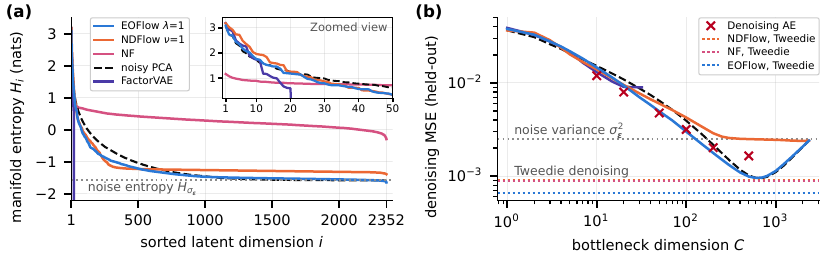}
    \caption{\textbf{Entropy spectra, compression and denoising rate-distortion} for representative models. \textbf{(a)} Manifold entropy spectra $H_i$ define the ordering of latent dimensions. EOFlows and NDFlows concentrate their entropy in a few hundred dimensions; EOFlows flatten towards the noise entropy $H_{\noisesig}$, separating high-entropy {\em core} from low-entropy {\em detail} dims. The plain NF spreads it over all dims, and FactorVAE (best VAE) uses at most $20$ active latents. \textbf{(b)} Denoising error of held-out images reconstructed through the top-$C$ entropy-ordered bottleneck; the plain NF is omitted, as its error is far above the plotted range. Additionally to models in (a): denoising autoencoders trained separately for each $C$ (App.~\ref{app: denoising AE}), and Tweedie denoising of the flows without a bottleneck (dotted).
    }
    \label{fig: CelebA Manifold Entropy Spectra}\label{fig: CelebA rate-distortion SSIM test}
\end{figure}

\begin{figure}[t]
    \centering
    \includegraphics[width=\linewidth]{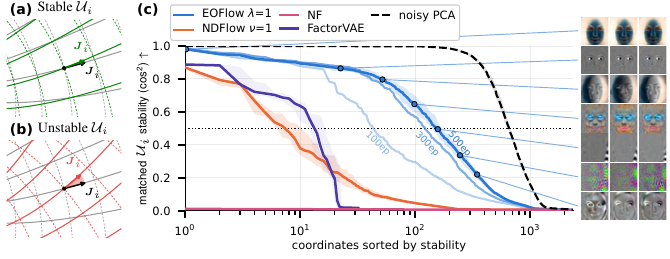}
    \caption{\textbf{Stable features across seeds.} \textbf{(a, b)} Sketch: coordinate lines $\C_i$ of a reference model (gray) against one that reproduces them (a) and one that does not (b), quantified by the alignment of their tangent vectors $\J_i$. \textbf{(c)} Left: matched $\C_i$ stability (Sec.~\ref{sec: setup}) of every coordinate, sorted in decreasing order (median over seed pairs). EOFlows keep gaining stable coordinates up to $500$ epochs, whereas NDFlows and FactorVAE, trained just as long, do not. The plain NF has no stable coordinate, and noisy PCA (dashed) has hundreds of them, by construction. Budgets and seeds in Table~\ref{tab: runs and seeds}. Right: archetypes of $7$ representative EOFlow features in three random seeds (all six: App.~\ref{app: stable features}).}

    \label{fig: stability all models}\label{fig: stability curves}

\end{figure}

\textbf{Stable features across seeds:}
To compare how stable features are found across random seeds, we compare the coordinate grids $\C_i$ locally via the decoder Jacobians between two trained runs. Instead of the naive correlation between latent codes, which is a statistic between distributions, our stability metric compares the Jacobians at every sample and takes the expectation afterwards (App.~\ref{app: Stability}).
For every model, we replace the Jacobian at fixed training images by its polar factor, the closest orthogonal matrix, whose columns point along the locally orthogonal coordinate lines. The squared cosines between these directions in two runs, averaged over the images and Hungarian-matched to remove the permutation (and, by squaring, the sign) ambiguity, give the \emph{matched $\C_i$ stability} of each coordinate, $1$ if it is reproduced exactly and $\approx 0$ for random frames. Coordinates matched at $0.5$ or more count as stable; we report the mean and standard deviation over seed pairs.

Fig.~\ref{fig: stability all models} and Table~\ref{tab: summary} compare the best setting of each model class. EOFlows at $\lam=1$ reproduce $121_{\pm5}$ coordinates per seed pair after $300$ epochs and $154_{\pm10}$ after $500$, whereas the tuned VAEs reproduce at most $16$ of their $32$ latents, NDFlows fewer than $10$, and NFs none. The threshold is conservative: the archetypes in Fig.~\ref{fig: stability all models}b) show that coordinates with a matched stability below $0.5$ are still similar.

\textbf{Coordinate sparsity:}
A feature is sparse if moving along its coordinate line changes only a small number of pixels of the decoded image. Archetypes show only the two endpoints of one coordinate line.
For a local measure, we use the Jacobian column $\J_i$, which gives the change of the image along $\C_i$ at a sample, ask over how many pixels this change is spread, and take the expectation over samples.
We use the effective percentage of non-zero pixels (ENZ; \citealp{he2026effectivesparsityunifiedframework}), a threshold-free and scale-invariant pixel count (App.~\ref{app: sparsity}).
It ranges from $1/D$ (a single pixel changes) to $1$ (all pixels change equally), and a random direction scores $0.48$.
To measure sparsity, we evaluate the ENZ over the $300$ highest-entropy \emph{core} features (all $32$ for VAEs) and report its minimum and mean (Fig.~\ref{fig: lambda dial}d, Table~\ref{tab: summary}).
The {\em local} features of Fig.~\ref{fig: archetypes overview} make EOFlow coordinates the sparsest: at $\lam=1$ the sparsest one touches $3\,\%$ and $20\,\%$ on average, whereas even the sparsest coordinate of every baseline touches more than a quarter of it.

\begin{wrapfigure}{r}{2.840in}
    \centering
    \vspace{-10pt}
    \includegraphics[width=2.840in]{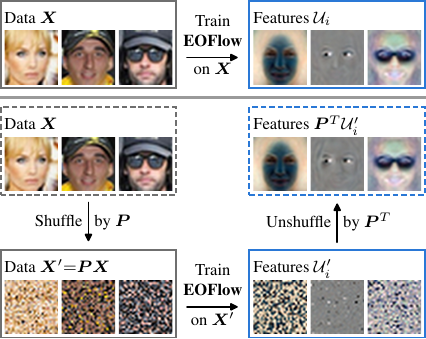}
    \caption{\textbf{Equivariance of EOFlows.}}
    \label{fig: equivariance}
    \vspace{-10pt}
\end{wrapfigure}
\textbf{Equivariance of EOFlows:}\label{sec: equivariance}
The EOFlow objective is equivariant to the basis of the data (Sec.~\ref{sec:objective}), and its fully connected architecture has no notion of pixel neighborhoods, i.e.\ EOFlows are pixel-agnostic: learned features should not be the result of inductive biases, e.g.\ convolutions, but emerge from the data distribution and the orthogonality bias alone.
To test this, we train EOFlows on CelebA images whose pixels have been shuffled randomly and have thus lost all locality (Fig.~\ref{fig: equivariance}).
Concretely, we shuffle each training image $\x$ by one fixed random permutation $\boldsymbol{P}$, such that $\x' = \boldsymbol{P}\x$, and denote the features learned on the shuffled data by $\C'_i$.
Permuted back, the features $\boldsymbol{P}^T\C'_i$ coincide with the features $\C_i$ of a model trained directly on $\x$.
Even local features, like eye gaze and eyeglasses, are found exactly.
By training multiple seeds, we verify that the {\em shuffled} EOFlows find as many stable features as the original ones.
This demonstrates the equivariance of the recipe empirically and indicates that the sparsity of EOFlow features originates in the data distribution itself, which the orthogonality bias merely reveals.
See App.~\ref{app: rotated basis} for details.

\textbf{Sweeping the orthogonality penalty $\lam$ from NF to PCA:}
Fig.~\ref{fig: lambda dial} follows six quantities as the Hadamard penalty $\lam$ runs from the unconstrained NF ($\lam=0$) to $\lam=10$.
Orthogonality (b) and the linearity of the coordinate grid (e) increase monotonically.
Already $\lam=0.01$ removes most of the Hadamard gap at negligible cost in likelihood (a) and none in sample quality (f), and its archetypes are interpretable (App.~\ref{app: CelebA archetypes}), yet it reproduces only $4.0_{\pm1.7}$ coordinates per seed pair (c).
Held-out likelihood (a), stability (c) and sparsity (d) peak at $\lam=1$.
Stronger penalties push the flow towards PCA (b, e), since a linear map can always be made exactly orthogonal by rotating the latent space, whereas orthogonal curvilinear coordinates are a restrictive class (Sec.~\ref{sec: disentanglement}).
Yet, at $\lam=10$ it does not inherit PCA's stability and falls behind in likelihood and FID, although PCA lies in its model class, likely because the variance of the stochastic penalty then dominates the training loss (Fig.~\ref{fig: penalty noise}c).

\begin{figure}[t]
    \centering
    \includegraphics[width=\linewidth]{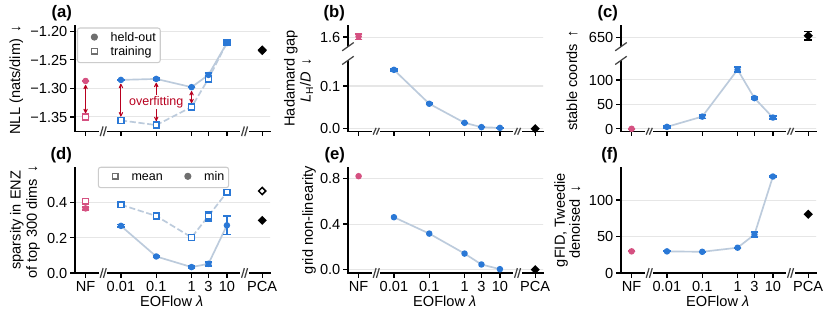}
    \caption{\textbf{Sweeping the orthogonality penalty $\lam$ from NF to PCA.} Each NF/EOFlow trained with $3$--$4$ seeds for $300$ epochs, $6$ noisy PCA fits. \textbf{(a)} Negative log-likelihood on training and held-out data; the gap between them is overfitting. \textbf{(b)} Hadamard gap $\L_\text{H}/D$. \textbf{(c)} Coordinates reproduced per seed pair (matched $\C_i$ stability $\geq0.5$). \textbf{(d)} Sparsity via mean and minimum ENZ. \textbf{(e)} Non-linearity of coordinates (one seed, App.~\ref{app: line straightness}). \textbf{(f)} gFID of prior samples after Tweedie denoising.}
    \label{fig: lambda dial}\label{fig: angles all}\label{fig: sparsity all}\label{fig: NLL_MTC_EOFLows}
\end{figure}

\begin{table}[t]
\centering\small
\caption{\textbf{Reproducible, sparse and disentangled.} Mean$_{\pm\text{sd}}$ over seeds. CelebA at $28\times28$: $300$ epochs (AnnealedVAE: $100$), fully-connected architectures; stability and sparsity as in Sec.~\ref{sec: setup}; TAD = attribute disentanglement score \citep{yeats2022nashae} (App.~\ref{app: attributes}). Orthogonality = Hadamard gap of the core coordinates (App.~\ref{app: orthogonality}). dSprites and Shapes3D: $64\times64$, $100$ epochs, convolutional architectures (App.~\ref{app: toy benchmarks}).
PCA fits orthogonal and reproducible components by construction (in parentheses). On the disentanglement benchmarks we compare EOFlows with the VAEs only.
}
\label{tab: summary}
\setlength{\tabcolsep}{2.5pt}
\setlength{\aboverulesep}{0pt}\setlength{\belowrulesep}{0pt}\renewcommand{\arraystretch}{1.15}
\begin{tabular}{lcccc|cc|cc}
 & \multicolumn{4}{c|}{CelebA (fully connected)} & \multicolumn{2}{c|}{dSprites (conv.)} & \multicolumn{2}{c}{Shapes3D (conv.)} \\ \cmidrule(lr){2-5}\cmidrule(lr){6-7}\cmidrule(lr){8-9}
model & stability $\uparrow$ & {\renewcommand{\arraystretch}{0.7}\begin{tabular}[b]{@{}c@{}}\rule{0pt}{2.3ex}orthog-\\onality $\downarrow$\end{tabular}} & {\renewcommand{\arraystretch}{0.7}\begin{tabular}[b]{@{}c@{}}\rule{0pt}{2.3ex}sparsity\\min/mean $\downarrow$\end{tabular}} & TAD $\uparrow$ & MIG $\uparrow$ & DCI $\uparrow$ & MIG $\uparrow$ & DCI $\uparrow$ \\
\midrule
\textcolor[HTML]{2A78D6}{EOFlow $\lam=1$} & \textbf{121}$_{\pm5}$ & \textbf{0.041}$_{\pm.001}$ & \textbf{0.03}/\textbf{0.20} & 0.72$_{\pm.01}$ & 0.14$_{\pm.10}$ & 0.07$_{\pm.04}$ & 0.15$_{\pm.03}$ & \textbf{0.52}$_{\pm.03}$ \\
\textcolor[HTML]{2A78D6}{EOFlow $\lam=0.1$} & 25.0$_{\pm3.6}$ & 0.17$_{\pm<.01}$ & 0.09/0.32 & \textbf{1.00}$_{\pm.01}$ & \textbf{0.26}$_{\pm.01}$ & 0.19$_{\pm<.01}$ & 0.17$_{\pm.07}$ & 0.41$_{\pm.09}$ \\
\textcolor[HTML]{EB6834}{NDFlow $\nu=1$} & 8.5$_{\pm1.8}$ & 0.29 & 0.30/0.40 & 0.24$_{\pm.01}$ & -- & -- & -- & -- \\
\textcolor[HTML]{EB6834}{NDFlow $\nu=0.1$} & 1.7$_{\pm1.2}$ & 0.46$_{\pm<.01}$ & 0.32/0.41 & 0.58$_{\pm.07}$ & -- & -- & -- & -- \\
\textcolor[HTML]{D55181}{NF} & 0 & 1.6$_{\pm<.1}$ & 0.37/0.41 & 0.00$_{\pm.00}$ & -- & -- & -- & -- \\
\midrule
\textcolor[HTML]{4A3AA7}{FactorVAE} & 15.7$_{\pm1.2}$ & 0.38$_{\pm.07}$ & 0.28/0.37 & 0.72$_{\pm.12}$ & 0.15$_{\pm.07}$ & 0.14$_{\pm.03}$ & 0.13$_{\pm.03}$ & 0.14$_{\pm.05}$ \\
\textcolor[HTML]{6A5CC6}{$\beta$-TCVAE} & 10.3$_{\pm1.2}$ & 0.55$_{\pm.01}$ & 0.33/0.41 & 0.44$_{\pm.01}$ & \textbf{0.26}$_{\pm.09}$ & \textbf{0.25}$_{\pm.03}$ & 0.31$_{\pm.12}$ & 0.42$_{\pm.06}$ \\
\textcolor[HTML]{8A7FDF}{$\beta$-VAE} & 9.7$_{\pm2.5}$ & 0.50$_{\pm.01}$ & 0.32/0.41 & 0.46$_{\pm.03}$ & 0.06$_{\pm.02}$ & 0.06$_{\pm.01}$ & 0.22$_{\pm.07}$ & 0.25$_{\pm.08}$ \\
\textcolor[HTML]{A79FEC}{AnnealedVAE} & 5.4$_{\pm1.4}$ & 0.77$_{\pm.05}$ & 0.36/0.43 & 0.40$_{\pm.10}$ & 0.06$_{\pm.05}$ & 0.04$_{\pm.03}$ & \textbf{0.35}$_{\pm.03}$ & 0.20$_{\pm<.01}$ \\
\midrule
noisy PCA & (653$_{\pm9}$) & (0) & 0.30/0.47 & 0.52 & -- & -- & -- & -- \\
\end{tabular}
\end{table}

\label{sec: arch res gt}
\textbf{Architecture ablation:} The Hadamard regularizer is not invariant to rotations of the latent space, so the flow must be free to choose its latent basis, which our twelve learnable $D\times D$ mixing matrices provide. Keeping all twelve learnable, only the final one, or none ($100$ epochs, three seeds each; App.~\ref{app: architecture ablation}) leaves the entropy ordering intact, but reproducibility drops from $46.7_{\pm7.0}$ stable coordinates per seed pair to $9.0_{\pm2.0}$ and $3.3_{\pm1.5}$, so dense learnable mixings are crucial.

\textbf{Scaling to higher resolution:} Fully connected MLPs become impractical beyond $D\approx10^4$. To show that the recipe scales in principle, we train a multiscale convolutional EOFlow on CelebA at $64\times64$ ($D=12288$, $\noisesig=0.1$; App.~\ref{app: conv64}) as a proof of concept. It orders its coordinates by entropy like the fully connected model and its leading archetypes show the same facial features. We use a smaller version of this convolutional flow on the higher-dimensional toy benchmarks below (App.~\ref{app: toy benchmarks}).

\textbf{Disentanglement benchmarks:} On dSprites and Shapes3D, EOFlows concentrate the entropy in a few coordinates and drop to the noise level at a sharp knee close to the number of ground-truth factors (Fig.~\ref{fig: toy spectra}, App.~\ref{app: toy benchmarks}). Against the state-of-the-art VAE-based methods \citep{locatello2019challengingcommonassumptionsunsupervised} they are competitive, measured by MIG \citep{NEURIPS2018_1ee3dfcd} and DCI \citep{eastwood2018a} (Table~\ref{tab: summary}).

\section{Limitations and conclusions}\label{sec: conclusion}

\leavevmode Unsupervised disentanglement needs an inductive bias, and EOFlows demonstrate that orthogonality of the decoder Jacobian is an effective one.
Made practical at image scale by a stochastic estimator and noise inflation, it yields interpretable features that can be ordered by manifold entropy, just as PCA orders its components by explained variance.
Keeping only the top features turns the bijective flow into an autoencoder whose bottleneck is chosen at inference time and which rivals dedicated autoencoders in rate--distortion without ever being trained to reconstruct.
Remarkably, many features affect only a few pixels, although neither architecture nor objective favors sparsity.

Unlike $\beta$-VAEs, which as injective models with a fixed bottleneck find only a handful of high-entropy features, the bijective EOFlows discover a rich set of orthogonal, sparse and reproducible features, an order of magnitude more than tuned $\beta$-VAE variants or nested-dropout flows, while remaining competitive on standard disentanglement benchmarks (Table~\ref{tab: summary}).

A fundamental hurdle on more complex data is topological mismatch: NFs struggle to map disconnected components (e.g.\ the classes of MNIST or ImageNet) to a unimodal Gaussian, because the Jacobian determinant must diverge between them \citep{loaiza-ganem2024deep}, whereas our regularizer keeps the coordinate grid, via its Jacobian, smooth and orthogonal.
EOFlows therefore do not yet disentangle multi-modal data globally; CelebA, with its single class ``face'', is largely benign in this respect.

\subsubsection*{Acknowledgments}
This work has been supported by Deutsche Forschungsgemeinschaft (DFG, German Research Foundation) under Germany's Excellence Strategy EXC-2181/1 -- 390900948 (the Heidelberg STRUCTURES Cluster of Excellence).
It is also supported by the Carl-Zeiss-Stiftung (Projekt P2021-02-001 ``Model-based AI'').
DG acknowledges support by the German Federal Ministry of Education and Research (BMBF) (project EMUNE/031L0293A).
The authors acknowledge support by the state of Baden-W\"urttemberg through bwHPC and the German Research Foundation (DFG) through grant INST 35/1597-1 FUGG.
We ran experiments on NVIDIA A40 GPUs of the bwForCluster Helix.
UK thanks the Klaus Tschira Stiftung for their support via the SIMPLAIX project.

\bibliography{main_bib}
\bibliographystyle{eof_preprint}

\appendix
\counterwithin{figure}{section}
\counterwithin{table}{section}
\counterwithin{equation}{section}
\counterwithin{lstlisting}{section}

\newpage

\newpage

\section{Stochastic and geometric disentanglement are distinct concepts} \label{app: stochastic and geometric disentanglement}

The term ``disentanglement'' intuitively refers to representations where each feature has a distinct semantic meaning that is largely independent of the values of the other features, and it can be formalized in at least two ways.
A representation is {\em statistically} (or stochastically) {\em disentangled} if the generated distribution factorizes, i.e.\ can be written as a product of single-variable functions $q\big(\X\! =\! \g(\z)\big) = \prod\nolimits_{i=1}^D \rho_i(z_i)$.
When $\g(\z)$ is a NF with separable latent prior, this is achieved if the determinant of the pull-back metric factorizes into a product of 1-dimensional functions
\begin{equation}
    \det \met(\z) = \prod\nolimits_{i=1}^D \psi_i(z_i)
\end{equation}
whereas geometric disentanglement requires the metric to be diagonal
\begin{equation}
    \met(\z) = \diag\big(\met(\z)_{11}, ..., \met(\z)_{DD}\big)
\end{equation}
If all diagonal elements of the metric are equal, $\met(\z)_{11} = ...= \met(\z)_{DD}$, one speaks of a {\em conformal} map.
When the orthogonal coordinates are suitably oriented, the nonlinear map resembles a localized PCA and de-correlates the features in a neighborhood of each point.
We present three simple examples to show that maps can fulfill either of these conditions individually or both simultaneously.
Consider 2-dimensional polar coordinates with $\z=[r, \phi]$ and decoder
\begin{equation}
    x_1 = r\,\cos(\phi) \qquad  \qquad  x_2 = r\,\sin(\phi)
\end{equation}
Calculation of its metric shows that polar coordinates fulfill both conditions
\begin{equation}
        \met_\text{polar}(\z) = 
    \begin{pmatrix}
    1 & 0 \\
    0 & r^2
    \end{pmatrix}  \qquad \qquad  \det \met_\text{polar}(\z) = r^2
\end{equation}
Parabolic coordinates are defined by the decoder 
\begin{equation}
    x_1 = \frac{z_1^2 - z_2^2}{2} \qquad  \qquad  x_2 = z_1 z_2
\end{equation}
The metric is also diagonal (indeed even conformal), but the determinant does not factorize
\begin{equation}
    \met_\text{parabolic}(\z) = \big(z_1^2 + z_2^2\big)\cdot\eye \qquad\qquad\det \met_\text{parabolic}(\z) = \big(z_1^2 + z_2^2\big)^2
\end{equation}
Finally, the decoder of log-linear coordinates is
\begin{equation}
    x_1 = e^{z_1+z_2} \qquad  \qquad  x_2 = e^{-z_1+z_2}
\end{equation}
Here, the metric is not diagonal, but the determinant factorizes nonetheless
\begin{equation}\label{eq:log-linear-map}
    \met_\text{loglin}(\z) = 2\, e^{2z_2} \begin{pmatrix}
        \cosh(2z_1) & \sinh(2z_1) \\
        \sinh(2z_1) & \cosh(2z_1)
    \end{pmatrix}\qquad\qquad \det \met_\text{loglin}(\z) = 4\,e^{4z_2}
\end{equation}
These examples demonstrate that {\em neither form of disentanglement implies the other}.

\textbf{Tolerable ambiguities.} Structural identifiability \citep{nelsonstatistical} asks whether the learned features reproduce the independent factors of the true generative process (TGP), statistical identifiability only whether repeated training converges to the same features -- a crucial relaxation when the TGP is unknown.
Which ambiguities are considered tolerable depends on the application.
Most commonly, one allows permutations and element-wise reparameterizations of the latent space -- it does not matter how the features are indexed and in what units they are measured.
Fixing the latent prior as a standard normal therefore loses no generality: any other separable prior can be obtained by univariate inverse probability transforms.

\newpage

\section{Numerical estimation of the Hadamard loss and its cost}\label{app: stochastic estimate of L_TC}

NF architectures are typically designed to output the log-determinant of the Jacobian along a forward pass, which makes maximum likelihood training tractable.

\subsection{Full loss estimate}
The EOFlow objective reads
\begin{equation}
     \L_\EO(\x) = \L_\ML(\x) + \lam \cdot \L_\text{H}(\x)\end{equation}
with
\begin{align}
    \L_\text{H}(\x) &= \sum_{i=1}^D \log\big\| \J_i(\f(\x))\big\|_2 - \log\big| \J(\f(\x))\big| \\
    &= \sum_{i=1}^D \frac{1}{2} \log\left( \sum_{k=1}^D {\J_{ki}(\f(\x))}^2 \right) - \log\big| \J(\f(\x))\big|
\end{align}

The log-determinant $\log\big| \J(\f(\x))\big|$ is already available from the likelihood term, so the cost lies in the $D$ column norms $\|\J_i\|_2$: computing them exactly requires the full $D\times D$ Jacobian, i.e.\ $D$ Jacobian-vector products (jvp) per sample, one for each unit vector $\hat{\boldsymbol{e}}_i$.

\subsection{Stochastic estimate used in all experiments}
For a batch $\{\x^j\}_{j=1}^{B}$ of any size $B$, with latent codes $\z^j=\f(\x^j)$, we draw one column index per sample, $i_j\sim\mathcal{U}\{1,\dots,D\}$, independently across samples and of the data, and evaluate a single jvp per sample, $\J_{i_j}(\z^j)=\texttt{jvp}(\g,\z^j,\hat{\boldsymbol{e}}_{i_j})$; the $B$ jvps form one batched forward-mode pass through the decoder.
The batch estimate of the Hadamard loss is
\begin{equation}\label{eq: stochastic L_H estimate}
    \widehat{\L}_\text{H}\big(\x^{1:B}\big) = \frac{1}{B}\sum_{j=1}^{B}\Big( D\,\log\big\|\J_{i_j}(\z^j)\big\|_2 - \log\big|\J(\z^j)\big|\Big),
\end{equation}
and the training loss is $\frac{1}{B}\sum_{j}\L_\ML(\x^j) + \lam\,\widehat{\L}_\text{H}(\x^{1:B})$, so $\lam$ enters only once.
As in the unbiased estimates of \citet{pmlr-v162-cunningham22a}, one random column per sample suffices; we write the weight $D$ explicitly so that $\lam$ multiplies the full sum over columns.

\textbf{Unbiasedness.} Since $\E_{i_j}\big[D\log\|\J_{i_j}(\z^j)\|_2\big]=\sum_{i=1}^{D}\log\|\J_i(\z^j)\|_2$, the expectation of (\ref{eq: stochastic L_H estimate}) over the index draws is the exact batch loss $\frac{1}{B}\sum_{j}\L_\text{H}(\x^j)$ for every batch, whatever its size $B$ relative to $D$. The index distribution does not depend on the network parameters, so the same holds for the gradient.
Non-negativity on average follows from unbiasedness: the expectation is the exact loss, which is non-negative by Hadamard's inequality, although a single estimate can be negative.
A batch need not contain every column. The indices are drawn with replacement, so the same column can be drawn for several samples of a batch while other columns are not drawn at all; this affects the variance but not the expectation, and no reweighting by the number of times a column was drawn is required.

\textbf{Variance.} The estimate needs $B$ jvps per step, whatever $D$ is (their cost is measured in App.~\ref{app: hadamard cost}). The price is variance: conditionally on the batch, the variance of (\ref{eq: stochastic L_H estimate}) is $\frac{D^2}{B^2}\sum_{j}\Var_i\big[\log\|\J_i(\z^j)\|_2\big]$.
It is set by how much the log column norms differ across $i$ -- essentially the width of the entropy spectrum, which entropy ordering makes large -- and not by $\lam$, while in the training loss it is multiplied by $\lam^2$.
Fig.~\ref{fig: penalty noise} tests this on the training runs of the $\lam$ sweep (Fig.~\ref{fig: lambda dial}): the step-to-step variance of the logged estimate agrees with the formula, evaluated at the final weights of each run, to within $5\,\%$ in all $20$ runs, and the variation of the exact loss between batches adds at most $2\,\%$ (a).
The noise is thus nearly independent of $\lam$, but the penalty it estimates shrinks as the flow becomes orthogonal, so the relative variance of one step's estimate grows from $0.014$ at $\lam=0.01$ to $0.8$ at $\lam=1$ and $250$--$300$ at $\lam=10$ (b), where a single estimate of $\L_\text{H}/D$ reads $0.002\pm0.030$.
Multiplied by $\lam^2$, the penalty's share of the per-step variance of the training loss overtakes that of the likelihood at $\lam=0.1$ and exceeds it about $80$-fold at $\lam=1$ and $4000$-fold at $\lam=10$ (c).

\begin{figure}[t]
    \centering
    \includegraphics[width=\linewidth]{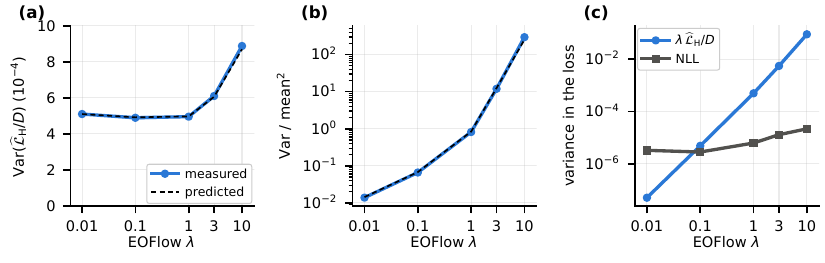}
    \caption{\textbf{Variance of the stochastic Hadamard estimate} along the $\lam$ sweep (last $50$ epochs of training; lines are means over $3$--$6$ seeds per $\lam$, shaded bands $\pm$ one standard deviation, mostly narrower than the line). \textbf{(a)} Per-step variance of the estimate of $\L_\text{H}/D$ in the logged training loss (blue) and as predicted by the variance formula at the same weights, from full Jacobians at $64$ noise-inflated training images (dashed). \textbf{(b)} The same variance relative to the squared mean, i.e.\ to the penalty itself. \textbf{(c)} Per-step variance of the two terms of the training loss.}
    \label{fig: penalty noise}
\end{figure}

\begin{remark}[Lower variance for $B\ge D$, optional]
When $B\ge D$, the index draw can in addition be stratified so that every column enters each batch: $D$ of the samples receive a random permutation of $1,\dots,D$, the remaining $B-D$ samples draw their index uniformly, and each column is averaged over the $m_i\ge1$ samples $S_i=\{j: i_j=i\}$ that drew it,
\begin{equation}
    \widetilde{\L}_\text{H}\big(\x^{1:B}\big) = \sum_{i=1}^{D}\frac{1}{m_i}\sum_{j\in S_i}\log\big\|\J_i(\z^j)\big\|_2 - \frac{1}{B}\sum_{j=1}^{B}\log\big|\J(\z^j)\big| .
\end{equation}
Because the index pattern is drawn independently of the data, this estimate is unbiased as well, and it removes the spread across columns from the variance, leaving only the variation of each column across samples.
On a synthetic Gaussian bed ($D=32$, $B=256$, eight trained 8-block flows, $200$ batches at fixed weights) its standard deviation around the exact batch value is 24--380$\times$ lower than that of (\ref{eq: stochastic L_H estimate}); the gap is smallest (24$\times$) for the most nonlinear flow.
This refinement is an option, not a requirement: (\ref{eq: stochastic L_H estimate}) is unbiased for every $B$, and all our experiments use it (on CelebA $28\times28$: $B=512$, $D=2352$).
\end{remark}

\subsection{Cost}\label{app: hadamard cost}
Table~\ref{tab: hadamard cost} times the estimator on two models: the fully connected CelebA model ($D=2352$, 12 coupling blocks with learnable QLU mixings, 244.9M parameters) and a multiscale convolutional EOFlow for $64\times64$ CelebA images ($D=12288$, 69.1M parameters; App.~\ref{app: conv64}).
Both were timed in the same process on one GPU, with randomly initialized weights and random inputs (median of seven timings).
The likelihood-only step comprises forward pass, loss, backward pass and optimizer update; the step with $\L_\text{H}$ adds the batched jvp and the backward pass through it.
\begin{center}
\captionof{table}{Wall-clock cost of the stochastic Hadamard estimate, $B=128$ unless stated. The multiplier is the step with $\L_\text{H}$ divided by the likelihood-only step. The training runs use $B=512$ for the fully connected model and $B=128$ for the convolutional one.}\label{tab: hadamard cost}
\small
\begin{tabular}{lrr}
\toprule
 & fully connected & convolutional \\
 & $D=2352$, 244.9M & $D=12288$, 69.1M \\
\midrule
decoder forward pass & 12.3 ms & 32.0 ms \\
decoder jvp (forward mode) & 53.9 ms & 137.2 ms \\
jvp / decoder forward pass & $4.39\times$ & $4.29\times$ \\
optimizer update alone & 25.2 ms & 7.6 ms \\
likelihood-only step & 71.2 ms & 98.5 ms \\
step with $\L_\text{H}$ & 177.9 ms & 415.1 ms \\
\textbf{multiplier} & \textbf{2.50}$\times$ & \textbf{4.22}$\times$ \\
\midrule
multiplier at $B=32 / 128 / 512$ & 2.61 / 2.50 / 2.49 & \\
multiplier at $B=32 / 128 / 256$ & & 4.54 / 4.22 / 4.08 \\
\bottomrule
\end{tabular}
\end{center}
At $B=128$ a forward-mode jvp costs about the same multiple of a decoder pass on both models ($4.39\times$ and $4.29\times$; 2.94--4.74$\times$ across the batch sizes of the table), so the type of layer does not explain why the step with $\L_\text{H}$ is $4.22\times$ slower on the convolutional model but only $2.50\times$ on the fully connected one.
The likely reason is that more of the fully connected step is parameter-bound work that $\L_\text{H}$ leaves unchanged: its optimizer update alone takes 25.2 ms (35\% of the likelihood-only step, against 7.6 ms or 8\%), and its twelve $2352\times2352$ mixing matrices are rebuilt in every step.
The two models differ in $D$, parameter count and architecture at once, so this is an accounting of the measured times, not a controlled comparison. Both multipliers decrease slightly as the batch grows.
The learnable mixings themselves are a moderate share of the step: in a separate benchmark of the fully connected model (forward pass, loss and backward pass without the optimizer update, $B=128$ and $256$, exclusive GPU), making all twelve mixings learnable QLU matrices cost about 25 ms of a $\approx 145$ ms step ($\approx 17\%$) compared with keeping all but the last one or three frozen at random orthogonal matrices, and 8--15 ms compared with eleven frozen dense mixings and a learnable last one.
At the training batch size $B=512$ this step takes $\approx 188$ ms; the share of the mixings at that batch size was not measured.

\subsection{Python implementation}\label{sec:Python Implementation}
\begin{lstlisting}[caption={Minimal PyTorch implementation of the EOFlow loss with the stochastic estimate (\ref{eq: stochastic L_H estimate}), as used in all our experiments}, basicstyle=\footnotesize\ttfamily]
import math
import torch
from torch.autograd.forward_ad import dual_level, make_dual, unpack_dual

def eoflow_loss(flow, x, lam, sigma):
    # flow(x) -> (z, log|det df/dx|),  flow(z, rev=True) -> (x, log|det dg/dz|)
    x = x + sigma * torch.randn_like(x)             # noise inflation, [B, D]
    B, D = x.shape
    z, logdet_f = flow(x, rev=False)                # encoder
    nll = 0.5 * (z**2).sum(1) + 0.5 * D * math.log(2 * math.pi) - logdet_f

    i = torch.randint(0, D, (B,), device=x.device)  # one column per sample
    e = torch.nn.functional.one_hot(i, D).to(z.dtype)
    with dual_level():                              # batched jvp, decoder
        x_dual, _ = flow(make_dual(z, e), rev=True)
        _, J_i = unpack_dual(x_dual)                # columns J_{i_j}(z_j)
    log_norm = 0.5 * torch.log((J_i.flatten(1)**2).sum(1) + 1e-12)

    L_H = D * log_norm + logdet_f                   # log|det J| = -logdet_f
    return (nll + lam * L_H).sum() / (B * D)        # per sample and dim
\end{lstlisting}

\newpage

\section{Linearized Normalizing Flows and PCA} \label{app: Linearized Normalizing Flows and PCA}

\subsection{Setup}

We assume a Linearized Normalizing Flow where we instantiate the decoder as an affine transformation
\begin{equation}
    \g(\z) = A\z + b
\end{equation}
$A \in \mathbb{R}^{D \times D}$ is an invertible ``mixing'' matrix and $b \in \mathbb{R}^{D}$ the ``bias'' term.
The decoder Jacobian $\J$ is simply $A$.\\
For completeness, the encoder becomes the inverse of the affine transformation:
\begin{equation}
    \f(\x) = A^{-1} \left( \x - b \right)
\end{equation}

Assume that the data distribution has mean $\mu$ and covariance matrix $\Sigma$ (of finite second moment).
We assume that $\Sigma$ has the following eigendecomposition
\begin{equation}
    \Sigma = Q  \Lambda Q^T
\end{equation}
where $Q=[Q_1, \dots, Q_D]$ is an orthogonal matrix and the columns are orthonormal eigenvectors of $\Sigma$. The diagonal matrix $\Lambda = \text{diag}(\lambda_1, \dots, \lambda_D)$ collects the eigenvalues, which we assume are ordered as
\begin{equation}
    \lambda_1 > \lambda_2 > \dots > \lambda_D > 0
\end{equation}

\subsection{PCA}

PCA can be viewed as a linear normalizing flow that matches the first two moments of the data.
In particular, PCA first centers the data by choosing the bias to match the data mean
\begin{equation}
    b \equiv \mu
\end{equation}
It then aligns the latent coordinates with the principal directions $Q_i$ and scales them according to the data variance $\lambda_i$. In our affine decoder setting, this corresponds to choosing
\begin{equation}
    A \equiv Q \Lambda^{\frac{1}{2}}
\end{equation}
With this choice, $A A^T = Q\Lambda Q^T = \Sigma$, i.e. the decoder maps a standard normal latent $\z \sim \mathcal{N}(\mathbf{0},\I)$ to a Gaussian in data space with the same covariance as the data distribution. Equivalently, the encoder
\begin{equation}
    \f(\x) = A^{-1}(\x - b) = \Lambda^{-\frac{1}{2}} Q^T (\x - \mu)
\end{equation}
is the usual PCA whitening transform: it rotates into the orthonormal eigenbasis and rescales each component by the inverse standard deviation $\lambda_i^{-1/2}$.

\subsection{Convergence of Maximum-Likelihood}

Let us write out the expectation of the ML objective
\begin{equation}
    \Expt{\x \sim p^{*}}{\L_\ML(\x)} = \Expt{\x \sim p^{*}}{-\log(\q(\X=\x))} = \Expt{\x \sim p^{*}}{\frac{1}{2} \left|\f(\x)\right|_2^2 + \log\left|\J(\f(\x))\right| + \frac{D}{2}\log(2\pi)}
\end{equation}
Inserting the linear encoder/decoder definition, we abbreviate $\L_\ML(A,b) \coloneq \Expt{\x \sim p^{*}}{\L_\ML(\x)}$ as a minimization objective under $A$ and $b$:
\begin{align}
    \L_\ML(A,b)
    &= \Expt{\x \sim p^{*}}{\frac{1}{2} \left(A^{-1} ( \x - b ) \right)^T \left(A^{-1} ( \x - b )\right) + \log\left|A\right| + \frac{D}{2}\log(2\pi)} \\
    &= \Expt{\x \sim p^{*}}{ \frac{1}{2} ( \x - b )^T (A A^T)^{-1} ( \x - b )} + \log\left|A\right| + \frac{D}{2}\log(2\pi) \\
    &= \Expt{\x \sim p^{*}}{ \frac{1}{2} \text{tr}\left((A A^T)^{-1} ( \x - b ) ( \x - b \right)^T ) }
    + \log\left|A\right| + \frac{D}{2}\log(2\pi) \\
    &= \frac{1}{2} \text{tr}\left( (A A^T)^{-1} \Expt{\x \sim p^{*}}{ ( \x - b ) ( \x - b )^T } \right) 
    + \log\left|A\right| + \frac{D}{2}\log(2\pi)
\end{align}
where
\begin{align} \label{eq: expectation of quadratic term}
    \Expt{\x \sim p^{*}}{ ( \x - b ) ( \x - b )^T }
    &= \Expt{\x \sim p^{*}}{ \x \x^T - \x b^T - b \x^T + b b^T } = \Sigma + \mu \mu^T - \mu b^T - b \mu^T + b b^T \\
    &= \Sigma + (\mu - b) (\mu - b)^T
\end{align}
such that
\begin{align}
    \L_\ML(A,b)
    &= \frac{1}{2} \text{tr}\left(  (A A^T)^{-1} \left(\Sigma + (\mu - b) (\mu - b)^T\right) \right) 
    + \log\left|A\right| + \frac{D}{2}\log(2\pi) \\
    &= \frac{1}{2} \text{tr}\left(  (A A^T)^{-1} \Sigma \right) + \frac{1}{2} \text{tr}\left( (\mu - b)^T (A A^T)^{-1} (\mu - b) \right) 
    + \log\left|A\right| + \frac{D}{2}\log(2\pi) \\
    &= \frac{1}{2} \left[ \text{tr}\left(  (A A^T)^{-1} \Sigma \right) + (\mu - b)^T (A A^T)^{-1} (\mu - b) + 2 \log\left|A\right| + D\log(2\pi) \right]
\end{align}

\begin{theorem}
    \label{thm: linear ML}The loss $\L_\ML(A,b)$ attains its global minimum for $b = \mu$ and $A A^T = \Sigma$.
\end{theorem}
\begin{proof}
    Let $C \coloneq AA^T \succ 0$ be a symmetric positive definite matrix. The loss is:
    \begin{equation}
        \L_\ML(A,b)
        = \frac{1}{2} \left[ \text{tr}\left( C^{-1} \Sigma \right) + (\mu - b)^T C^{-1} (\mu - b) + \log\left|C\right| \right] + \text{const.}
    \end{equation}
    Since $C^{-1} \succ 0$, the quadratic term $(\mu - b)^T C^{-1} (\mu - b)$ is non-negative and vanishes only at $b = \mu$.\\
    Substituting into the loss,
    \begin{equation}
        \L_\ML(A,b)
        = \frac{1}{2} \left[ \tr{C^{-1} \Sigma} + \log|C| \right]  + \text{const.}
    \end{equation}
    Let $M \coloneq C^{-\frac{1}{2}} \Sigma C^{-\frac{1}{2}} \succ 0$. We can write $\tr{M} = \tr{C^{-1}\Sigma}$ and $\log|M| = \log|\Sigma| - \log|C|$ and substitute
    \begin{equation}
        \L_\ML(A,b)
        = \frac{1}{2} \left[ \tr{M} - \log|M| + \log|\Sigma| \right]  + \text{const.}
    \end{equation}
    By the inequality
    \begin{equation}
        \tr{M} - \log|M| \geq D
    \end{equation}
    with equality iff $M = \I_D$, we obtain
    \begin{equation}
        \L_\ML(A,b)
        \geq \frac{1}{2} \left[ D + \log|\Sigma| \right]  + \text{const.}
    \end{equation}
    Thus equality holds iff
    \begin{equation}
        M = \I_D \Longleftrightarrow C^{-\frac{1}{2}} \Sigma C^{-\frac{1}{2}} = \I_D \Longleftrightarrow C = \Sigma \Longleftrightarrow AA^T = \Sigma
    \end{equation}
\end{proof}

\subsection{Convergence of the Hadamard regularizer}

Let us write the expectation of the pointwise manifold entropy of a single dimension $i$
\begin{equation}
    \Expt{\x \sim p^{*}}{\L_i(\x)} = \Expt{\x \sim p^{*}}{-\log(q_i(\X_i=\x))} = \Expt{\x \sim p^{*}}{\frac{1}{2} \left|\f_i(\x)\right|^2 + \log\left|\J_i(\f(\x))\right| + \frac{1}{2}\log(2\pi)}
\end{equation}

As a prerequisite we need to evaluate the $i$-th entry of the encoder output:
\begin{equation}
    \f_i(\x) = (A^{-1} \left( \x - b \right))_i
\end{equation}
which can be explicitly written by the SVD decomposition of A
\begin{equation}
    A = USV^T , \quad A^{-1} = V S^{-1} U^T
\end{equation}
The diagonal matrix $S$ has only positive entries (convention without loss of generality).
We denote the $i$-th row vector of $V$ by $V_{i,*}$ such that
\begin{equation}
    \f_i(\x) = V_{i,*} S^{-1} U^T \left( \x - b \right)
\end{equation}
The $i$-th column vector of the decoder Jacobian $\J_i$ is the $i$-th column vector of $A$, which we denote as $(A)_{*,i} \coloneq A_i$:
\begin{equation}
    A_i = (U S V^T)_{*,i} = U S (V_{i,*})^T
\end{equation}

For shorter notation we will abbreviate the $i$-th row vector $V_{i,*}$ simply as $V_i$.

Inserting the linear encoder/decoder definition, we abbreviate $\L_i(A,b) \coloneq \Expt{\x \sim p^{*}}{\L_i(\x)}$ as a minimization objective under $A$ and $b$:
\begin{align} \label{eq: L_i linearized 1}
    \L_i(A,b)
    &= \Expt{\x \sim p^{*}}{\frac{1}{2} \left((\f_i(\x))^T (\f_i(\x))\right) + \log\left|A_i\right| + \frac{1}{2}\log(2\pi)} \\
    &= \Expt{\x \sim p^{*}}{\frac{1}{2} (V_i S^{-1} U^T (\x - b))^T (V_i S^{-1} U^T (\x - b)) + \log\left|U S V_i^T\right| + \frac{1}{2}\log(2\pi)} \\
    &= \Expt{\x \sim p^{*}}{\frac{1}{2} (\x - b)^T U S^{-1} V_i^T V_i S^{-1} U^T (\x - b)} \nonumber\\
    &\qquad + \frac{1}{2} \log\left|\det((U S {V_i}^T)^T (U S {V_i}^T))\right| + \frac{1}{2}\log(2\pi) \\
    &= \Expt{\x \sim p^{*}}{\frac{1}{2} \text{tr} \left( U S^{-1} V_i^T V_i S^{-1} U^T (\x - b) (\x - b)^T\right)} \nonumber\\
    &\qquad + \frac{1}{2} \log\left|\det( V_i S^2 {V_i}^T)\right| + \frac{1}{2}\log(2\pi) \\
    &= \frac{1}{2} \text{tr} \left( U S^{-1} V_i^T V_i S^{-1} U^T \Expt{\x \sim p^{*}}{ (\x - b) (\x - b)^T}\right) \nonumber\\
    &\qquad + \frac{1}{2} \log\left|\det( V_i S^2 {V_i}^T)\right| + \frac{1}{2}\log(2\pi)
\end{align}
Inserting the solution eq.(\ref{eq: expectation of quadratic term})
\begin{align} \label{eq: L_i linearized 2}
    \L_i(A,b)
    &= \frac{1}{2} \left[ \text{tr}\left( U S^{-1} V_i^T V_i S^{-1} U^T \left(\Sigma + (\mu - b) (\mu - b)^T\right) \right) + \log\left|\det( V_i S^2 {V_i}^T)\right| + \log(2\pi) \right]
\end{align}

Per sample, the prior terms cancel in the difference of the two losses, so the Hadamard regularizer (\ref{eq: def Total Disentanglement MAIN}) is $\L_\text{H}(\x) = \sum_i \L_i(\x) - \L_\ML(\x)$, and its expectation is
\begin{equation}
    \L_\text{H}(A,b) \coloneq \Expt{\x \sim p^{*}}{\L_\text{H}(\x)} = \sum_i \L_i(A, b) - \L_\ML(A, b)
\end{equation}
Inserting the previous results
\begin{align}
    \L_\text{H}(A,b)
    &= \frac{1}{2} \sum_{i=1}^{D} \Big[ \text{tr}\left( U S^{-1} V_i^T V_i S^{-1} U^T \left(\Sigma + (\mu - b) (\mu - b)^T\right) \right) \nonumber\\
    &\hspace{6em} + \log\left|\det( V_i S^2 {V_i}^T)\right| + \log(2\pi) \Big] \\
    & \quad - \frac{1}{2}\left[ \text{tr}\left(  (A A^T)^{-1} \left(\Sigma + (\mu - b) (\mu - b)^T\right) \right) + 2\log\left|A\right| + D\log(2\pi) \right]
\end{align}
Inserting the SVD decomposition of $A$, we can simplify the determinant of $A$
\begin{equation}
    \left|\det(A)\right|=\det(S)=\prod_i s_i
\end{equation}
where $s_i$ are the strictly positive singular values.
We can also simplify the second determinant expression by observing that $V_i S^2 {V_i}^T$ is a scalar and strictly positive.
\begin{align}
    \L_\text{H}(A,b)
    &= \frac{1}{2} \sum_{i=1}^{D} \Big[ \text{tr}\left( U S^{-1} V_i^T V_i S^{-1} U^T \left(\Sigma + (\mu - b) (\mu - b)^T\right) \right) \nonumber\\
    &\hspace{6em} + \log\left( V_i S^2 {V_i}^T\right) + \log(2\pi) \Big] \\
    & \quad \ - \frac{1}{2}\left[ \text{tr}\left(  U S^{-2} U^T \left(\Sigma + (\mu - b) (\mu - b)^T\right) \right) + 2 \log\left(\prod_{i=1}^D s_i\right) + D\log(2\pi) \right] \\
    &= \frac{1}{2} \text{tr}\left( \Big[ \sum_{i=1}^{D} U S^{-1} V_i^T V_i S^{-1} U^T - U S^{-2} U^T \Big] \left(\Sigma + (\mu - b) (\mu - b)^T\right) \right) \\
    & \quad \ + \left[ \frac{1}{2} \sum_{i=1}^{D} \log (V_i S^2 {V_i}^T) - \log\left(\prod_{i=1}^D s_i\right) \right] \\
    &= \frac{1}{2} \text{tr}\left(  U S^{-1} \underbrace{\left[\sum_{i=1}^{D} V_i^T V_i - \I_D \right]}_{\equiv 0} S^{-1} U^T  \left(\Sigma + (\mu - b) (\mu - b)^T\right) \right) \\
    & \quad \ + \sum_{i=1}^{D} \left[ \frac{1}{2} \log (V_i S^2 {V_i}^T) - \log(s_i) \right] \\
    &= \frac{1}{2} \sum_{i=1}^{D} \log\left(\frac{V_i S^2 {V_i}^T}{s_i^2}\right)
\end{align}
Since $V$ is orthogonal, $\sum_{i} V_i^T V_i = V^T V = \I_D$, and the trace term vanishes.
\begin{theorem}\label{thm: linear H}
    The loss $\L_\text{H}(A,b)$ attains its global minimum $0$ if $V$ is a permutation matrix up to signs and, if the singular values of $A$ are distinct, only then.
\end{theorem}
\begin{proof}
    The loss is:
    \begin{equation}
        \L_\text{H}(A,b)
        = \frac{1}{2} \sum_{i=1}^{D} \log\left(\frac{V_i S^2 {V_i}^T}{s_i^2}\right)
    \end{equation}
    Since $V$ is orthogonal, the coefficients $V_{ij}^2$ form a probability distribution for each $i$. By Jensen's inequality applied to the concave function log:
    \begin{equation}
        \log(V_i S^2 V_i^T) = \log\left(\sum_j V_{ij}^2 s_j^2\right) \geq \sum_j V_{ij}^2 \log s_j^2
    \end{equation}
    Summing over $i$ and using $\sum_i V_{ij}^2 = 1$ yields
    \begin{equation}
        \sum_i \log(V_i S^2 V_i^T) \geq \sum_i \log s_i^2,
    \end{equation}
    i.e.\ $\L_\text{H}(A,b) \geq 0$. If $V$ is a permutation matrix up to signs, then $V_i S^2 V_i^T = s_{\pi(i)}^2$ for a permutation $\pi$, and equality holds. Conversely, since the logarithm is strictly concave, equality requires $s_j^2$ to be constant on the support of every row distribution $V_{ij}^2$; for distinct singular values, every row of $V$ then has exactly one non-zero entry, $\pm1$.
\end{proof}
\begin{corollary}\label{cor: linear EOFlow}
    For every $\lam > 0$, the expected loss $\L_\ML(A,b) + \lam\,\L_\text{H}(A,b)$ of a linear EOFlow attains its global minimum exactly at $b = \mu$ and $A = Q\Lambda^{\frac{1}{2}} P$, where $P$ is a permutation matrix up to signs, i.e.\ at the PCA solution up to the tolerable ambiguities.
\end{corollary}
\begin{proof}
    $A = Q\Lambda^{\frac{1}{2}} P$ has the SVD $U = Q$, $S = \Lambda^{\frac{1}{2}}$, $V = P^T$ and thus minimizes both terms at once (Theorems~\ref{thm: linear ML} and~\ref{thm: linear H}). Hence the minimum of the sum is the sum of the two minima, and every minimizer of the sum minimizes both terms. Then $b = \mu$ and $AA^T = US^2U^T = \Sigma$, so the singular values of $A$ are the square roots of the distinct eigenvalues of $\Sigma$, and $V$ is a permutation matrix up to signs by Theorem~\ref{thm: linear H}. The columns of $A = USV^T$ are therefore $\pm s_j U_j$, and $US^2U^T = Q\Lambda Q^T$ with distinct eigenvalues implies $U_j = \pm Q_k$ and $s_j^2 = \lambda_k$ for a permutation $j \mapsto k$.
\end{proof}
Sorting the latent dimensions by decreasing $s_i$ gives the PCA decoder $A = Q\Lambda^{\frac{1}{2}}$ up to signs: the latent axes are the principal directions, scaled by their standard deviations.

\subsection{Manifold Entropic Metrics in the linear regime}

In the linear setting one can compute manifold entropic metrics analytically. The manifold entropy of a single latent dimension $i$ reads:
\begin{align}
    H_i &= \frac{1}{2}(1 + \log(2\pi)) + \Expt{\z \sim p(\z)}{\log\|\J_i(\z)\|_2} = \frac{1}{2}(1 + \log(2\pi)) + \log\|A_i\|_2 \\
    &= \frac{1}{2}(1 + \log(2\pi)) + \frac{1}{2} \log\left|\det( V_i S^2 {V_i}^T)\right|
\end{align}
For PCA, $V = \I_D$ and this simplifies to
\begin{equation}
    H_i = \frac{1}{2}(1 + \log(2\pi)) + \log(s_i)
\end{equation}
This shows that for PCA, sorting by explained variance or manifold entropy is equivalent.

\newpage

\section{Noise inflation} \label{app: Noise inflation}

Throughout the paper, every training image is inflated with fresh isotropic Gaussian noise of standard deviation $\noisesig=0.05$, redrawn at every step. This appendix explains why the noise is needed (App.~\ref{app: noise why}), what happens with less of it on CelebA (App.~\ref{app: noise less}), which trends a grid of smaller models and a Gaussian bed with known axes show (App.~\ref{app: noise grid} and~\ref{app: noise gauss}), and why the coordinates that model the noise cannot be reproduced across runs (App.~\ref{app: noise detail}).

\subsection{Why noise is needed}\label{app: noise why}

\textbf{Memorization.} With finite data, maximum likelihood rewards a density concentrated on the training points. A finite training set is a zero-dimensional manifold on which the likelihood is unbounded, so a flexible flow memorizes the training set unless something prevents it (Sec.~\ref{sec:noise}).

\textbf{Noise as a kernel density estimate.} Because the noise is redrawn at every step, training on noise-inflated images is formally the same as fitting the flow to the Gaussian kernel density estimate $\frac1N\sum_n\mathcal{N}(\x_n,\noisesig^2\I_D)$ of the training set. Its density is smooth and bounded, so the likelihood of the flow cannot diverge on the training points.
The identity does not prescribe a noise level. We chose $\noisesig$ empirically (App.~\ref{app: noise less} and~\ref{app: noise grid}), and more noise is not better, since it hides every direction of the data whose variance lies below it.

\textbf{A regularizer without a preferred basis.} The usual remedies against memorization are less capacity or architectural biases, and the latter, such as convolutions, presuppose the pixel basis.
Isotropic noise and the Hadamard penalty presuppose no basis: both are unchanged by a rotation of the data space, so the flow can stay fully connected and over-parametrized, as the experiments in a shuffled or rotated pixel basis require (App.~\ref{app: rotated basis}).
Recent normalizing flows for images likewise found noise augmentation essential and set its level empirically \citep{zhai2025normalizing}.

\subsection{Less noise on CelebA}\label{app: noise less}

We trained the EOFlow of Sec.~\ref{sec: setup} ($\lam=1$, batch size $512$, learning rate $6\cdot10^{-4}$, $300$ epochs) three times at each of two lower noise levels and left everything else unchanged.

\textbf{Training diverges.} At $\noisesig=0.01$, two of the three runs diverged, after about $70$ and $90$ epochs, and the third completed all $300$ epochs.
At $\noisesig=0.002$, about twice the standard deviation of uniform dequantization noise for 8-bit images ($1/(256\sqrt{12})\approx0.0011$), all three runs diverged within the first $34$ epochs.
At $\noisesig=0.05$, one of the eight runs we trained diverged (App.~\ref{app: Training setup}).

\textbf{The surviving run overfits.} We compare the run that completed at $\noisesig=0.01$ with the three runs at $\noisesig=0.05$ after the same $300$ epochs (mean $\pm$ standard deviation over the three).
Its gap between held-out and training likelihood is $0.214$ nats per dimension against $0.035_{\pm.002}$, and it grows throughout training ($0.050$, $0.103$, $0.141$, $0.170$, $0.194$ and $0.214$ after $50$, $100$, \dots, $300$ epochs), while its held-out likelihood is best after $100$ epochs and declines afterwards.
Its decoder is also less orthogonal, with a Hadamard gap $\L_\text{H}/D$ of $0.040$ against $0.014$.

\textbf{Its samples are worse.} Through the best bottleneck ($C=200$) its samples reach FID $27.5$ against $24.0_{\pm0.2}$, and from the full prior after a Tweedie step FID $44.7$ against $34.7_{\pm0.4}$.

\textbf{Its features are shared less.} It shares $33.3_{\pm1.2}$ stable coordinates with each of the three runs at $\noisesig=0.05$, which share $121_{\pm5}$ among themselves. Since no second run at $\noisesig=0.01$ finished, its reproducibility across runs of its own noise level cannot be measured.

\textbf{Caveat.} We did not tune the learning rate for the lower noise levels, so a smaller step size might stabilize them. The alternatives that make likelihood training robust by limiting what the network can fit, fewer parameters or convolutional layers, presuppose a scale or a pixel basis, which EOFlows avoid.

\subsection{Noise levels on a grid of smaller models}\label{app: noise grid}

Before settling on the architecture of Sec.~\ref{sec: setup}, we swept the noise level from $\noisesig=0.01$ to $0.5$ on a grid of smaller CelebA models ($117$M parameters, $100$ epochs, $\lam\in\{0.01,0.1,1\}$, five or six runs per setting). Its absolute numbers are not comparable to the rest of the paper, but three trends informed our choice of $\noisesig=0.05$:
\begin{itemize}
    \item \textbf{Noise and the penalty both close the train/test gap.} From $\noisesig=0.01$ to $0.1$ the gap between held-out and training likelihood shrinks $13$--$26\times$, and at every noise level it is smaller for a larger $\lam$ (at $\noisesig=0.01$: $0.33$, $0.28$ and $0.13$ nats per dimension for $\lam=0.01$, $0.1$ and $1$).
    \item \textbf{More noise does not make the features more stable.} At $\lam=0.1$ and $\lam=1$, the number of stable features decreases as the noise level grows. Noise is needed to bound the likelihood, not to stabilize the features.
    \item \textbf{Samples are best at an intermediate noise level.} At $\lam=1$ the best bottleneck samples reach FID $29.8$ at $\noisesig=0.05$, against $34.1$ at $\noisesig=0.01$ and $34.2$ at $\noisesig=0.1$.
\end{itemize}

\subsection{A Gaussian bed with known axes}\label{app: noise gauss}

To see the effect of the noise and the penalty where the answer is known, we trained small flows ($8$ coupling blocks with learnable mixings, $20{,}000$ steps) on $32$-dimensional Gaussian data whose variances fall geometrically from $1$ to $1.3\cdot10^{-4}$ along the axes of a fixed random rotation.
The principal axes of the data, and the number of axes whose variance exceeds any given noise level, are therefore known exactly.
We varied the number of training points from $10^3$ to $10^5$, the noise level from $0.003$ to $2$ and the penalty over $\lam\in\{0,0.01,0.1,1,10\}$, with three runs per setting. Two results carry over to CelebA:
\begin{itemize}
    \item \textbf{The penalty prevents memorization as well.} With too little noise the flow memorizes its training points: with $10^4$ points at $\noisesig=0.003$, the gap between held-out and training likelihood reaches $55.7$ nats per dimension without the penalty, and $22.0$, $4.8$ and $0.47$ at $\lam=0.1$, $1$ and $10$.
    The noise level below which the gap exceeds $0.01$ nats per dimension accordingly falls with $\lam$, from $0.17$ at $\lam=0$ and $\lam=0.1$ to $0.14$ at $\lam=1$ and $0.027$ at $\lam=10$ (with $10^4$ points), and it also falls with the number of training points. The noise a model needs thus depends on the model and the data size, not only on the geometry of the data.
    \item \textbf{The noise floor counts the axes above noise.} A coordinate carries more signal than noise if its variance, which includes the noise, exceeds twice the noise variance, i.e.\ if its manifold entropy exceeds $H_{\noisesig}+\tfrac12\log 2$. The number of such coordinates matches the true number of axes with variance above $\noisesig^2$ to within about one coordinate on average ($0.12$--$1.33$, depending on $\lam$, with $10^4$ points or more).
\end{itemize}

\subsection{Unidentifiable detail coordinates}\label{app: noise detail}

\textbf{Why the detail cannot be reproduced.} Isotropic noise is invariant under rotations, so coordinates that encode only noise can be rotated among themselves without changing the likelihood or the Hadamard penalty, and no training procedure can make them reproducible.
Where the data have some variance below the noise level, the symmetry is broken only weakly.

\textbf{Noisy PCA.} Noisy PCA shows the consequence without any optimization error. Fitted on the training set augmented tenfold with independent noise draws (six fits, Fig.~\ref{fig: stability all models}c), it reproduces $653_{\pm9}$ coordinates per pair of fits, about the $657$ principal components whose variance exceeds $\noisesig^2$.
With $500$ noise draws per image instead of $10$ (drawn with an exact sampler of the augmented covariance), it reproduces $1252_{\pm12}$: below the noise floor, the weakly broken symmetry is resolved only by averaging over many noise draws.

\newpage

\section{Related work} \label{app: Related work}

\subsection{Principal Component Flows}
While Principal Component Flows (PCF) is an important precursor, EOFlows differ in three fundamental ways:
\begin{enumerate}
    \item Scalability: PCF derive the contrast (their eq.~(14) is our eq.~(\ref{eq: def EOFlow objective})) and estimate it with one random Jacobian product per sample, as we do; their practical objective, however, uses \emph{encoder rows} (vector-Jacobian products), whereas ours uses \emph{decoder columns} (Jacobian-vector products). For a square Jacobian both vanish on the same set of orthogonal maps, but the two gaps differ away from the optimum, and only the column form is the quantity that enters the manifold entropies. We demonstrate the estimate at $D=2352$ and $D=12288$ (App.~\ref{app: stochastic estimate of L_TC}); PCF's experiments stay below $D=10$.
    \item Noise Inflation: PCF add Gaussian noise with a fixed standard deviation of $0.01$ in one low-dimensional ($3$-D) experiment, to keep the data full-rank and to stop the model from collapsing; they neither derive this level nor study how noise affects reproducibility across runs. We find that noise acts together with the Hadamard penalty as a guard against memorization (App.~\ref{app: Noise inflation}).
    \item Dimension ordering: PCF do not aim at disentanglement and propose no way to sort the latent dimensions, and thus lack the notions of feature importance and of a core/detail split.
\end{enumerate}

\subsection{NDFlows vs EOFlows}

NDFlows impose an \textit{active} sorting of the latent dimensions at training time: their reconstruction loss sets the detail latents to zero and thus acts on the latent distribution as a whole.
The Hadamard regularizer of EOFlows is local and blind to permutations of the latent dimensions, so EOFlow latents are sorted \textit{passively}, by their manifold entropy after training.
It also does not force an individual latent to stay globally coherent across the data (see Fig.~1(B) of \citet{galperin2025analyzing}), which poses no problem in the linear case (App.~\ref{app: Linearized Normalizing Flows and PCA}).
For a fair comparison, all models, NDFlows included, are evaluated with their latents sorted by decreasing manifold entropy, $H_i \geq H_{i+1}$; for NDFlows this changes little (App.~\ref{app: Nested Dropout implementation}).

\subsection{Topology of the data}\label{app: topology}
We assume that CelebA consists of roughly one connected mode. A simple indication is whether the data mean is itself in distribution: the average of all CelebA faces resembles a face, whereas the average of all EMNIST digits resembles no digit.
Data with several disconnected modes do not match the topology of the unimodal latent Gaussian, which limits the global disentanglement of EOFlows on such data (Sec.~\ref{sec: conclusion}).

\newpage

\section{Experimental setup} \label{app: Experiments}
This appendix specifies the flow architecture (App.~\ref{app: Normalizing Flow architecture}), data and training (App.~\ref{app: Training setup} and~\ref{app: CelebA training data}), the baselines (App.~\ref{app: VAE baselines}--\ref{app: denoising AE}) and the metrics (App.~\ref{app: Jacobian-based metrics}); App.~\ref{app: results} collects additional results.

\subsection{Normalizing Flow architecture} \label{app: Normalizing Flow architecture}

Although our training objective can be applied to any invertible normalizing flow architecture, a few considerations motivate the architecture we use.
In particular, one has to strike a balance between expressivity, unbiased-ness and training speed.

Recent autoregressive flows such as TarFlow and STARFlow \citep{zhai2025normalizing, gu2025starflow} are highly expressive at a comparable training time, but their decoder runs sequentially over the tokens of an image and is therefore slow.
Since our estimator evaluates the decoder in every training step, autoregressive flows are impractical for EOFlows.
We thus opt for a symmetric encoder-decoder design, in particular coupling layers.

As we are aiming to find features which are stemming from the data-generating process, we opt out of using convolutional layers as this would introduce a bias towards local-pixel correlations.
Nonetheless a convolutional approach as in \citet{NEURIPS2018_d139db6a} is possible; App.~\ref{app: conv64} tests one on $64\times64$ images.

Finally, Neural Spline Flows (NSF) \citep{NEURIPS2019_7ac71d43} is a family of highly expressive NF architectures, which allow learning highly complex low-dimensional distributions. This, however, comes at the cost of throughput time and makes training very slow. We thus opt for the less expressive, yet still effective, affine coupling layers introduced in \citet{dinh2017density}.

All flows in this paper -- EOFlows, NDFlows and the plain NF -- use the same architecture. An ActNorm layer is followed by $12$ affine coupling blocks \citep{dinh2017density}, implemented as the GLOW coupling blocks \citep{NEURIPS2018_d139db6a} of FrEIA \citep{freia}. Each block updates both halves of its input with affine couplings whose scales and shifts are computed by fully connected MLPs with two hidden layers of width $1024$ ($1176 \to 1024 \to 1024 \to 2352$, ReLU, output layer initialized at zero, scales soft-clamped at $2$), and then applies a fixed random permutation, an ActNorm layer and a learnable $D \times D$ mixing matrix. The last block has no mixing of its own; a final learnable mixing follows it instead, so the flow contains $12$ learnable mixings. Maximum likelihood is invariant to rotations of the latent space whereas the Hadamard regularizer is not, and the learnable mixings give the regularizer room to choose the latent basis; App.~\ref{app: architecture ablation} measures how much of this freedom is needed.

\textbf{Parametrization of the mixings.} Each mixing is $W = QLU$, where $Q$ is a random orthogonal matrix that stays fixed, $L$ is lower triangular with unit diagonal and $U$ is upper triangular with diagonal entries $\exp(s_i)$. $L$ and $U$ are initialized at the identity, so every mixing starts as the random rotation $Q$. The log-determinant of a mixing is $\sum_i s_i$ and its inverse takes two triangular solves, so no determinant or matrix inverse is computed during training. The model has $245$M parameters as stored; only the strict triangles of the stored $D \times D$ tensors for $L$ and $U$ are used, so $178$M of them are effective -- as many as in the same flow with dense learnable mixings.

\textbf{Why this parametrization.} We compared parametrizations of the mixings on a linear version of the flow, in which every coupling block is replaced by the identity, so that the optimum of the EOFlow loss is known in closed form: noisy PCA (App.~\ref{app: Linearized Normalizing Flows and PCA}). On grayscale CelebA ($D = 784$, eight mixings, $\lam = 1$, $\noisesig = 0.05$, batch size $128$, $10$ epochs, three seeds), learnable dense and QLU mixings come within about $0.002$ nats per dimension of the likelihood of this optimum ($-1.0607$ and $-1.0608$ against $-1.0631$). Two alternatives lag behind after the same budget: all but the final mixing frozen at their random orthogonal initialization ($-1.0425$), and a pivoted LU factorization of the same random orthogonal matrices with learnable offsets scaled by $1/\sqrt{D}$ ($-1.0445$). At $\lam = 0.1$ these two arms reach $-1.0611$ and $-1.0607$, as close to the optimum as the learnable arms at $\lam = 1$, and their difference to the learnable arms ($-1.0624$) shows mainly in the Hadamard gap $\L_\text{H}/D$ ($0.165$ and $0.201$ against $0.023$ and $0.021$ for dense and QLU). The LU parametrization fails because the LU factors of a random orthogonal matrix are ill-conditioned ($\|L\|_2 \approx 28$, $\|U\|_2 \approx 52$); QLU keeps that matrix as the fixed factor $Q$ and learns well-conditioned factors that start at the identity. The comparison concerns optimization within a short budget: after $10$ epochs even the learnable arms align only about $20$ of the $300$ leading decoder columns with an eigenvector of the data covariance.

To summarize, we use a large fully connected architecture with fast encoder and decoder passes. We did not tune its depth or width for the final recipe.

\subsection{Dataset and Training setup} \label{app: Training setup}

We use the aligned CelebA images \citep{liu2015faceattributes} ($178\times218$ pixels), center-cropped to $178\times178$, cropped to a $112\times112$ face region, downsampled with Lanczos interpolation to $28 \times 28 \times 3$ ($D = 2352$) and scaled to $[0,1]$:
\begin{lstlisting}[caption={CelebA data preparation (torchvision)}]
from torchvision import transforms as T

CelebA_transform = T.Compose([
    T.CenterCrop(178),                                             # 178 x 178
    T.Lambda(lambda img: img.crop((33, 43, 33 + 112, 43 + 112))),  # 112 x 112 face
    T.Resize((28, 28), interpolation=T.InterpolationMode.LANCZOS),
    T.ToTensor(),                                                  # float in [0, 1]
])
\end{lstlisting}

\textbf{Data splits.} All models are trained on the official CelebA training partition ($162{,}770$ images). Held-out likelihoods and reconstructions use a fixed subset of $10{,}000$ images of the validation partition, and FID compares $10{,}000$ generated or reconstructed images with the test partition ($19{,}962$ images).

\textbf{Training.} The flows are trained on noise-inflated images $\xnoisy = \x + \noisesig \eps$ with $\noisesig = 0.05$ and fresh noise in every batch, using schedule-free AdamW \citep{NEURIPS2024_136b9a13} with batch size $512$, learning rate $6\cdot10^{-4}$ and $100$ warm-up steps; $500$ epochs are $158{,}500$ steps. EOFlows use $\lam = 1$ unless stated otherwise and the stochastic estimator of the Hadamard regularizer (App.~\ref{app: stochastic estimate of L_TC}). Batch size and learning rate were set in a smoke test with one seed and at most $23$ epochs (batch sizes $128$ to $1024$, learning rates $3\cdot10^{-4}$ to $1.2\cdot10^{-3}$, the largest of which diverged during warm-up) and were not tuned further. The plain NF uses the same recipe with $\lam = 0$ and the estimator switched off.

\textbf{Noise inflation of the baselines.} Noisy PCA is PCA of the noise-inflated training data, i.e.\ of the data covariance plus $\noisesig^2\I_D$, which is the optimum of a linear EOFlow at the same noise level (App.~\ref{app: Linearized Normalizing Flows and PCA}); it has the principal directions of the clean data, each with its variance increased by $\noisesig^2$. To measure its stability like that of the trained models (Table~\ref{tab: summary}, Fig.~\ref{fig: stability all models}c), we fit it six times on the training set augmented tenfold with independent noise draws (App.~\ref{app: noise detail}). The VAEs are trained on clean images: a VAE is an injective model, whose decoder maps its $32$ latents onto a low-dimensional manifold while its likelihood, here Bernoulli without a learnable scale, covers the remaining directions, so it does not suffer from the manifold overfitting of bijective flows that noise inflation prevents (Sec.~\ref{sec:noise}). Moreover, the inflation--deflation argument is specific to likelihood-based flows under the manifold hypothesis.

\textbf{Evaluation.} All numbers are computed with the averaged weights of the schedule-free optimizer. Raw training iterates can reverse the ordering of two models: in a 100-epoch comparison at batch size $128$, the raw iterates put frozen intermediate mixings ahead of learnable ones in held-out likelihood and $\L_\text{H}$, and the averaged weights put the learnable ones ahead. Likelihoods are computed on the $10{,}000$ held-out images with fixed inflation noise, $\L_\text{H}/D$ exactly on $32$ held-out images, entropy spectra on $512$ prior samples, and stability on $256$ fixed training images (App.~\ref{app: Jacobian-based metrics}).

\textbf{Hardware.} Training ran on single NVIDIA RTX 4090 (24 GB) and A40 (48 GB) GPUs. An EOFlow epoch takes about $140$ s, both alone on an A40 and with two runs sharing an RTX 4090; on an A40, an NDFlow epoch takes about $94$ s and a plain-flow epoch about $59$ s (medians of the logged epoch times).

\textbf{Runs and seeds.} Table~\ref{tab: runs and seeds} lists the trained models behind the reported numbers. EOFlow $\lam = 1$ seed 104 diverged in epoch 1, about $400$ steps after the warm-up, and was replaced by seed 107; seed 107 was an outlier at $300$ epochs (held-out NLL $-1.2958$ against $-1.2976$ to $-1.2980$ for the other three seeds, prior FID $28.1$ against $24.7$ to $24.9$) and was not continued to $500$ epochs. EOFlow $\lam = 3$ seed 105 diverged in epoch 0 and was replaced by seed 106. Two of eight NDFlow $\nu = 1$ runs diverged before $500$ epochs (App.~\ref{app: Nested Dropout implementation}). Fig.~\ref{fig: archetypes overview} shows features of EOFlow $\lam = 1$ seed 105 and the appendix figures of a single model show seed 103, both at $300$ epochs unless stated otherwise. AnnealedVAE was trained only for its tuning budget of $100$ epochs; its entries in Table~\ref{tab: summary} and Fig.~\ref{fig: stability all models} refer to that budget.

\begin{table}[h]
    \centering
    \caption{Trained models behind the reported numbers (CelebA $28 \times 28$, $\noisesig = 0.05$; flows at batch size $512$ and learning rate $6\cdot10^{-4}$ unless noted). Cells at $300$ and $500$ epochs of the same $\lam$ or $\nu$ are the same trainings, evaluated at both lengths.}
    \label{tab: runs and seeds}
    \small
    \setlength{\tabcolsep}{4pt}
    \begin{tabular}{llll}
        \toprule
        model & epochs & seeds & notes \\
        \midrule
        EOFlow $\lam = 1$ & $500$ & 103, 105, 106, 108--110 & \\
        EOFlow $\lam = 1$ & $300$ & 103, 105, 106 & \\
        EOFlow $\lam = 0.1$ & $500$ & 103, 105, 106 & \\
        EOFlow $\lam = 0.1$ & $300$ & 103, 105, 106 & snapshots after $301$--$302$ epochs \\
        EOFlow $\lam = 0.01$ & $300$ & 103, 105, 106 & \\
        EOFlow $\lam = 3$ & $300$ & 103, 106, 108, 109 & \\
        EOFlow $\lam = 10$ & $300$ & 103, 105, 108, 109 & \\
        plain NF ($\lam = 0$) & $300$ & 103, 105, 106 & \\
        NDFlow $\nu = 1$ & $500$ & 105, 106, 109--112 & 111, 112 replace two diverged runs \\
        NDFlow $\nu = 1$ & $300$ & 103, 105, 106 & stability: 105, 106, 109, 110 \\
        NDFlow $\nu \in \{0.1, 10\}$ & $300$ & 103, 105, 106 & mixing learning rates: App.~\ref{app: Nested Dropout implementation} \\
        architecture ablation & $100$ & 301--303 per arm & App.~\ref{app: architecture ablation} \\
        shuffled pixel basis & $100$ & 101--104 & batch size $128$, learning rate $3\cdot10^{-4}$ \\
        rotated pixel basis & $100$ & 101, 102 & batch size $128$, learning rate $3\cdot10^{-4}$ \\
        $\beta$-VAE $\beta = 4$ & $300$ & 101--103 & \\
        $\beta$-TCVAE $\beta = 4$ & $300$ & 101--103 & \\
        FactorVAE $\gamma = 10$ & $300$ & 101--103 & batch size $256$ \\
        FactorVAE $\gamma = 10$ & $500$ & 101--106 & batch size $256$; stability (Fig.~\ref{fig: stability all models}) \\
        AnnealedVAE $C = 25$ & $100$ & 101--105 & its tuning budget \\
        VAE tuning grid & $100$ & 101--103 per cell & FactorVAE: $200$ epochs, batch size $256$ \\
        five-seed tuning cells & $100$ & 104, 105 & $\beta$-VAE $\beta = 4$, $\beta$-TCVAE $\beta = 3$, \\
        & & & FactorVAE $\gamma = 10$, AnnealedVAE $C = 25$ \\
        \bottomrule
    \end{tabular}
\end{table}

\subsection{Training data}\label{app: CelebA training data}

Fig.~\ref{fig: CelebA training samples} shows $12$ random training images, clean and inflated as the flows see them; the noise shows as a grain on every image. The flows are trained only on inflated images, with fresh noise in every batch; the VAEs are trained on clean images. Clean images also serve as the reference in evaluation, e.g.\ for FID. All image figures upscale the $28\times28$ images with Lanczos interpolation for display.

\begin{figure*}[!htbp]
    \centering
    \includegraphics[width=\textwidth]{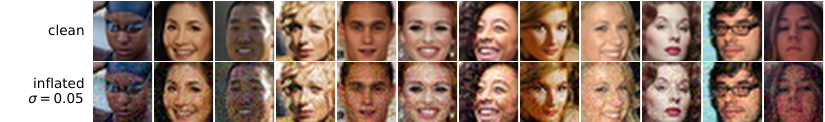}
    \caption{Twelve random CelebA training images, clean (top) and inflated with Gaussian noise of $\noisesig=0.05$ as seen by the flows (bottom; one noise draw, redrawn at every training step).}
    \label{fig: CelebA training samples}\label{fig: CelebA clean samples}\label{fig: CelebA noisy samples}
\end{figure*}

\subsection{VAE baselines} \label{app: VAE baselines}

We compare with four VAE objectives of the study of \citet{locatello2019challengingcommonassumptionsunsupervised}: $\beta$-VAE \citep{higgins2017betavae}, $\beta$-TCVAE \citep{NEURIPS2018_1ee3dfcd} with $\alpha = \gamma = 1$, FactorVAE \citep{pmlr-v80-kim18b}, and AnnealedVAE \citep{burgess2018understandingdisentanglingbetavae} with $\gamma = 100$ and the capacity annealed from $0$ to $C$ over $10{,}000$ steps. We use the loss implementations of \citet{dubois2019dvae} without changes and train on the same preprocessed image tensors as the flows.

\textbf{Architecture and training.} Following the fully connected models of \citet{higgins2017betavae} and \citet{NEURIPS2018_1ee3dfcd}, encoder and decoder are ReLU MLPs (encoder $2352 \to 1200 \to 1200 \to z$, decoder $z \to 1200 \to 1200 \to 1200 \to 2352$ with sigmoid output) with a Bernoulli likelihood; the model has $10.1$M parameters at $z = 32$. We train with schedule-free AdamW (learning rate $3\cdot10^{-4}$, $100$ warm-up steps) at batch size $128$ for $100$ epochs. FactorVAE, whose implementation splits every batch in two for its discriminator, is trained at batch size $256$ for $200$ epochs, which gives the same number of updates. The selected settings of $\beta$-VAE, $\beta$-TCVAE and FactorVAE (below) were then trained for $300$ epochs, and FactorVAE also for $500$ epochs with six seeds; AnnealedVAE was not trained beyond $100$ epochs (Table~\ref{tab: runs and seeds}). All evaluations use the averaged weights.

\textbf{Tuning.} Regularization strength and latent size were tuned for the number of stable features, a count of the same kind as for the flows (see below): $\beta$-VAE with $\beta \in \{4, 8\}$ and $\beta$-TCVAE with $\beta \in \{2, 3, 4, 6\}$, both at $z \in \{32, 64\}$, and FactorVAE with $\gamma \in \{10, 40\}$ and AnnealedVAE with $C \in \{25, 50\}$, both at $z = 32$. This grid followed a pilot of $28$ runs ($30$ epochs, two seeds). Every cell was trained for $100$ epochs with seeds 101 to 103.

\textbf{Stability count.} For each VAE we compute the decoder Jacobian at the encoder mean of $512$ fixed training images, average the squared cosine similarity between the unit Jacobian columns of two seeds over the images, and match the $z$ latents of the two seeds with the Hungarian algorithm. For each seed, its latents are ordered by their matched similarity averaged over the other seeds; these ordered curves are averaged over all ordered seed pairs of a cell, and the cell's count is the number of ranks at which the average is at least $0.5$. This count differs from the one used for the flows in two ways: it uses the raw squared cosine instead of the polar factor, and it is pooled over all seed pairs instead of reported per pair. A pooled count over five seeds is stricter than one over three ($\beta$-TCVAE with $\beta = 3$: $11$ with three seeds, $9$ with five). This pooled count was used only to select the tuned settings; Table~\ref{tab: summary} and Fig.~\ref{fig: stability all models} evaluate these settings with the per-pair polar count of Sec.~\ref{sec: setup}, like the flows.

\textbf{Selected settings.} With three seeds, the pooled count ranks $\beta$-VAE $\beta=4$ ($9$ stable features) and AnnealedVAE $C=25$ ($4$) first for their objectives and ties $\beta$-TCVAE $\beta=3$ with $\beta=4$ ($11$ each) and FactorVAE $\gamma=10$ with $\gamma=40$ ($10$ each). Our selection rule broke ties by the number of latents with a pooled similarity of at least $0.3$ and then by the mean over the ten best latents, which chose $\beta=3$ and $\gamma=10$; these four cells were extended by seeds 104 and 105, and the stricter five-seed count gives $9$ for $\beta$-TCVAE $\beta=3$ and FactorVAE $\gamma=10$, $8$ for $\beta$-VAE and $3$ for AnnealedVAE. For the longer runs we kept $\beta$-VAE $\beta=4$ and FactorVAE $\gamma=10$ and took $\beta$-TCVAE $\beta=4$, the other cell of its tie. Doubling the latent size to $z=64$ does not raise the best count (best $z=64$ cell: $10$, $\beta$-TCVAE $\beta=2$); at equal $\beta$ the count changes by $-4$ to $+1$.

\textbf{Settings used in the paper.} All CelebA results use the same four settings, each with $z=32$ latents: $\beta$-VAE $\beta=4$, $\beta$-TCVAE $\beta=4$ and FactorVAE $\gamma=10$ after $300$ epochs with seeds 101--103 (FactorVAE also after $500$ epochs with six seeds in Fig.~\ref{fig: stability all models}), and AnnealedVAE $C=25$ after its $100$-epoch tuning budget with seeds 101--105. The five-seed $\beta$-TCVAE $\beta=3$ runs enter no reported result. The ground-truth benchmarks use convolutional VAEs at the settings of the literature instead (App.~\ref{app: toy benchmarks}).

The VAEs are trained on clean images, without noise inflation (App.~\ref{app: Training setup}).

\textbf{Entropy order.} For the manifold entropies $H_i$ of a VAE, we treat its decoder output as a deterministic function of the standard normal prior, as for a flow and as in \citet{galperin2025analyzing}, and sort its latents by $H_i$.

\subsection{Nested Dropout implementation}\label{app: Nested Dropout implementation}

NDFlows \citep{bekasov2020orderingdimensionsnesteddropout} use the bottleneck reconstruction (\ref{eq:autoencoder-like-reconstruction}) during training. For every noise-inflated training image $\xnoisy$ they draw a bottleneck size $C \sim p_C$ and minimize the reconstruction error $\L_{\rec,C}(\xnoisy) := \tfrac{1}{D}\|\xnoisy - \g([\f_{\le C}(\xnoisy), 0])\|_2^2$ in a weighted sum with the maximum likelihood loss,
\begin{equation}
    \L_\ND(\xnoisy) = \L_\ML(\xnoisy) + \nu \cdot \L_{\rec, C}(\xnoisy) ,\quad \nu > 0 ,
\end{equation}
where both terms are computed per dimension and averaged over the batch. Because $C$ is drawn independently for every image from a geometric distribution, most reconstructions keep only the $z_i$ with small $i$, and the model is pushed to channel the core information into the first dimensions. The latents are not re-sorted during training.

Following \citet{pmlr-v32-rippel14} and \citet{bekasov2020orderingdimensionsnesteddropout}, $p_C$ is a geometric distribution, $p_C(C) \propto (1 - p)^{C-1}p$, truncated to $1 \leq C \leq D$. We use $p = 0.01$ for all NDFlows, so that the expected bottleneck is $1/p = 100$, only $0.99^{300} \approx 5\,\%$ of the draws keep more than $300$ latents, and fewer than $1\,\%$ more than $500$. This bias towards small $C$ may be why NDFlows denoise poorly at large $C$ in Fig.~\ref{fig: CelebA rate-distortion SSIM test}b: bottlenecks of more than a few hundred latents are almost never trained for.

\textbf{Training recipe.} NDFlows use the architecture, data, noise level, batch size, base learning rate and evaluation of the EOFlows, without the Hadamard regularizer. The reconstruction term required three changes, found in short validation runs. (i) $500$ instead of $100$ warm-up steps: with $100$ steps, $\nu = 1$ diverged at the end of the warm-up at learning rates $6\cdot10^{-4}$ and $4.5\cdot10^{-4}$, whereas the same flow without the reconstruction term trained stably. (ii) Gradient clipping at norm $2$, above the largest gradient norm of a $300$-step run at $\nu = 10$ without clipping ($1.62$); it was added after the first $\nu = 1$ runs had started, so seeds 103 and 105 trained their first $14$ epochs without it. (iii) A lower learning rate for the mixing matrices: with all parameters at $6\cdot10^{-4}$, $\nu = 3$ and $\nu = 10$ diverged within their first two epochs. In the $300$-epoch sweep the mixings were trained at $3\cdot10^{-4}$, with two exceptions: $\nu = 10$ seeds 103 and 105 diverged at this rate and were retrained at $10^{-4}$, and $\nu = 1$ seeds 103 and 105 ran at a uniform $6\cdot10^{-4}$ from epoch 14 on, after a restart that reset the learning rates of the parameter groups and the optimizer state. The $500$-epoch $\nu = 1$ runs used a uniform $6\cdot10^{-4}$, except seed 106 ($3\cdot10^{-4}$). Seed 106 differs from the other $\nu = 1$ seeds mainly in a larger train--test gap ($+0.124$ against $+0.092$ to $+0.097$ at $500$ epochs, $+0.099$ against $+0.077$ at $300$ epochs).

\textbf{Divergences.} Two of the eight $\nu = 1$ runs started for $500$ epochs diverged: seed 108 at epoch 116 and seed 103, resumed from $300$ epochs, at epoch 494. In both, the forward pass blew up, which gradient clipping cannot prevent. Seeds 111 and 112 replaced them. All six EOFlow $\lam = 1$ runs of the $500$-epoch set completed.

\textbf{Entropy ordering.} Like all flows, NDFlows are evaluated with their latents sorted by manifold entropy. Nested dropout already orders them almost by index: for the six $\nu = 1$ models at $500$ epochs, $289$ to $296$ of the $300$ highest-entropy latents have an index below $300$, and the ten highest-entropy latents include at least nine of the indices $0$ to $9$.

\textbf{Effect of $\nu$.} At $300$ epochs and three seeds per setting, a stronger reconstruction term trades likelihood for reconstruction: the held-out NLL is $-1.267$, $-1.274$ and $-1.236$ nats per dimension at $\nu=0.1$, $1$ and $10$ (the two retrained $\nu=10$ seeds: $-1.224$), the Hadamard gap $\L_\text{H}/D$ falls from $0.44$ to $0.20$ and $0.10$, and the FID of prior samples through the $C=100$ bottleneck is $26$, $22$ and $22$. None of the three reproduces its coordinates across seeds: $1.7_{\pm1.2}$, $8.5_{\pm1.8}$ and $5.7_{\pm0.6}$ polar factors per seed pair match at $\cos^2\geq0.5$ ($\nu=1$: four seeds, Table~\ref{tab: runs and seeds}), against $121$ for EOFlows at $\lam=1$. Denoising through the bottleneck (as in Fig.~\ref{fig: CelebA rate-distortion SSIM test}b, which shows $\nu=1$) improves with $\nu$ but saturates in every case: the held-out error levels off at $5\cdot10^{-3}$ for $\nu=0.1$, reaches the noise variance $\noisesig^2$ at $C\approx200$ for $\nu=1$ and a minimum of $1.7\cdot10^{-3}$ at $C=500$ for $\nu=10$, against $1.0\cdot10^{-3}$ for EOFlows at $\lam=1$. On the attribute score TAD the order is reversed (Table~\ref{tab: tad}): $\nu=0.1$ scores $0.58$, $\nu=1$ and $\nu=10$ about $0.25$. $\nu=1$ is therefore used as the representative NDFlow in Figs.~\ref{fig: stability all models} and \ref{fig: arch100 nd_nu1_e300}, with the $\nu=0.1$ and $\nu=10$ galleries in Figs.~\ref{fig: arch100 nd_nu0.1_e300} and \ref{fig: arch100 nd_nu10_e300}.

\subsection{Denoising autoencoders}\label{app: denoising AE}
The rate--distortion comparison (Fig.~\ref{fig: CelebA rate-distortion SSIM test}b) includes autoencoders trained for exactly the task that the bottleneck of a flow solves only implicitly. For each bottleneck $C\in\{10,20,50,100,200,500\}$ we train a separate deterministic autoencoder on inputs inflated with fresh noise of $\noisesig=0.05$ in every step and minimize the squared error to the clean image. All autoencoders use the budget of the flows, batch size $512$, schedule-free AdamW (learning rate $3\cdot10^{-4}$, gradient clipping at norm $1$) and $300$ epochs, i.e.\ $95{,}100$ updates, with one seed, and they are evaluated with averaged weights on the same noisy held-out images as the flows.

\textbf{Two sizes.} We tried two architectures. The \emph{small} autoencoder, shown in Fig.~\ref{fig: CelebA rate-distortion SSIM test}b, is a fully connected MLP (encoder $2352\to1200\to1200\to C$ with GELU activations, mirrored decoder with linear output; $8.6$--$9.7$M parameters). The \emph{large} one has the depth of the flows: $12$ residual blocks per side, each with two pre-LayerNorm updates at width $D$ and hidden width $512$, and linear projections to and from the code ($122$--$124$M parameters). Despite the noise inflation, the large autoencoder over-fits, with a training error $3$ to $7$ times below its held-out error, and it is worse than the small one on held-out images at every $C$ (Table~\ref{tab: denoising AE}).

\textbf{Denoising versus pure reconstruction.} Trained for pure reconstruction, i.e.\ with the noisy input itself as the target, as the flows see their data, both autoencoders perform almost the same as with the denoising objective: the held-out error changes by at most $4\cdot10^{-5}$ (small) and $1.4\cdot10^{-4}$ (large) at every $C$. The two losses differ only by the cross term $-2\,\mathbb{E}\langle\widehat{\x}-\x,\eps\rangle$ between reconstruction error and noise, and a small code carries little of the input noise ($0.45\,\%$ at $C=10$, $4\,\%$ at $C=100$ and at most $12\,\%$ at $C=500$).

\begin{table}[h]
\centering\small
\caption{Held-out denoising error (MSE, $\times10^{-3}$; noisy input, clean target) of the autoencoders, one model per $C$ and seed $101$, and of EOFlows ($\lam=1$, mean over three seeds) through their top-$C$ bottleneck.}
\label{tab: denoising AE}
\begin{tabular}{lcccccc}
\toprule
 & $C=10$ & $20$ & $50$ & $100$ & $200$ & $500$ \\
\midrule
small AE ($9$M), denoising (Fig.~\ref{fig: CelebA rate-distortion SSIM test}b) & $11.95$ & $7.98$ & $4.76$ & $3.17$ & $2.02$ & $1.65$ \\
small AE ($9$M), pure reconstruction & $11.96$ & $7.98$ & $4.76$ & $3.17$ & $2.03$ & $1.61$ \\
large AE ($122$M), denoising & $18.58$ & $13.41$ & $8.68$ & $6.12$ & $4.21$ & $3.14$ \\
large AE ($122$M), pure reconstruction & $18.48$ & $13.35$ & $8.59$ & $6.00$ & $4.07$ & $3.18$ \\
EOFlow $\lam=1$ & $14.66$ & $9.40$ & $4.96$ & $3.00$ & $1.77$ & $0.99$ \\
\bottomrule
\end{tabular}
\end{table}

\subsection{Jacobian-based metrics} \label{app: Jacobian-based metrics}

To evaluate the orthogonality, sparsity and stability of the learned features, we rely on the decoder Jacobians.
The full Jacobian represents the curvilinear coordinate system at a point $\z$ and $\J_i$ represents the first-order approximation of one coordinate line. Metrics based on $\J_i$ offer a highly informative analysis of the learned representation space as they describe the behavior of the encoder/decoder-pair locally. To obtain global metrics, we average over data samples.

\citet{galperin2025analyzing} visualized the coordinates by the average Jacobian column $\Expt{\z = \f(\x)}{\J_i(\z)}$. Our archetype difference images (App.~\ref{app: CelebA archetypes}) are closely related: $\x_{i+4} - \x_{i-4} = \int_{-4}^{4} \J_i\, dz_i$ integrates the Jacobian column along the mean coordinate line $\C_i(\z_{\ne i}=\mathbf{0})$.

All metrics use the decoder Jacobian $\J(\z)$ at the latents $\z = \f(\x)$ of fixed images $\x$, computed with the averaged (schedule-free) weights of each model, and we view its $i$-th column $\J_i$ as the $i$-th latent representation at the point $\x$. The metrics differ in the images and columns they use: effective sparsity uses the columns of the $300$ highest-entropy latent dimensions at $512$ clean training images, stability uses all $D$ columns at the first $256$ of these images, and the Hadamard gap uses all $D$ columns at $32$ noise-inflated held-out images.

\subsubsection{Orthogonality}

We measure orthogonality with the Hadamard gap of Eq.~(\ref{eq: def Total Disentanglement MAIN}), which is zero if and only if the decoder-Jacobian columns are orthogonal and does not change when a column is rescaled. VAE decoders map $32$ latents to $D$ pixels and are not bijective, so we use its injective form over the $C$ core coordinates $S$, i.e.\ the $C$ highest-entropy ones, per coordinate:
\begin{equation}\label{eq: injective hadamard gap}
    \L_\text{H}^{(C)}(\z) = \frac1C\Big[\sum_{i\in S}\log\|\J_i(\z)\|_2 - \tfrac12\log\det\big(\J_S(\z)^\top\J_S(\z)\big)\Big] \;\geq\; 0 .
\end{equation}
It is non-negative by Hadamard's inequality for the Gram matrix $\J_S^\top\J_S$, and for a flow and $C=D$ it equals $\L_\text{H}/D$. Its value depends on $C$, since the gap over a set of columns is at least the sum of the gaps over any split of the set into parts, so values are strictly comparable only at equal $C$. Table~\ref{tab: orthogonality} therefore lists every variant we computed: $C=32$ for every model, i.e.\ all latents of the VAEs, $C=300$ and $C=D$ for the flows. Table~\ref{tab: summary} combines $C=300$ for the flows with $C=32$ for the VAEs, as the sparsity measure does; at the matched $C=32$ the ranking of the models is the same. For $C\leq300$ we average over the $512$ clean training images of the sparsity measure (VAEs: at the posterior means); $C=D$ is computed exactly from the full $D\times D$ Jacobian at $32$ held-out images with added training noise ($\noisesig=0.05$, one fixed draw), in double precision. All models use their averaged weights.
\label{app: orthogonality}

\begin{table}[h]
\centering\small
\caption{Orthogonality of all models: Hadamard gap per coordinate (Eq.~\ref{eq: injective hadamard gap}) over the $C$ core coordinates, mean$_{\pm\text{sd}}$ over seeds (NDFlow $\nu=1$ at $C\leq300$: one seed). Table~\ref{tab: summary} combines $C=300$ for the flows and PCA with $C=32$, i.e.\ all latents, for the VAEs; $C=D$ for EOFlows and the NF is Fig.~\ref{fig: lambda dial}b.}
\label{tab: orthogonality}
\begin{tabular}{lccc}
\toprule
model & $C=32$ & $C=300$ & $C=D$ \\
\midrule
\textcolor[HTML]{2A78D6}{EOFlow $\lam=1$} & 0.035$_{\pm<.001}$ & 0.041$_{\pm.001}$ & 0.014$_{\pm<.001}$ \\
\textcolor[HTML]{2A78D6}{EOFlow $\lam=0.1$} & 0.13$_{\pm<.01}$ & 0.17$_{\pm<.01}$ & 0.058$_{\pm.002}$ \\
\textcolor[HTML]{EB6834}{NDFlow $\nu=1$} & 0.20 & 0.29 & 0.20$_{\pm<.01}$ \\
\textcolor[HTML]{EB6834}{NDFlow $\nu=0.1$} & 0.38$_{\pm.01}$ & 0.46$_{\pm<.01}$ & 0.44$_{\pm<.01}$ \\
\textcolor[HTML]{D55181}{NF} & 1.2$_{\pm<.1}$ & 1.6$_{\pm<.1}$ & 1.6$_{\pm<.1}$ \\
\midrule
\textcolor[HTML]{4A3AA7}{FactorVAE} & 0.38$_{\pm.07}$ & -- & -- \\
\textcolor[HTML]{6A5CC6}{$\beta$-TCVAE} & 0.55$_{\pm.01}$ & -- & -- \\
\textcolor[HTML]{8A7FDF}{$\beta$-VAE} & 0.50$_{\pm.01}$ & -- & -- \\
\textcolor[HTML]{A79FEC}{AnnealedVAE} & 0.77$_{\pm.05}$ & -- & -- \\
\midrule
noisy PCA & 0 & 0 & 0 \\
\bottomrule
\end{tabular}
\end{table}

\subsubsection{Effective Sparsity}\label{app: sparsity}

We introduce the normalized squared Jacobian column vector as:
\begin{equation}
    \bar{\J}_i \coloneqq \frac{(\J_i)^2}{\sum_{j=1}^D (\J_{ji})^2}
\end{equation}
where the square is performed elementwise and the sum goes over the data dimension $j$. Evidently, the sum $\sum_j \bar{\J}_{ji} \equiv 1$. This can be used to define the following sparsity metric:
\begin{equation}
    \mathrm{ENZ}(i) \coloneqq \frac{1}{D}\,\Expt{\z = \f(\x)}{2^{H_2(\bar{\J}_i(\z))}}, \qquad H_2(\bar{\J}_i) \coloneqq -\sum_{j=1}^{D} \bar{\J}_{ji} \log_2 \bar{\J}_{ji}
\end{equation}
Here $2^{H_2(\bar{\J}_i)}$ is the effective number of active pixels of $\J_i$, analogous to the Effective Number of Non-Zero Pixels measure proposed in \citet{he2026effectivesparsityunifiedframework}; dividing by $D$ gives the Effective Percentage of Non-Zero Pixels (ENZ). If $\bar{\J}_i$ is spread equally over $k$ pixels and zero elsewhere, $\mathrm{ENZ}(i) = k/D$, so ENZ generalizes the L0 sparsity count: a single active pixel gives $1/D$ (most sparse), all pixels equally active give $1$ (least sparse), and a random Gaussian direction gives $0.482$ on average. Unlike a count of the pixels whose change exceeds a threshold, ENZ depends neither on a threshold nor on the scale of the change. We evaluate ENZ at $512$ clean training images for the $300$ highest-entropy latent dimensions of each flow; the VAEs have only $32$ latents, which are all used (with the Jacobian at the encoder mean), and the PCA components are ranked by variance. We restrict ENZ to the leading coordinates because the low-entropy tail carries no signal. Over all $D$ coordinates, the sparsest PCA components are null-space eigenvectors (median variance rank of about $2300$ of $2352$, with numerically zero eigenvalues), which can concentrate on any pixels and would make PCA look sparser than EOFlows at $\lam=0.1$; EOFlows at $\lam=1$ would also gain sparse coordinates ($203$ below $0.1$ instead of $90$), but mostly inside the core (median entropy rank $330$).

\subsubsection{Stability} \label{app: Stability}

To measure stability, we compare the decoder Jacobians of two models $a$ and $b$ that differ only in the random seed. Before comparing their columns, we replace the full $D \times D$ decoder Jacobian $\J$ at each image by its polar factor
\begin{equation}
    \boldsymbol{P} \coloneqq (\J\J^T)^{-1/2}\,\J = \boldsymbol{U}\boldsymbol{V}^T \qquad \text{for } \J = \boldsymbol{U}\boldsymbol{\Sigma}\boldsymbol{V}^T ,
\end{equation}
the orthogonal matrix closest to $\J$. Its column $\boldsymbol{P}_i$ is the direction of latent $i$ after the whole frame has been made orthonormal with the smallest possible change; if the columns of $\J$ are already orthogonal ($\L_\text{H} = 0$), then $\boldsymbol{P}_i = \J_i / \|\J_i\|_2$. We compute $\boldsymbol{P}$ with $40$ Newton--Schulz iterations started from $\J / \|\J\|_F$, which agree with the SVD up to a largest entrywise difference of $3 \cdot 10^{-4}$. The similarity matrix $S^{(ab)} \in \mathbb{R}^{D \times D}$ is the squared overlap of the polar factors, averaged over the same $N = 256$ fixed clean training images $\x_n$ for every model:
\begin{equation}
    S^{(ab)}_{ij} \coloneqq \frac{1}{N} \sum_{n=1}^{N} \Big( {\boldsymbol{P}_i^{(a)}(\x_n)}^T \boldsymbol{P}_j^{(b)}(\x_n) \Big)^2 ,
\end{equation}
where $\boldsymbol{P}^{(a)}(\x_n)$ is the polar factor of the decoder Jacobian of model $a$ at $\z = \f^{(a)}(\x_n)$. Since ${\boldsymbol{P}^{(a)}}^T \boldsymbol{P}^{(b)}$ is an orthogonal matrix, the squares of its entries sum to one in every row and every column, so $S^{(ab)}$ is doubly stochastic (numerically, on $32$ images of one seed pair with a single-precision SVD, all row and column sums lie between $1.0000$ and $1.0022$). For exactly orthogonal frames, $S^{(ab)}$ reduces to the raw squared cosine similarity of the Jacobian columns.

Latent $i$ of model $a$ need not correspond to latent $i$ of model $b$, so we match the two orderings: Hungarian matching of $S^{(ab)}$ (\texttt{linear\_sum\_assignment} in SciPy) permutes the latents of model $b$ such that the sum of the diagonal entries of the permuted matrix $\widehat{S}^{(ab)}$ is maximal. We are allowed to do this, since stable training can only be expected up to permutations and sign flips of the latent dimensions, and the squared overlap does not see signs. We denote the matched order by $k$; it generally does not coincide with the entropy order $i$. The matching runs over all $D$ latents and not only over the $300$ highest-entropy ones, because stable latents also exist below entropy rank $300$ (in an earlier check with the raw squared cosine on one seed pair of a 200-epoch model, $15$ of $109$ matched values $\geq 0.5$ lay below that rank). 

We take the diagonal entries of $\widehat{S}^{(ab)}$, sorted in decreasing order, as the stability of the matched latents
\begin{equation}
    \mathrm{Stability}(k, a, b) \coloneq \widehat{S}^{(ab)}_{kk} \quad \text{s.t. } \mathrm{Stability}(k, a, b) \geq \mathrm{Stability}(k+1, a, b)
\end{equation}

A latent is stable if its matched value is at least $0.5$, and we count the stable latents for every pair of seeds. The threshold is conservative: shifting a decoder-Jacobian column horizontally by one pixel lowers its squared cosine with the unshifted column to $0.25$ ($\lam = 0.1$) and $0.30$ ($\lam = 1$) on average (10--90\,\% range about $0.01$--$0.65$), and a Gaussian blur with a standard deviation of one pixel lowers it to $0.69$ and $0.71$. These references were measured with the raw squared cosine on the $300$ highest-entropy columns of one model (seed 101) per $\lam$ from an earlier 200-epoch recipe with batch size 128, and not repeated on the models of this paper.

\textbf{Why the polar factor.} The raw squared cosine also finds no stable latents between random orthonormal frames, but it is only meaningful between frames that are themselves nearly orthogonal. When the compared frames are not orthogonal, it is inflated by the variance spectrum. For a randomly rotated copy $\A\boldsymbol{R}$ of the noisy-PCA decoder $\A$, every column is dominated by the same few high-variance directions, and two of its columns have $\mathbb{E}\cos^2 \approx \sum_i \lambda_i^2 / (\sum_i \lambda_i)^2 = 1/\mathrm{PR}$, where $\lambda_i$ are the PCA eigenvalues and PR is their participation ratio. For CelebA-28 (RGB), $\mathrm{PR} = 11.9$ predicts $0.084$ (measured: $0.063$), and matching two such rotated frames yields $338$ spurious matched values $\geq 0.5$, against $0$ for a flat spectrum. The same happens when the frames of trained flows are rotated. For two EOFlows ($\lam = 1$, seeds 103 and 105, 300 epochs, $256$ images, one draw of each rotation), rotating the latent frames of both models yields $527$ ``stable'' latents under the raw squared cosine ($682$ if both rotated frames come from the same model), and a trained frame against a rotated one yields $5$. Under the polar factor, all three comparisons yield $0$, with a mean matched value of $0.008$--$0.009$ over the best $100$ latents, the level of random orthonormal frames. This holds by construction: the polar factor of a randomly rotated frame is itself a random orthonormal frame. Trained EOFlow frames are nearly orthogonal (mean squared cosine between two columns of the same frame $3 \cdot 10^{-5}$), and on the trained pair the two metrics agree ($116$ raw, $121$ polar): the spectrum does not inflate the reported EOFlow counts, and the polar factor makes the comparison with rotated and random frames valid. It does not remove artifacts of the basis itself: architectures whose latent directions stay tied to the pixel axes reproduce those axes trivially (App.~\ref{app: architecture ablation}), so stability should be read together with the entropy order.

This protocol is used for every model in this paper: EOFlows at every $\lam$, NDFlows, the plain NF, the arms of the architecture ablation and the VAEs, for which the polar factor and the matching run over their $32$ latents.

\textbf{Coordinate grids, not latent codes.} A simpler comparison would correlate the latent codes of two runs. Such a correlation is a statistic between the two code distributions and cannot be evaluated at a single data point, whereas our similarity compares the two coordinate grids at every sample and takes the expectation afterwards. This is arguably stricter and resembles the nonlinear Amari distance of \citet{gresele2021independent}. Fig.~\ref{fig: latent pearson} shows both for the models and seed pairs of Fig.~\ref{fig: stability all models}, (a) our protocol and (b) the squared Pearson correlation $r^2$ between the latents of two runs on $10{,}000$ held-out images, Hungarian-matched in the same way. The difference is visible for NDFlows and FactorVAE. About six of their leading latents correlate almost perfectly across seeds ($r^2 \geq 0.95$), while the coordinate lines of their six best-matched latents agree only at $0.87$ to $0.54$ (NDFlow) and $0.89$ to $0.70$ (FactorVAE), and the NDFlow count of matches $\geq 0.5$ doubles from $7$ to $14.5$ under $r^2$. Two runs can thus agree on what a latent measures across the dataset and still move through data space in different directions at a given sample when it changes, since the coordinate line $\C_i$ also depends on all other latents, which are held fixed along it. The $r^2$ matrix is moreover not doubly stochastic (median row sums between $0.77$ and $2.14$), so a value $\geq 0.5$ does not single out one partner, and $r^2$ drops under any monotone non-linear reparametrization of a latent, which leaves the coordinate line itself unchanged.

\begin{figure}[h]
    \centering
    \includegraphics[width=\linewidth]{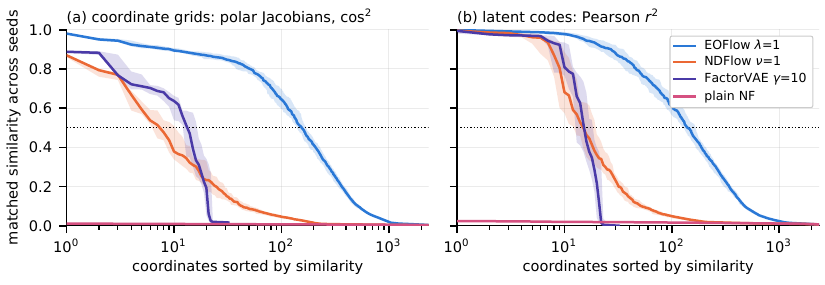}
    \caption{Stability of coordinate grids (a, the protocol of Fig.~\ref{fig: stability all models}c) and correlation of latent codes (b, squared Pearson correlation on $10{,}000$ held-out images, Hungarian-matched), sorted per seed pair (median, min--max band); same models and seed pairs as Fig.~\ref{fig: stability all models}.}
    \label{fig: latent pearson}
\end{figure}

Fig.~\ref{fig: stability all models} shows the sorted matched values for every seed pair of the best model of each class, App.~\ref{app: stable features} the $100$ most stable EOFlow features in all six seeds, and App.~\ref{app: CelebA archetypes} the archetypes of every configuration.

\subsubsection{Linear geometry of the detail coordinates}\label{app: line straightness}
Sec.~\ref{sec:ordering} argues that the bottleneck reconstruction (\ref{eq:autoencoder-like-reconstruction}) is an orthogonal projection onto the data manifold because the detail subspace encodes only isotropic noise and its coordinate lines $\C_i$ are straight. Such coordinates must have (a) the variance of the noise and (b) a rigid, linear geometry. We test both for all $D=2352$ coordinates of the EOFlow of the appendix figures ($\lam=1$, seed $103$, $300$ epochs) and, as a control, of the plain NF with the same seed, at the $512$ clean training images of App.~\ref{app: Jacobian-based metrics}. With the unit tangent $\boldsymbol{u}_i(\z) \coloneqq \J_i(\z)/|\J_i(\z)|_2$, the non-linearity of $\C_i$ is the variance of its tangent directions across images,
\begin{equation}
    \Expt{\z=\f(\x)}{\big|\boldsymbol{u}_i(\z)-\bar{\boldsymbol{u}}_i\big|_2^2} = 1 - \big|\bar{\boldsymbol{u}}_i\big|_2^2 , \qquad \bar{\boldsymbol{u}}_i \coloneqq \Expt{\z=\f(\x)}{\boldsymbol{u}_i(\z)} ,
\end{equation}
where the equality holds because $|\boldsymbol{u}_i|_2=1$. For small deviations it is the mean squared angle between the tangents and their common direction $\bar{\boldsymbol{u}}_i$. It is zero iff the lines $\C_i$ through all images are parallel, i.e.\ iff coordinate $i$ is one fixed direction in data space as in a linear map, it approaches one for unrelated directions, and it does not depend on the length of $\J_i$.
Averaged over the $300$ highest-entropy coordinates, it is the non-linearity of the coordinate grid in Fig.~\ref{fig: lambda dial}e, which falls monotonically along the sweep of $\lam$ from $0.82$ for the plain NF over $0.46$, $0.32$, $0.14$ and $0.045$ to $0.004$ for $\lam=0.01$, $0.1$, $1$, $3$ and $10$ (seed $103$ each), and is $0$ for PCA by construction.

Fig.~\ref{fig: line straightness}a shows it against the entropy rank. For the EOFlow it falls from $0.28$ for the ten highest-entropy coordinates to $0.008$ beyond rank $1000$: the tangents of these coordinates deviate by about $5^\circ$ from a common direction across all images. The plain NF has no such subspace; its coordinate lines point in largely unrelated directions at every rank ($0.75$--$0.88$ on average, never below $0.48$). Fig.~\ref{fig: line straightness}b shows the mean norm of the Jacobian columns. For the EOFlow it decreases with the rank and levels off at $0.052$ beyond rank $1000$ (minimum $0.050$), the noise level $\noisesig=0.05$, in agreement with the entropy spectrum of Fig.~\ref{fig: CelebA Manifold Entropy Spectra}a; the columns of the plain NF stay above $0.18$. The low-entropy coordinates of the EOFlow are thus fixed directions of the scale of the noise.
\begin{figure}[!h]
    \centering
    \includegraphics[width=0.85\linewidth]{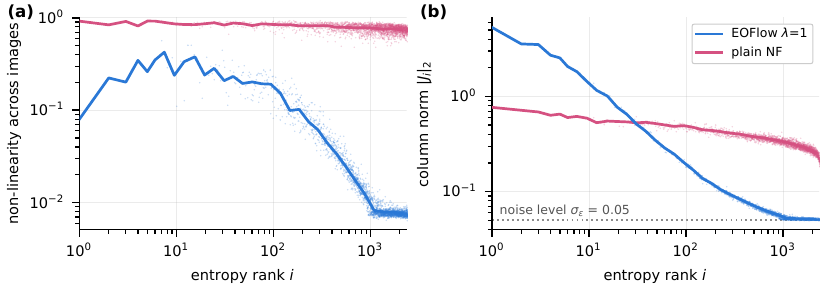}
    \caption{Geometry of the coordinates against the entropy rank $i$ for all $2352$ coordinates of the EOFlow ($\lam=1$) and of the plain NF (same seed, $300$ epochs each) at $512$ training images: (a) non-linearity across images (zero iff the lines $\C_i$ through all images are parallel), (b) mean norm of the Jacobian column $\J_i$. Dots are single coordinates, lines running medians over rank bins.}
    \label{fig: line straightness}
\end{figure}

\clearpage
\section{Additional results}\label{app: results}
This appendix collects ablations and further experiments (App.~\ref{app: architecture ablation}--\ref{app: rotated basis}), further evaluations of the CelebA models (App.~\ref{app: Rate-distortion plots} and~\ref{app: attributes}) and image galleries of the stable features across seeds and of the prior samples, reconstructions and archetypes of all models (App.~\ref{app: stable features}--\ref{app: CelebA archetypes}).

\subsection{Architecture ablation: which part of the mixing is needed?} \label{app: architecture ablation}

The flow alternates affine coupling blocks with dense $D \times D$ mixing layers. Since the Hadamard regularizer selects a latent basis, whereas maximum likelihood is indifferent to it, one may ask whether the learnable mixing layers are an arbitrary choice or an essential one. We retrain the CelebA model ($D = 2352$, 12 coupling blocks with two hidden layers of width 1024, $\lam = 1$, $\noisesig = 0.05$, batch size 512, 100 epochs, three seeds per arm, otherwise identical) in five arms that differ only in the mixing between the blocks (parameter counts as stored, App.~\ref{app: Normalizing Flow architecture}):
\begin{itemize}
    \item \textbf{ref}: all twelve $D \times D$ mixing matrices learnable (the model used throughout the paper), $245$M parameters;
    \item \textbf{last1}: eleven mixing matrices frozen at their random orthogonal initialization, only the final one learnable ($123$M);
    \item \textbf{frozen}: all mixing matrices frozen ($112$M);
    \item \textbf{nomix}: no mixing matrices at all, only the random permutations between the coupling blocks ($112$M);
    \item \textbf{noperm}: neither mixing matrices nor permutations ($112$M).
\end{itemize}
Table \ref{tab: architecture ablation} lists held-out likelihood, the Hadamard gap $\L_\text{H}/D$ (zero for an exactly orthogonal decoder Jacobian), the FID of prior samples through a bottleneck of $C=100$, and the number of stable features per pair of seeds under the protocol of Appendix \ref{app: Jacobian-based metrics}: decoder Jacobians at 256 fixed training images, polar (orthonormal) factors, the squared cosine similarity between the latent directions of two seeds averaged over the images, Hungarian matching of the latent dimensions, and a threshold of $0.5$ on the matched value. Fig.~\ref{fig: architecture ablation stability} shows the matched similarity of every latent dimension and Fig.~\ref{fig: architecture ablation archetypes} the leading archetypes of each arm.

\begin{table}[h]
    \centering
    \caption{Architecture ablation on CelebA, 100 epochs, three seeds per arm (means; two ref seeds reach NLL $-1.294$ and FID $27$, the third $-1.284$ and $33$; all other arms agree across seeds to the last digit). NLL: held-out negative log-likelihood per dimension of noise-inflated images; $\L_\text{H}/D$: Hadamard gap; FID at $C = 100$ for prior samples; stable: latent dimensions with matched squared cosine similarity $\geq 0.5$, given for the three seed pairs. The stable counts of the two mixing-free arms are set in italics: they count the pixel basis, not learned features (see text).}
    \label{tab: architecture ablation}
    \begin{tabular}{lcccc}
        \toprule
        arm & NLL & $\L_\text{H}/D$ & FID & stable features per seed pair \\
        \midrule
        ref (12 learnable mixings) & $-1.291$ & $0.012$ & $29$ & $54$ / $46$ / $40$ \\
        last1 (final mixing learnable) & $-1.285$ & $0.017$ & $28$ & $9$ / $11$ / $7$ \\
        frozen (no learnable mixing) & $-1.262$ & $0.030$ & $42$ & $2$ / $3$ / $5$ \\
        nomix (permutations only) & $-1.251$ & $0.027$ & $37$ & \textit{1827 / 1828 / 1824} \\
        noperm (couplings only) & $-1.247$ & $0.025$ & $38$ & \textit{1834 / 1838 / 1843} \\
        \bottomrule
    \end{tabular}
\end{table}

\textbf{Learnable mixings buy reproducibility, not likelihood.} Among the three arms that mix all dimensions (ref, last1, frozen), the ordering by every quality metric is the same, but the differences are modest: a single learnable mixing matrix recovers the likelihood and the sample quality of the full model to within $0.01$ nats per dimension and one FID point, and even the frozen arm orders its latent dimensions -- the leading hundred entropies are nearly identical across all arms (Fig.~\ref{fig: architecture ablation spectra}). Reproducibility across seeds, however, separates the arms by an order of magnitude: $46.7_{\pm7.0}$ stable features per seed pair with all mixings learnable, $9.0_{\pm2.0}$ with only the final one, $3.3_{\pm1.5}$ with none. The stable features of the full model are also the important ones: $39.0_{\pm5.0}$ of them lie among the $100$ highest-entropy latent dimensions. The leading archetypes of the last1 and frozen arms (Fig.~\ref{fig: architecture ablation archetypes}) show the same vocabulary as the full model -- illumination, background, skin and hair hue, glasses, mouth -- but a different selection and ordering in every seed. Maximum likelihood is indifferent to the latent basis and the Hadamard regularizer only acts through the decoder Jacobian; without enough learnable rotations, the regularizer cannot move the representation to the basis in which the coordinates are orthogonal, and each seed settles for a different compromise. The coupling blocks cannot provide these rotations, since each transforms only half of the coordinates element-wise, conditioned on the other half.

\textbf{Data-biased architectures are stable by design.} The two mixing-free arms appear to be the most reproducible of all, with $\approx 1830$ of $2352$ latent dimensions matched at squared cosine similarity $\geq 0.5$. This is an artifact of their bias towards the pixel axes: without a mixing layer, the coupling blocks are initialized as the identity map and never rotate the latent space away from the data axes, so the two seeds find the {\em same} axes trivially. In the noperm arm, $89\%$ of the latent dimensions are matched to the same index in the other seed. The matched similarity, plotted along the entropy order (Fig.~\ref{fig: architecture ablation stability}b), makes the mechanism visible: it is low ($\approx 0.3$) for the leading $\approx 300$ latent dimensions, which are global, seed-specific mixtures, and jumps to $\approx 0.9$ beyond, where every latent dimension is a single pixel. These dimensions form a plateau of $\approx 2000$ latent dimensions of identical entropy in the spectrum, which lies just above the noise floor -- the ``stability'' of these arms sits entirely in the noise dimensions, and only $4.0_{\pm1.0}$ (noperm) and $9.3_{\pm4.0}$ (nomix) of their hundred highest-entropy dimensions reproduce. The same caveat applies to any architecture whose inductive bias is tied to the data axes, such as convolutions: it will score higher on a stability metric by construction, without its features being more useful. This is why the paper's model is built from fully-connected couplings with learnable mixings, and why the stability metric should be read together with the entropy order.

\begin{figure}[h]
    \centering
    \includegraphics[width=\linewidth]{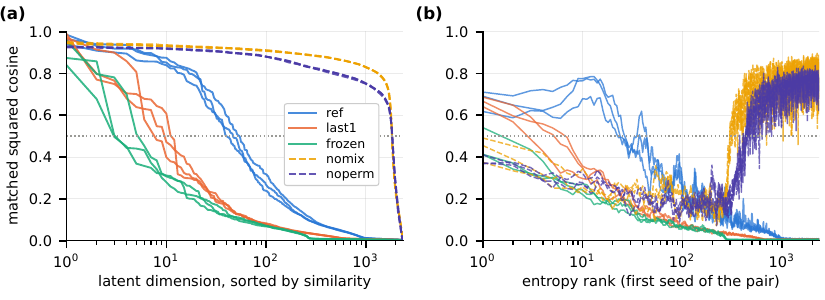}
    \caption{Stability of the latent dimensions across seeds for the five arms of the architecture ablation, three seed pairs per arm. \textbf{(a)} Matched squared cosine similarity of every latent dimension, sorted; the crossing with the dotted line at $0.5$ gives the number of stable features. \textbf{(b)} The same quantity along the entropy order of the first seed (mean over ten neighboring dimensions); the two mixing-free arms (dashed) are stable only beyond rank $\approx 300$, in the pixel-basis tail.}
    \label{fig: architecture ablation stability}
\end{figure}

\begin{figure}[p]
    \centering
    \includegraphics[width=\linewidth]{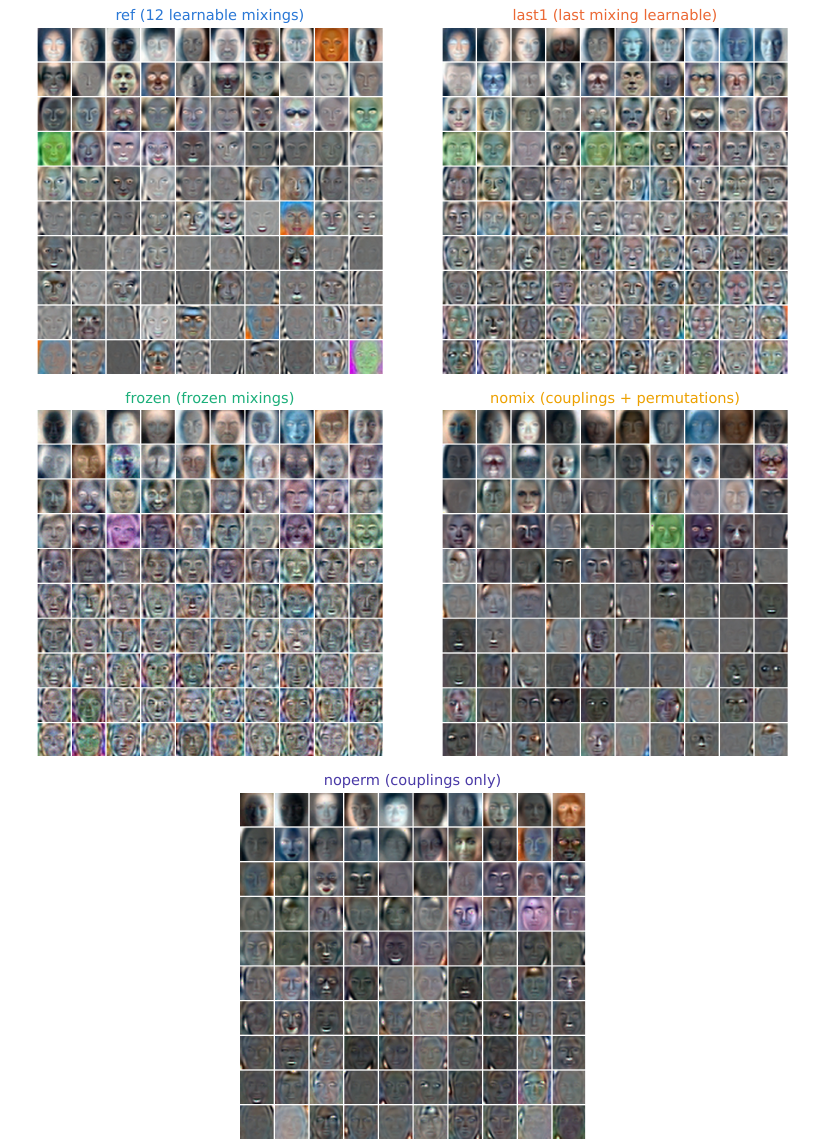}
    \caption{Archetypes of the five arms of the architecture ablation: the $100$ highest-entropy archetypes of one seed of each arm (difference images $\x_{i+4} - \x_{i-4}$ as in Appendix \ref{app: CelebA archetypes}, each normalized to $[0,1]$, entropy order row-major).}
    \label{fig: architecture ablation archetypes}
\end{figure}

\begin{figure}[h]
    \centering
    \includegraphics[width=\linewidth]{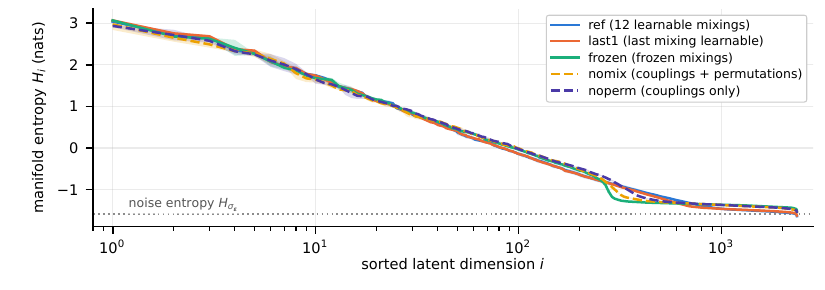}
    \caption{Entropy spectra of the five arms of the architecture ablation, mean $\pm$ standard deviation over three seeds; the two mixing-free arms are dashed, and the dotted line is the noise entropy $H_{\noisesig}$. The spectra coincide over the leading ranks, so the entropy ordering is comparable across the arms.}
    \label{fig: architecture ablation spectra}
\end{figure}

\subsection{A convolutional EOFlow on $64\times64$ images}\label{app: conv64}
\textbf{Architecture.} To test the regularizer beyond fully connected flows, we train a multiscale convolutional flow on CelebA at $64\times64$ pixels ($D=12288$). It has three stages, each a Haar transform followed by six GLOW coupling blocks with convolutional subnetworks ($3\times3$, $1\times1$ and $3\times3$ convolutions, $128$ hidden channels), every block followed by ActNorm and a learnable $1\times1$ channel mixing initialized at the identity. After the first two stages the nine detail channels leave the flow as latents ($9216$ and $2304$ dimensions) and the three low-pass channels continue; after the third stage the remaining $12\times8\times8=768$ dimensions enter a fully connected flow of the main-experiment design ($12$ coupling blocks with MLPs of width $1024$ and learnable mixings). Every latent keeps a standard normal prior, so its entropy can be read on its own; the model has $69.1$M parameters. We train it with $\noisesig=0.1$, batch size $128$ and schedule-free AdamW (learning rate $3\cdot10^{-4}$) for $100$ epochs, once with $\lam=1$ and once with $\lam=0$, one seed each.

\textbf{Result.} With the regularizer (Table~\ref{tab: conv64}), the manifold entropies span the range that PCA of the noisy data dictates ($4.81$ against $4.79$ nats), the Hadamard gap is below that of the $28\times28$ model ($0.008$ against $0.014$), and the $100$ leading coordinates denoise held-out images to about half the error of the noisy input. Without it, the flow reaches a better likelihood, but its Hadamard gap is $16$ times larger, its $768$ coarsest latents span only $1.2$ nats against $4.7$, and its $100$ leading coordinates denoise $13.5$ times worse than with the regularizer, seven times worse than the noisy input itself. Seed-to-seed stability was not measured on this model. Fig.~\ref{fig: conv64}a shows the difference: the EOFlow spectrum follows that of noisy PCA down to the noise entropy, whereas the unregularized flow spreads the entropy almost uniformly over its $768$ coarsest latents. Fig.~\ref{fig: conv64}b--d shows what the regularized model learns. Its $100$ highest-entropy coordinates form the same vocabulary as the fully connected model on $28\times28$ images (Fig.~\ref{fig: arch100 eo_lam1_e300}), in roughly the same order: brightness and background, illumination, skin and hair hue, sunglasses, mouth and eyes, and single color channels. Prior samples decoded through the $100$ leading coordinates are clean and diverse faces; decoded at full $D$ with one Tweedie step they keep more detail but also more artifacts. Their quality is well below that of dedicated image generators, which this experiment does not aim at: it shows that the recipe runs at this dimension and yields ordered, interpretable coordinates.

\begin{figure}[t]
    \centering
    \includegraphics[width=\linewidth]{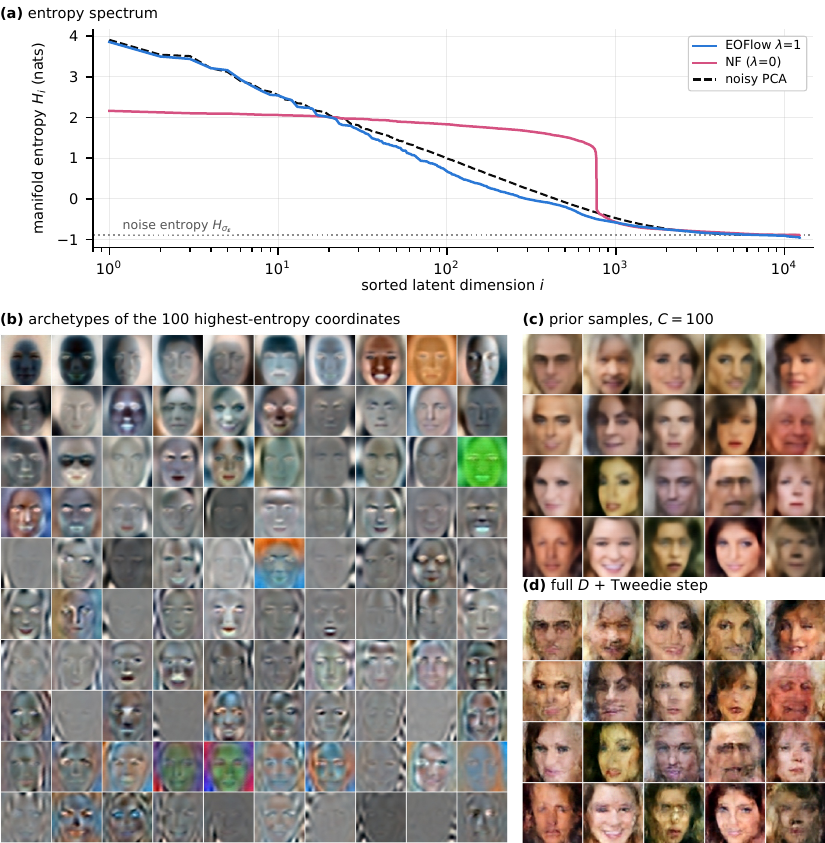}
    \caption{\textbf{Convolutional EOFlow on CelebA $64\times64$} ($D=12288$, $\noisesig=0.1$, $100$ epochs). \textbf{(a)} Sorted manifold entropy spectra of the EOFlow ($\lam=1$), the same flow trained by maximum likelihood alone (NF, $\lam=0$) and PCA of the noisy data (dashed), with the noise entropy $H_{\noisesig}$ (dotted). \textbf{(b)} Archetypes of the $100$ highest-entropy EOFlow coordinates in entropy order (difference images $g(+4\e_i)-g(-4\e_i)$ from the latent origin, each normalized). \textbf{(c)} EOFlow prior samples decoded through the $C=100$ leading coordinates and \textbf{(d)} the same latent draws decoded at full $D$ after one Tweedie step.}
    \label{fig: conv64}
\end{figure}

\begin{table}[h]
\centering\small
\caption{Convolutional EOFlow on CelebA $64\times64$ ($D=12288$, $\noisesig=0.1$, $100$ epochs, one seed each) and PCA of the noisy data. NLL: held-out negative log-likelihood per dimension; entropy range: spread of the manifold entropies $H_i$ over all latents and over the $768$ coarsest ones; MSE: denoising error of $2000$ held-out images through the $C$ leading coordinates (noisy input: $0.0100$).}
\label{tab: conv64}
\begin{tabular}{lccc}
\toprule
 & $\lam=1$ & $\lam=0$ & PCA \\
\midrule
NLL (nats/dim) & $-0.785$ & $-0.797$ & $-0.767$ \\
entropy range, all latents (nats) & $4.81$ & $3.05$ & $4.79$ \\
entropy range, $768$ coarsest (nats) & $4.70$ & $1.18$ & -- \\
$\L_\text{H}/D$ & $0.0079$ & $0.123$ & $0$ \\
MSE, $C=100$ & $0.0054$ & $0.0730$ & -- \\
MSE, $C=768$ & $0.0019$ & $0.0021$ & -- \\
\bottomrule
\end{tabular}
\end{table}

\subsection{Ground-truth factor benchmarks}\label{app: toy benchmarks}
\textbf{Data and models.} We train on two standard disentanglement benchmarks with known generative factors, both at $64\times64$ pixels: dSprites \citep{dsprites17}, $737{,}280$ binary images of a single shape with five factors (shape, scale, orientation and $x$, $y$ position), and Shapes3D \citep{3dshapes18}, $480{,}000$ RGB images of a rendered scene with six factors (floor, wall and object hue, scale, shape and orientation). Each dataset is split at random into $90\%$ training and $10\%$ held-out images. The EOFlow is a smaller version of the convolutional flow of App.~\ref{app: conv64}: two GLOW coupling blocks with $64$ hidden channels per stage, followed by a fully connected flow of four coupling blocks ($512$ hidden units) on the $D/16$ coarsest dimensions, with $4.3$M parameters on dSprites and $11.7$M on Shapes3D. We train four seeds per dataset with $\lam=0.1$ and three with $\lam=1$, all with $\noisesig=0.2$, batch size $256$ and schedule-free AdamW (learning rate $3\cdot10^{-4}$) for $100$ epochs.

\textbf{Entropy spectra.} On dSprites, all four seeds have a sharp knee after the fifth coordinate, the number of ground-truth factors (Fig.~\ref{fig: toy spectra}a): in three seeds the entropy drops there by $1.6$ to $2.3$ nats, in the fourth a sixth coordinate lies halfway down. The coordinates above the knee carry the $x$ and $y$ position, the scale and, on a pair of coordinates, the three-valued shape; orientation is isolated neither by EOFlows nor by any VAE baseline. On Shapes3D, the knee falls after $7$ to $8$ coordinates (Fig.~\ref{fig: toy spectra}b). The ten values of each hue factor are evenly spaced around the color circle, so each hue forms a closed loop and occupies a pair of coordinates; orientation takes one or two more, while scale and shape stay below the knee. PCA has no knee on either dataset: $190$ (dSprites) and $226$ (Shapes3D) of its components lie more than $0.35$ nats above the noise entropy, against at most $22$ and $11$ EOFlow coordinates.

\begin{figure}[t]
    \centering
    \includegraphics[width=\linewidth]{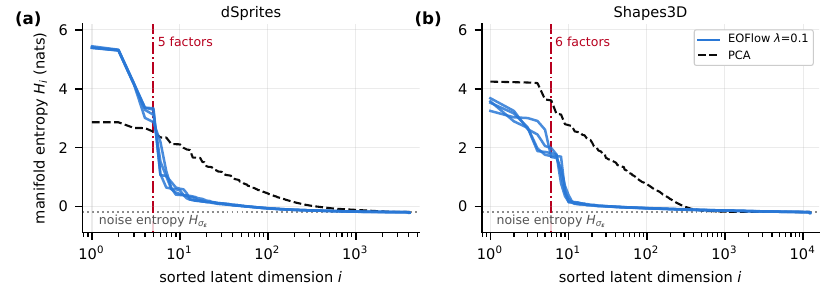}
    \caption{\textbf{Entropy spectra on ground-truth benchmarks} (four seeds each, $\noisesig=0.2$, $\lam=0.1$, $100$ epochs). Sorted manifold entropy spectra $H_i$ of EOFlows on \textbf{(a)} dSprites and \textbf{(b)} Shapes3D, the spectrum of PCA of the noisy data (dashed) and the noise entropy $H_{\noisesig}$ (dotted). The vertical line marks the number of ground-truth factors. EOFlows concentrate the entropy in a few coordinates and drop to the noise level at a sharp knee, which PCA lacks.}
    \label{fig: toy spectra}
\end{figure}

\textbf{Disentanglement scores.} Table~\ref{tab: toy benchmarks} compares EOFlows with noisy PCA and the four VAE objectives of App.~\ref{app: VAE baselines}, all trained on the same splits for $100$ epochs. The VAEs use the convolutional encoder and decoder of \citet{burgess2018understandingdisentanglingbetavae} with $10$ latents, a Bernoulli likelihood, batch size $128$ and the literature settings of each objective. All scores are computed on clean held-out images and on $10$ latents per model, the VAE posterior means and the $10$ highest-entropy coordinates of EOFlows and noisy PCA. MIG and DCI use $20{,}000$ images; mutual information is estimated from $20$ equal-mass bins and normalized by the factor entropy, and DCI uses a random forest. The FactorVAE score \citep{pmlr-v80-kim18b} follows the protocol of \citet{locatello2019challengingcommonassumptionsunsupervised}: each vote fixes one factor in a batch of $64$ held-out images and records the latent with the smallest normalized variance, latents with a global standard deviation below $0.05$ are pruned, and a majority-vote classifier is fit on $10{,}000$ votes and scored on $5{,}000$. On dSprites, EOFlows at $\lam=0.1$ reach the MIG of the best VAE, $\beta$-TCVAE, with a tenth of its seed-to-seed spread, but a lower DCI score, and $\lam=1$ is both weaker and seed-dependent ($0.14\pm0.10$). On Shapes3D the ranking turns around: every EOFlow stays below three VAEs in MIG, because each hue is spread evenly over a pair of coordinates, which MIG, the gap between the two most informative latents, counts as entangled, whereas DCI, which reads a whole importance matrix, ranks $\lam=1$ first of all models ($0.52\pm0.03$ against $0.42\pm0.06$ for $\beta$-TCVAE); at $\lam=1$ the flow also isolates shape and orientation instead of only orientation. The FactorVAE score, which a factor spread over a pair of coordinates does not penalize, ranks $\lam=0.1$ first on dSprites ($0.81\pm0.02$ against $0.78\pm0.07$ for $\beta$-TCVAE) and puts $\lam=1$ level with the best VAE on Shapes3D ($0.80\pm0.01$ against $0.83\pm0.15$ for $\beta$-VAE) at a fraction of its seed-to-seed spread.

\begin{table}[t]
\centering\small
\caption{\textbf{Disentanglement on dSprites and Shapes3D} ($64\times64$, $100$ epochs). Mean$_{\pm\text{sd}}$ over seeds. MIG \citep{NEURIPS2018_1ee3dfcd}, DCI disentanglement \citep{eastwood2018a} and the FactorVAE score \citep{pmlr-v80-kim18b} on $10$ latents per model (VAEs: all $10$; EOFlows and noisy PCA: the $10$ highest-entropy coordinates), from clean held-out images. VAEs at literature settings ($\beta$-VAE $\beta=4$, AnnealedVAE $C=25$ and $\gamma=1000$, FactorVAE $\gamma=10$, $\beta$-TCVAE $\beta=6$); EOFlows with $\noisesig=0.2$. Seeds (dSprites / Shapes3D): EOFlow $\lam=1$ 3 / 3, EOFlow $\lam=0.1$ 4 / 4, FactorVAE 3 / 3, $\beta$-TCVAE 6 / 4, $\beta$-VAE 4 / 3, AnnealedVAE 3 / 3.}
\label{tab: toy benchmarks}
\setlength{\aboverulesep}{0pt}\setlength{\belowrulesep}{0pt}\renewcommand{\arraystretch}{1.2}
\begin{tabular}{lccc|ccc}
 & \multicolumn{3}{c|}{dSprites (conv.)} & \multicolumn{3}{c}{Shapes3D (conv.)} \\ \cmidrule(lr){2-4}\cmidrule(lr){5-7}
model & MIG $\uparrow$ & DCI $\uparrow$ & FactorVAE $\uparrow$ & MIG $\uparrow$ & DCI $\uparrow$ & FactorVAE $\uparrow$ \\
\midrule
\textcolor[HTML]{2A78D6}{EOFlow $\lam=1$} & 0.14$_{\pm.10}$ & 0.07$_{\pm.04}$ & 0.65$_{\pm.13}$ & 0.15$_{\pm.03}$ & \textbf{0.52}$_{\pm.03}$ & 0.80$_{\pm.01}$ \\
\textcolor[HTML]{2A78D6}{EOFlow $\lam=0.1$} & \textbf{0.26}$_{\pm.01}$ & 0.19$_{\pm<.01}$ & \textbf{0.81}$_{\pm.02}$ & 0.17$_{\pm.07}$ & 0.41$_{\pm.09}$ & 0.70$_{\pm.02}$ \\
\midrule
\textcolor[HTML]{4A3AA7}{FactorVAE} & 0.15$_{\pm.07}$ & 0.14$_{\pm.03}$ & 0.76$_{\pm.09}$ & 0.13$_{\pm.03}$ & 0.14$_{\pm.05}$ & 0.73$_{\pm.04}$ \\
\textcolor[HTML]{6A5CC6}{$\beta$-TCVAE} & \textbf{0.26}$_{\pm.09}$ & \textbf{0.25}$_{\pm.03}$ & 0.78$_{\pm.07}$ & 0.31$_{\pm.12}$ & 0.42$_{\pm.06}$ & 0.71$_{\pm.09}$ \\
\textcolor[HTML]{8A7FDF}{$\beta$-VAE} & 0.06$_{\pm.02}$ & 0.06$_{\pm.01}$ & 0.64$_{\pm.06}$ & 0.22$_{\pm.07}$ & 0.25$_{\pm.08}$ & \textbf{0.83}$_{\pm.15}$ \\
\textcolor[HTML]{A79FEC}{AnnealedVAE} & 0.06$_{\pm.05}$ & 0.04$_{\pm.03}$ & 0.62$_{\pm.15}$ & \textbf{0.35}$_{\pm.03}$ & 0.20$_{\pm<.01}$ & 0.68$_{\pm.02}$ \\
\midrule
noisy PCA & 0.03 & 0.04 & 0.48 & 0.02 & 0.22 & 0.63 \\
\end{tabular}
\end{table}

\subsection{Equivariance: training on shuffled pixels}\label{app: rotated basis}

The EOFlow objective is equivariant to orthogonal changes of the data basis (Sec.~\ref{sec:objective}). This section details the test of Sec.~\ref{sec: equivariance}: trained on CelebA with every image shuffled by the same fixed random permutation $\boldsymbol{P}$, so that no pixel neighborhood survives, an EOFlow learns the same features once they are shuffled back (Fig.~\ref{fig: equivariance}). The figure shows the shuffled seed $102$, whose partners of the displayed plain features (skin tone, eye gaze, eyeglasses) have a matched $\cos^2$ of $0.99$, $0.73$ and $0.84$; seed $101$ matches the eye gaze at only $0.39$.

We train the same EOFlow ($\lam=1$, $\noisesig=0.05$, $100$ epochs, batch size $128$) on $\boldsymbol{P}\X$ with four seeds ($101$--$104$), where $\boldsymbol{P}$ applies one fixed random permutation of the $784$ pixel positions to all three color channels, and compare with two plain runs of the same recipe (seeds $101$ and $102$). Everything the shuffled models produce in pixel space is mapped back with $\boldsymbol{P}^T$, which is exact for the stability protocol as well, since the polar factor of $\J\boldsymbol{P}^T$ is the polar factor of $\J$ times $\boldsymbol{P}^T$. The shuffled runs are twins of the plain ones: their five leading manifold entropies agree with those of the plain runs to within $0.1$ nats, and the mean ENZ of their $50$ leading coordinates ($0.34$--$0.36$) matches that of the plain runs ($0.34$ and $0.35$).

Measured with the protocol of App.~\ref{app: Stability}, a shuffled run reproduces $65.6_{\pm4.0}$ coordinates of a plain run (eight pairs, $57$--$69$), as many as the two plain runs reproduce of each other ($63$) and as the shuffled runs reproduce among themselves ($62.5_{\pm5.8}$, six pairs). A change of the pixel basis therefore costs nothing beyond the seed-to-seed variation, and the agreement does not depend on whether the two runs share a seed, as the permutation scrambles the model basis in which the seed lives. A Haar-random rotation $\boldsymbol{R}$, which mixes all pixels instead of moving them, is the harder variant of this test: two such runs reproduce $52.5$ coordinates of their plain twins ($55$ and $50$) and $45$ of each other, about $10$--$18$ fewer, which is consistent with the architecture being only approximately equivariant, since coupling partitions, permutations and ActNorm act on individual coordinates, while the objective is equivariant exactly.

Fig.~\ref{fig: arch100 b7} shows the $100$ highest-entropy archetypes of the shuffled seed $101$ after un-shuffling. They are the familiar global and local features, although the model never saw a pixel neighborhood: the sparsest matched coordinates reach an ENZ of $0.06$, among them the mouth at entropy rank $59$ (matched $\cos^2$ $0.86$, ENZ $0.06$) and the eyes at rank $48$ ($0.59$, ENZ $0.14$). In the basis the model was trained in, the same coordinates are scattered pixels (Fig.~\ref{fig: arch100 b7}, bottom). Since neither architecture nor objective favors local pixel neighborhoods, local features such as mouth and eyes must stem from the data distribution together with the orthogonality bias.

\begin{figure}[p]
    \centering
    \includegraphics[width=\linewidth]{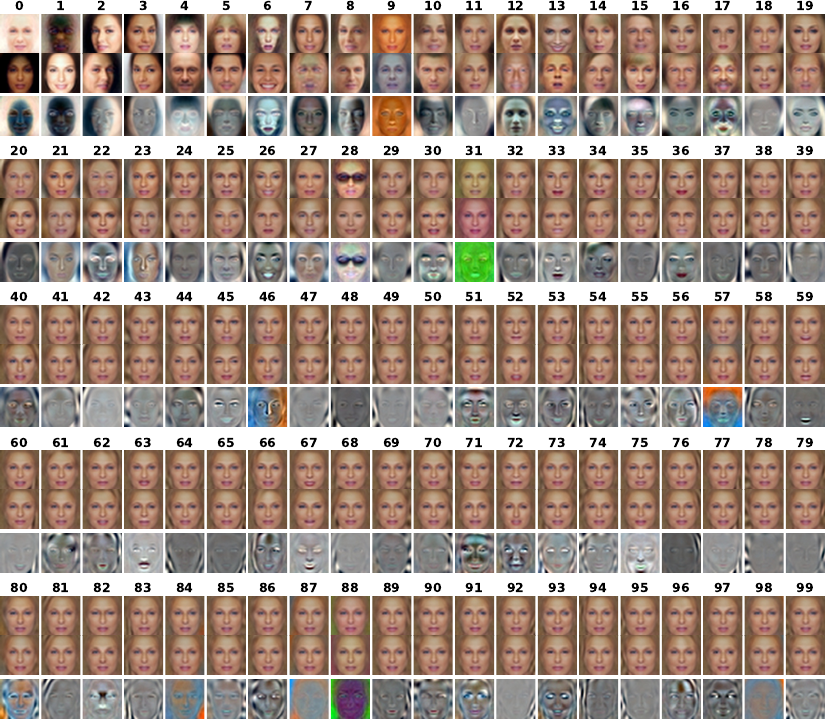}\\[4pt]
    \includegraphics[width=0.25\linewidth]{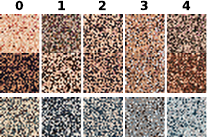}
    \caption{\textbf{Archetypes: EOFlow trained on shuffled pixels ($\lam=1$, $100$ ep, seed $101$).} Top: the $100$ highest-entropy latents after un-shuffling with $\boldsymbol{P}^T$, in the format of App.~\ref{app: CelebA archetypes} (number above each triplet is the entropy rank). Bottom: the five leading latents in the basis the model was trained in, each image normalized to $[0,1]$.}
    \label{fig: arch100 b7}
\end{figure}

\clearpage
\subsection{Rate-distortion plots} \label{app: Rate-distortion plots}
Fig.~\ref{fig: CelebA rate-distortion SSIM test}b reports the denoising error of held-out images reconstructed through the top-$C$ bottleneck (noisy input, clean target), together with the denoising autoencoders of App.~\ref{app: denoising AE}. Fig.~\ref{fig: sweep fid} reports, for the EOFlows and PCA, the FID of prior samples drawn through the bottleneck (no denoising step) as a function of $C$, and, as horizontal lines, the FID of samples from the full prior after one Tweedie step. The bottleneck-prior FID has a minimum at $C=100$--$200$ for every $\lam$ ($19$, $20$ and $24$ at $\lam=0.01$, $0.1$ and $1$), below the full-prior values, whereas PCA is best at $C=100$--$200$ at $38$. 

\textbf{The PCA spectrum.} In Fig.~\ref{fig: CelebA Manifold Entropy Spectra}a, PCA is treated as a linear NF with one latent per principal direction of the noise-inflated data, so its manifold entropies are $H_i=\tfrac12\log\big(2\pi e\,(s_i^2+\noisesig^2)\big)$, with $s_i^2$ the variance of the clean training data along the $i$-th principal direction. Among all distributions with a given variance, the Gaussian has the largest entropy. The data distribution along a principal direction therefore has at most the entropy of the corresponding PCA coordinate, and a model that fits the data better than the Gaussian needs less entropy in total. For EOFlows, whose orthogonal coordinates make the manifold entropies add up to the model entropy, this holds on CelebA: their manifold entropies sum to $-1.29$ nats per dimension against $-1.23$ for PCA, in line with the held-out likelihoods of Fig.~\ref{fig: lambda dial}a. The bound concerns the total and not single ranks, as the leading EOFlow coordinates carry more entropy than the leading principal directions, and for non-orthogonal models the manifold entropies do not add up to the model entropy.
\begin{figure}[!h]
    \centering
    \includegraphics[width=1\linewidth]{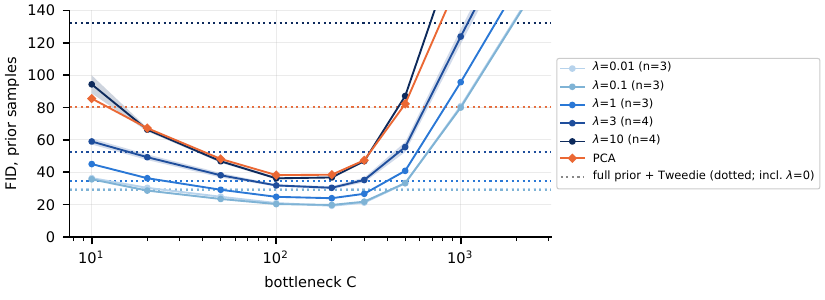}
    \caption{FID of prior samples through the top-$C$ bottleneck (no Tweedie step; mean $\pm$ sd over seeds) for EOFlows at $300$ epochs and PCA, and of full-prior samples after one Tweedie step (dotted lines, including the plain NF).}
    \label{fig: sweep fid}
\end{figure}
\clearpage
\subsection{Correlation with CelebA attributes}\label{app: attributes}
Table~\ref{tab: tad} reports the attribute-based disentanglement score TAD \citep{yeats2022nashae} of every model on $10{,}000$ validation images. For each of the $40$ binary CelebA attributes, every latent (flows and PCA: all $D$, VAEs: all $32$) is scored by its AUROC as a single-feature classifier of the attribute, irrespective of sign. An attribute counts as captured if its best latent reaches an AUROC of at least $0.75$ and it is nearly independent of the other attributes (mutual information with each of them at most $20\,\%$ of its entropy); TAD sums, over the captured attributes, the AUROC margin between the best and the second-best latent. It thus rewards attributes that are each isolated in a single latent. 
\begin{table}[t]
\centering\small
\caption{\textbf{Attribute-based disentanglement (TAD)} \citep{yeats2022nashae} on the CelebA validation split (higher is better; mean $\pm$ sd over seeds). Flows: same architecture and recipe; VAEs: best tuned cell of each family (App.~\ref{app: VAE baselines}); number of seeds in parentheses. The weakly regularized $\lam=0.1$ scores highest; $\lam=1$ ties the best VAE with far less seed scatter.}
\label{tab: tad}
\begin{tabular}{lcc}
\toprule
model & epochs & TAD \\
\midrule
plain NF ($\lam=0$) & 300 & 0.00 $\pm$ 0.00 (3) \\
EOFlow $\lam=0.01$ & 300 & 0.48 $\pm$ 0.21 (3) \\
EOFlow $\lam=0.1$ & 300 & 1.00 $\pm$ 0.01 (3) \\
EOFlow $\lam=1$ & 300 & 0.72 $\pm$ 0.01 (3) \\
EOFlow $\lam=3$ & 300 & 0.60 $\pm$ 0.05 (4) \\
EOFlow $\lam=10$ & 300 & 0.45 $\pm$ 0.03 (4) \\
PCA & -- & 0.52 (1) \\
\midrule
NDFlow $\nu=0.1$ & 300 & 0.58 $\pm$ 0.07 (3) \\
NDFlow $\nu=1$ & 300 & 0.24 $\pm$ 0.01 (3) \\
NDFlow $\nu=10$ & 300 & 0.28 $\pm$ 0.01 (3) \\
\midrule
FactorVAE $\gamma=10$ & 300 & 0.72 $\pm$ 0.12 (3) \\
$\beta$-TCVAE $\beta=4$ & 300 & 0.44 $\pm$ 0.01 (3) \\
$\beta$-VAE $\beta=4$ & 300 & 0.46 $\pm$ 0.03 (3) \\
AnnealedVAE $C=25$ & 100 & 0.40 $\pm$ 0.10 (5) \\
\bottomrule
\end{tabular}
\end{table}

\subsection{Stable features across seeds}\label{app: stable features}
Fig.~\ref{fig: stable features} extends the tiles of Fig.~\ref{fig: stability all models} from $7$ features in three seeds to the $100$ most stable features of EOFlow $\lam=1$ after $500$ epochs in all six seeds. With seed $105$ as the reference, as in Fig.~\ref{fig: stability all models}, every latent of the reference is matched to one latent of each other seed by the protocol of App.~\ref{app: Stability}, and the reference latents are ranked by their mean matched $\cos^2$ over the five other seeds; $124$ of them reach a mean of at least $0.5$. The matched latents show the same feature in every seed and sit at a similar place in the entropy order: their entropy ranks differ from that of the reference by a median of $4$. Stable features are not confined to the highest entropies: $57$ of the $100$ lie among the $100$ highest-entropy latents of the reference and $89$ among the $300$ highest, the deepest at rank $732$. Of the $100$, $55$ are matched at $\cos^2\geq0.5$ in all five other seeds and $89$ in at least four; the standard deviations mark the features that some seeds reproduce less well.

\subsection{Generating samples through a bottleneck}\label{app: prior samples}

As EOFlows learn a meaningful ordering of the inflated data from core to detail, it is possible to sample from the core prior and set the detail latents to their mean, i.e. to zero.
This artificial bottleneck resembles an injective model, which should only generate ``clean'' samples, i.e.\ with little to no noise:
\begin{equation}
    \x = \g\big([\z_{\le C}, \z_{>C}]\big) , \quad \z_{\le C} \sim \mathcal{N}(\mathbf{0}, \I_C), \ \z_{>C} = \mathbf{0}
\end{equation}

Fig.~\ref{fig: CelebA generated samples} decodes the same $20$ fixed prior codes with every model, one panel per model: from the full prior, after one Tweedie denoising step $\x+\noisesig^2\nabla\log q(\x)$ on these samples, and through bottlenecks that keep the $C=100$ or $C=300$ highest-entropy latents and set the others to zero. It covers the plain NF, EOFlows with $\lam\in\{0.01,0.1,1,3,10\}$, PCA and NDFlows with $\nu\in\{0.1,1,10\}$, all at $300$ epochs (seed $103$; NDFlow $\nu=10$: seed $106$), and the four VAEs of App.~\ref{app: VAE baselines}. For PCA, the full-prior samples are Gaussian with the covariance of the noisy data, the Tweedie step is their posterior mean, and a bottleneck samples only the top $C$ components. The VAEs decode the first $32$ entries of each code and give full-prior samples only: both bottlenecks exceed their $32$ latents, and their Bernoulli decoder defines no density in pixel space for a Tweedie step. The plain NF has no meaningful latent order, so its bottlenecks keep arbitrary latents and fail. For EOFlows and NDFlows, the bottleneck removes the noise without a denoising step because the discarded coordinates are the ones below the noise entropy; FID as a function of $C$ is given in App.~\ref{app: Rate-distortion plots} (Fig.~\ref{fig: sweep fid}).

\begin{figure}[p]
    \centering
    \includegraphics[width=\linewidth]{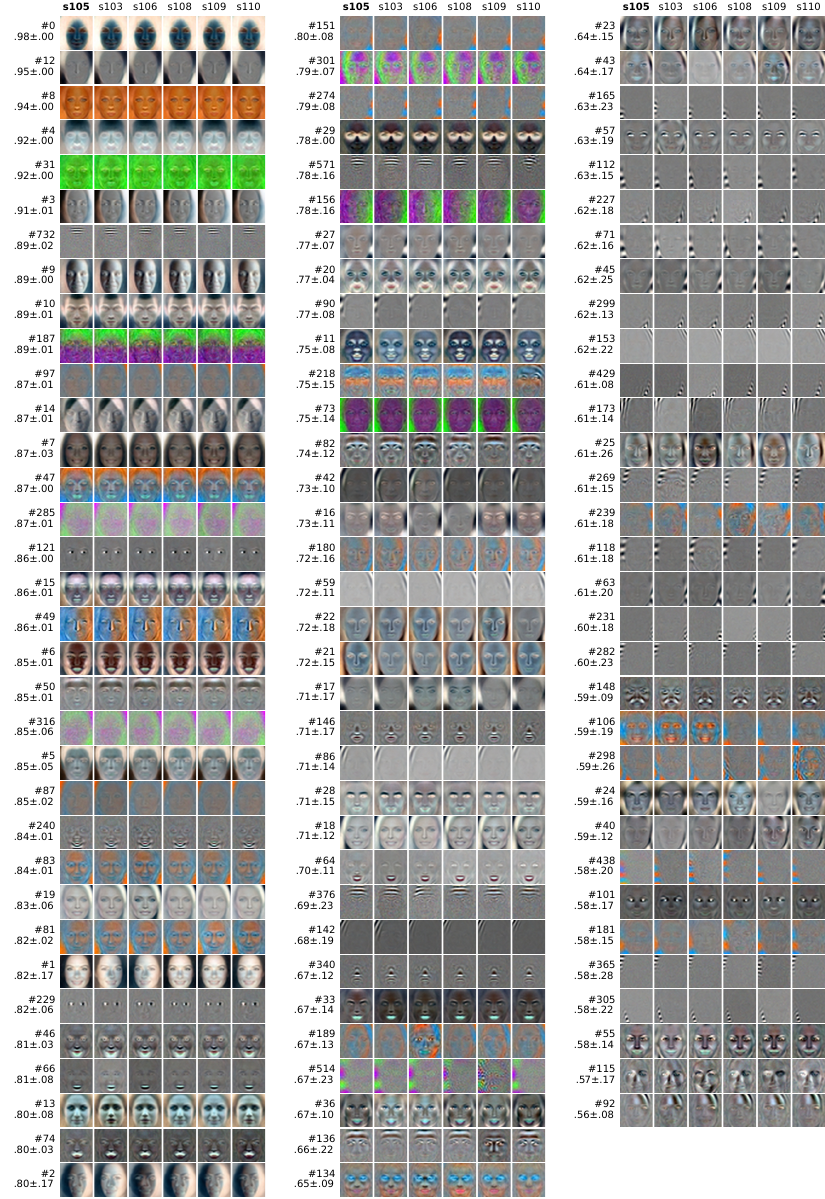}
    \caption{\textbf{The $100$ most stable features} of EOFlow $\lam=1$ after $500$ epochs in all six seeds. The latents of the reference seed $105$ are ranked by their mean matched stability over the five other seeds, column by column; each row shows the difference images $\x_{i+4}-\x_{i-4}$ of the matched latents in all six seeds (reference first, sign flips resolved). Left of each row: the entropy rank of the feature in the reference seed (\#) and the mean $\pm$ sd of its matched $\cos^2$ over the five seed pairs.}
    \label{fig: stable features}
\end{figure}

\begin{figure}[p]
    \centering
    \begin{subfigure}{\linewidth}\centering\includegraphics[width=\linewidth]{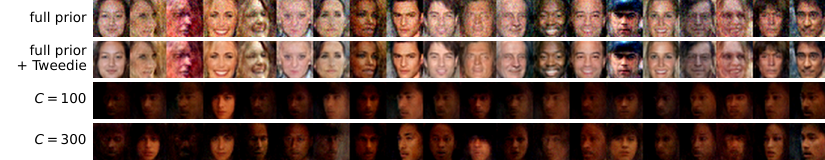}\caption{plain NF ($\lam=0$)}\end{subfigure}\\[2pt]
    \begin{subfigure}{\linewidth}\centering\includegraphics[width=\linewidth]{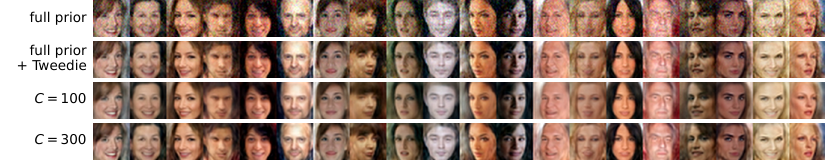}\caption{EOFlow $\lam=0.01$}\end{subfigure}\\[2pt]
    \begin{subfigure}{\linewidth}\centering\includegraphics[width=\linewidth]{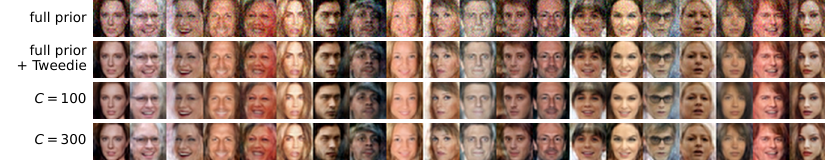}\caption{EOFlow $\lam=0.1$}\end{subfigure}\\[2pt]
    \begin{subfigure}{\linewidth}\centering\includegraphics[width=\linewidth]{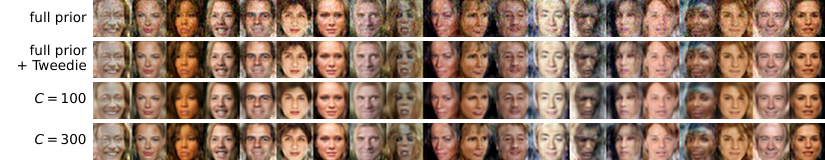}\caption{EOFlow $\lam=1$}\end{subfigure}\\[2pt]
    \begin{subfigure}{\linewidth}\centering\includegraphics[width=\linewidth]{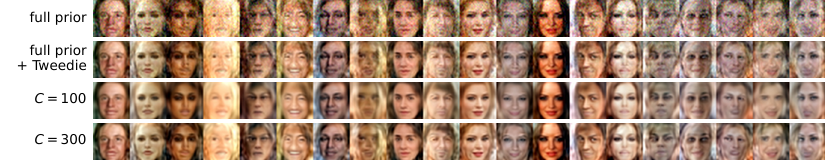}\caption{EOFlow $\lam=3$}\end{subfigure}\\[2pt]
    \begin{subfigure}{\linewidth}\centering\includegraphics[width=\linewidth]{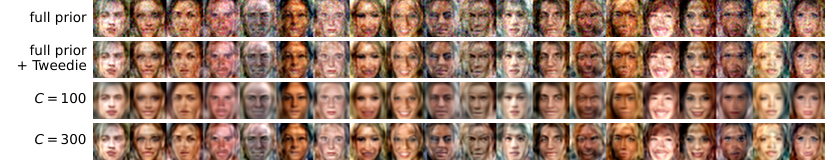}\caption{EOFlow $\lam=10$}\end{subfigure}
    \caption{\textbf{Prior samples} of every model from the same $20$ prior codes: decoded from the full prior, after one Tweedie step on these samples, and through bottlenecks that keep the $C=100$ or $C=300$ highest-entropy latents ($300$ epochs, seed $103$). The plain NF has no meaningful latent order, so its bottlenecks keep arbitrary latents. Continued on the next page.}
    \label{fig: CelebA generated samples}
\end{figure}
\begin{figure}[p]
    \ContinuedFloat
    \centering
    \begin{subfigure}{\linewidth}\centering\includegraphics[width=\linewidth]{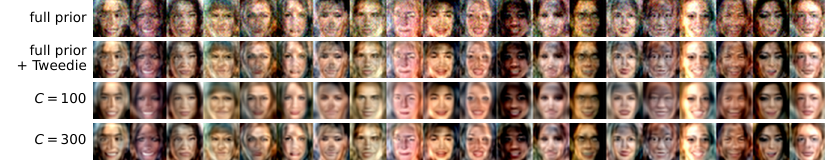}\caption{PCA of the noise-inflated data}\end{subfigure}\\[2pt]
    \begin{subfigure}{\linewidth}\centering\includegraphics[width=\linewidth]{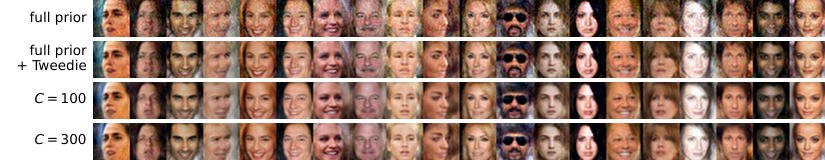}\caption{NDFlow $\nu=0.1$}\end{subfigure}\\[2pt]
    \begin{subfigure}{\linewidth}\centering\includegraphics[width=\linewidth]{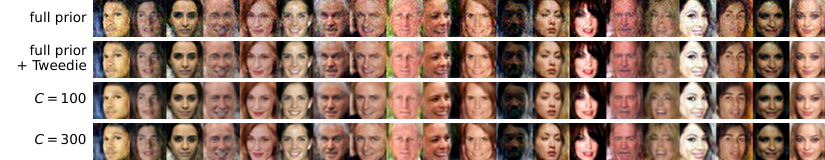}\caption{NDFlow $\nu=1$}\end{subfigure}\\[2pt]
    \begin{subfigure}{\linewidth}\centering\includegraphics[width=\linewidth]{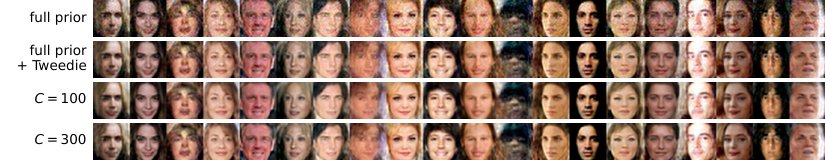}\caption{NDFlow $\nu=10$ (seed $106$)}\end{subfigure}\\[2pt]
    \begin{subfigure}{\linewidth}\centering\includegraphics[width=\linewidth]{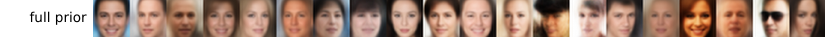}\caption{FactorVAE $\gamma=10$}\end{subfigure}\\[2pt]
    \begin{subfigure}{\linewidth}\centering\includegraphics[width=\linewidth]{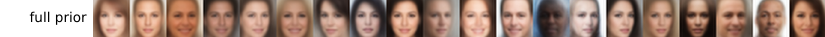}\caption{$\beta$-TCVAE $\beta=4$}\end{subfigure}\\[2pt]
    \begin{subfigure}{\linewidth}\centering\includegraphics[width=\linewidth]{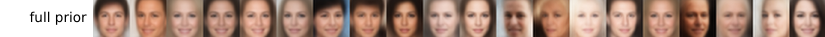}\caption{$\beta$-VAE $\beta=4$}\end{subfigure}\\[2pt]
    \begin{subfigure}{\linewidth}\centering\includegraphics[width=\linewidth]{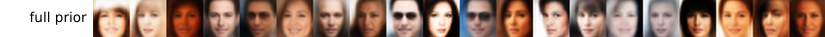}\caption{AnnealedVAE $C=25$ ($100$ epochs)}\end{subfigure}
    \caption[]{\textbf{Prior samples} (continued): PCA (Gaussian samples with the covariance of the noisy data, their posterior mean, and samples of the top $C$ components), NDFlows ($300$ epochs, seed $103$ unless noted) and the VAEs of App.~\ref{app: VAE baselines} ($300$ epochs, seed $101$ unless noted), which decode the first $32$ entries of each code and give full-prior samples only.}
\end{figure}

\subsection{Reconstructions} \label{app: Reconstruction}
Figs.~\ref{fig: CelebA reconstruction lam_ortho=0}--\ref{fig: CelebA reconstruction pca} show twelve held-out images, their noise-inflated version ($\noisesig=0.05$) and the reconstruction of the noisy version through bottlenecks $C\in\{5,10,20,50,100,500\}$, through the full latent space (which returns the noisy input) and after one Tweedie step, for the plain NF, EOFlows with $\lam\in\{0.01,0.1,1,3,10\}$ and PCA ($300$ epochs, seed $103$). The plain NF cannot reconstruct through any bottleneck, because its coordinates carry no ordering. EOFlows recover most of the image at $C=50$ and add detail up to $C\approx500$; larger bottlenecks reintroduce the noise. None of these flows sees a clean image or a reconstruction loss during training.

\begin{figure*}[!h]\centering\includegraphics[width=0.9\textwidth]{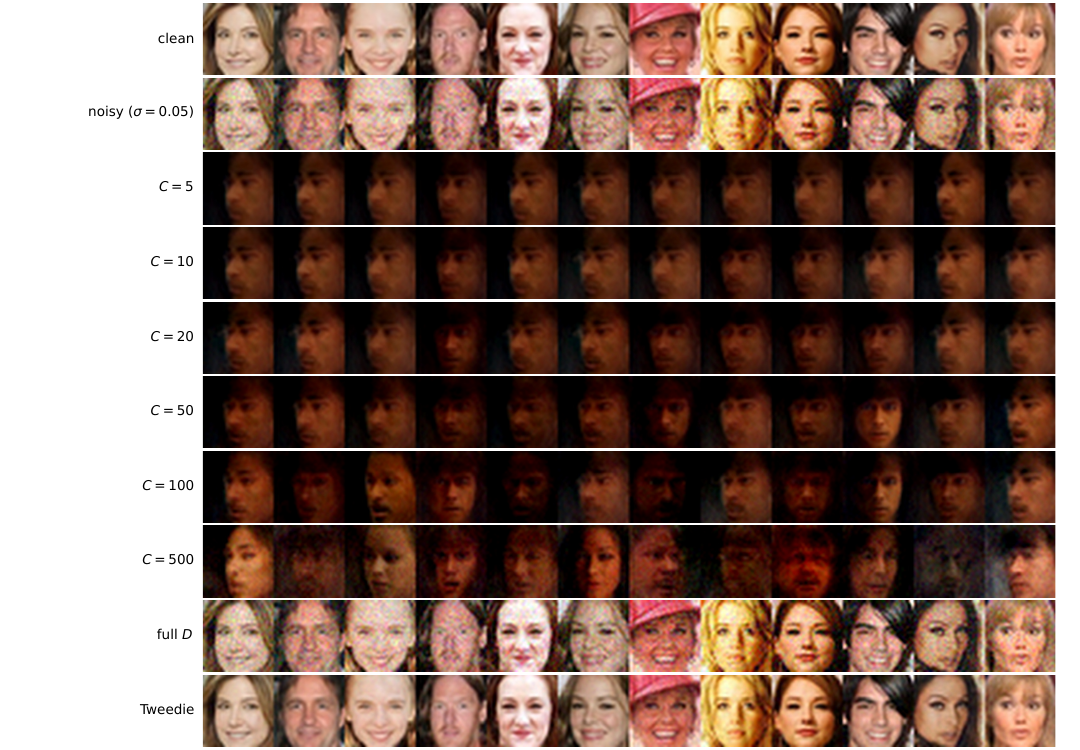}\caption{Reconstructions through the bottleneck, plain NF ($\lam=0$).}\label{fig: CelebA reconstruction lam_ortho=0}\end{figure*}
\begin{figure*}[!h]\centering\includegraphics[width=0.9\textwidth]{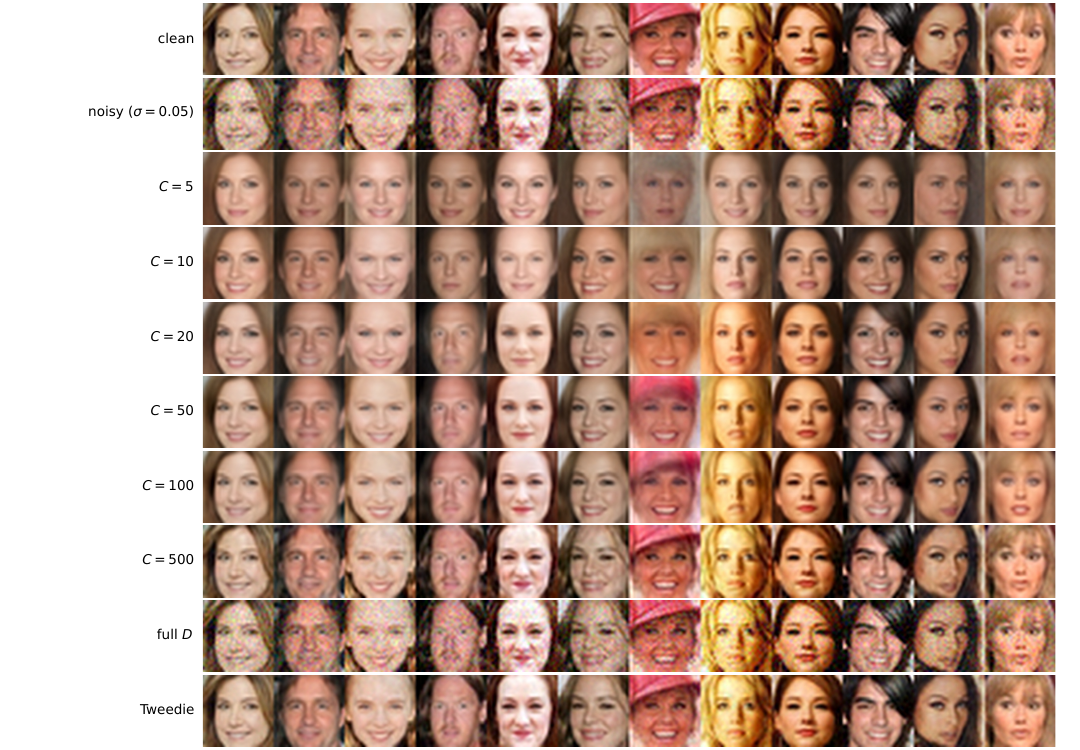}\caption{Reconstructions through the bottleneck, EOFlow $\lam=0.01$.}\label{fig: CelebA reconstruction lam 0.01}\end{figure*}
\begin{figure*}[!h]\centering\includegraphics[width=0.9\textwidth]{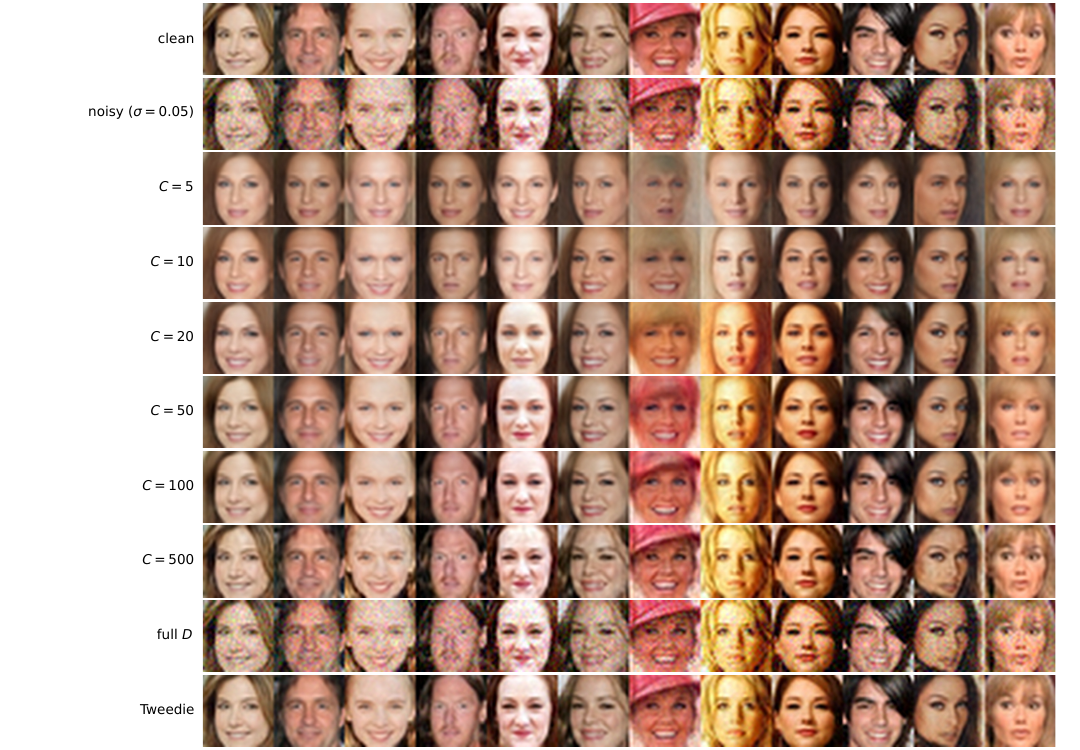}\caption{Reconstructions through the bottleneck, EOFlow $\lam=0.1$.}\label{fig: CelebA reconstruction lam_ortho=0.1}\end{figure*}
\begin{figure*}[!h]\centering\includegraphics[width=0.9\textwidth]{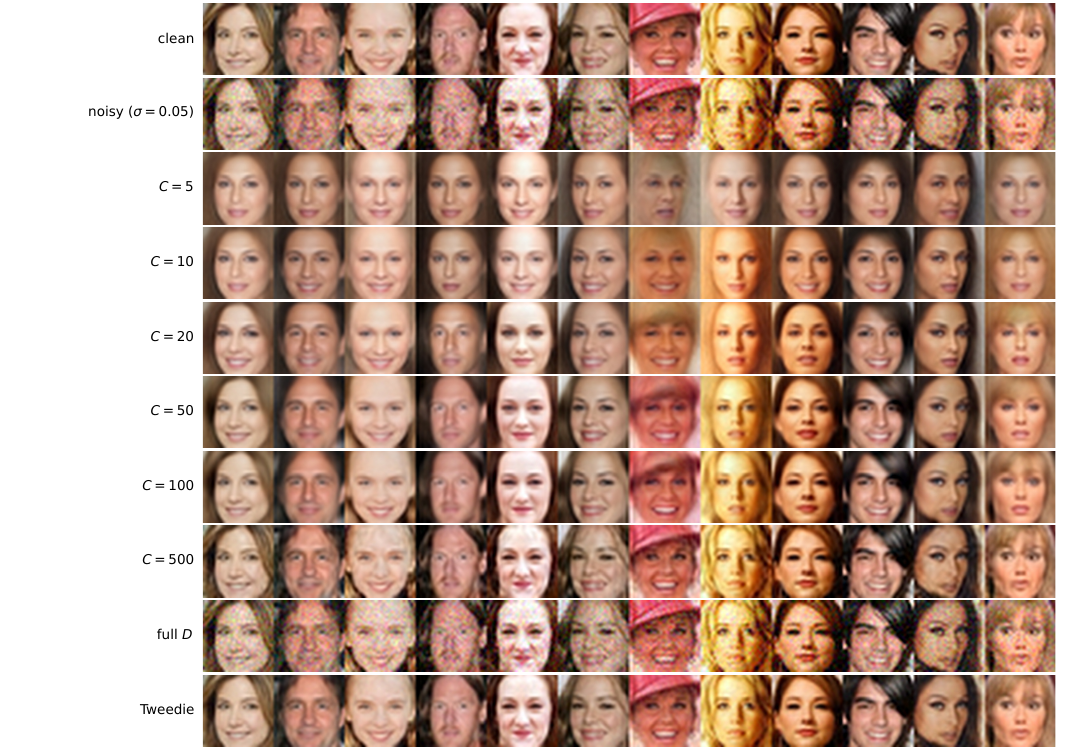}\caption{Reconstructions through the bottleneck, EOFlow $\lam=1$.}\label{fig: CelebA reconstruction lam_ortho=1.0}\end{figure*}
\begin{figure*}[!h]\centering\includegraphics[width=0.9\textwidth]{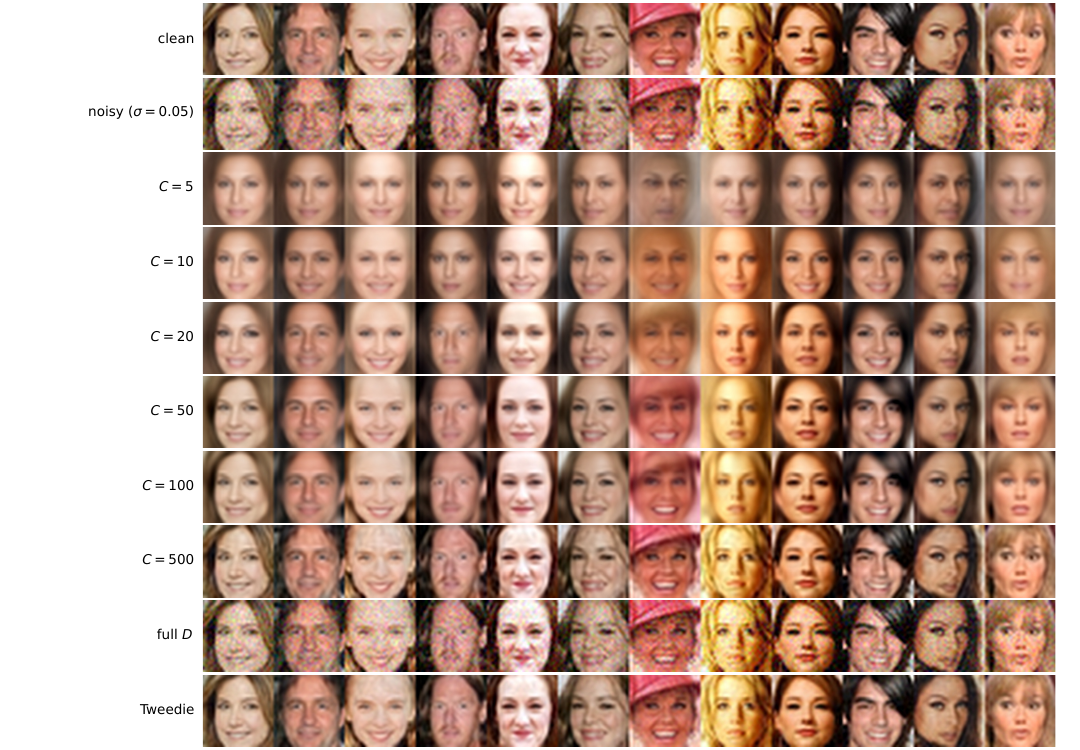}\caption{Reconstructions through the bottleneck, EOFlow $\lam=3$.}\label{fig: CelebA reconstruction lam_ortho=3}\end{figure*}
\begin{figure*}[!h]\centering\includegraphics[width=0.9\textwidth]{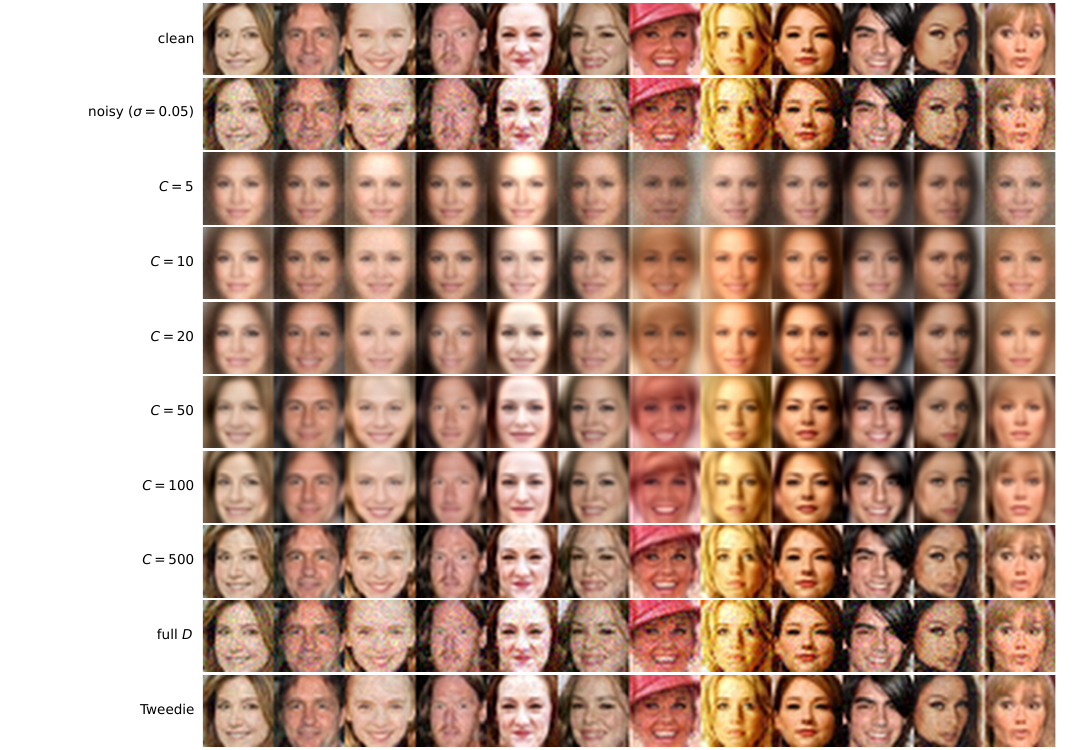}\caption{Reconstructions through the bottleneck, EOFlow $\lam=10$.}\label{fig: CelebA reconstruction lam_ortho=10}\end{figure*}
\begin{figure*}[!h]\centering\includegraphics[width=0.9\textwidth]{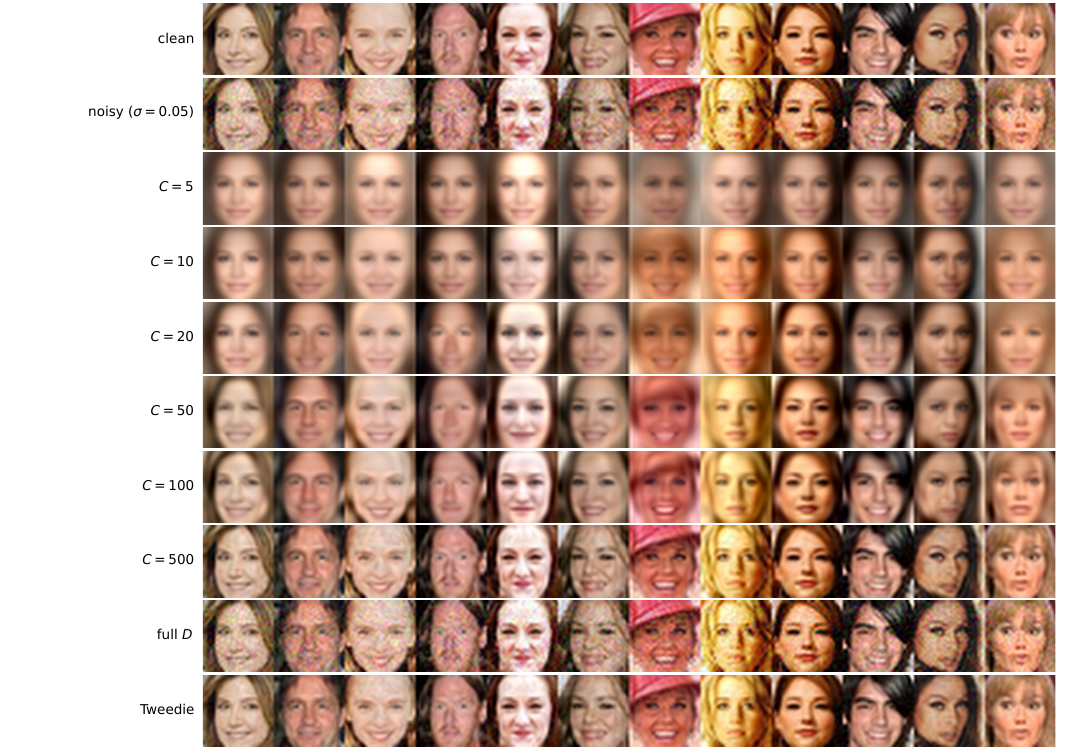}\caption{Reconstructions through the bottleneck, PCA (projection onto the leading $C$ components; Tweedie = posterior mean of the Gaussian model).}\label{fig: CelebA reconstruction pca}\end{figure*}
\clearpage
\clearpage
\subsection{Archetypes} \label{app: CelebA archetypes}

Archetypes visualize feature $i$ by moving from the latent origin $\z=\mathbf{0}$ along its coordinate line $\C_i$ by $4$ standard deviations in both directions,
\begin{equation}
    \x_{i+4} \coloneq \g([\z_{<i}=\mathbf{0}, z_i = +4, \z_{>i}=\mathbf{0}]), \qquad \x_{i-4} \coloneq \g([\z_{<i}=\mathbf{0}, z_i = -4, \z_{>i}=\mathbf{0}]),
\end{equation}
and by their difference image $\x_{i+4} - \x_{i-4}$, which is min--max normalized to $[0,1]$ per tile for display.

Figs.~\ref{fig: arch100 nf_e300}--\ref{fig: arch100 betaB_C25} show the $100$ leading archetypes, with both decoded extremes and their difference image, of one run per model configuration at the 300-epoch budget: the plain NF, EOFlows with $\lam\in\{0.01,0.1,1,3,10\}$ and PCA, then NDFlows with $\nu\in\{0.1,1,10\}$ and the VAEs of App.~\ref{app: VAE baselines} (AnnealedVAE at its $100$-epoch budget); VAEs have 32 latents and are ordered by their entropy $H_i$ (App.~\ref{app: VAE baselines}), PCA by eigenvalue with $\pm4$ standard deviations along each component. Fig.~\ref{fig: archetypes overview} shows hand-picked features of another seed ($105$) of the $\lam=1$ setting.

\begin{figure}[p]
    \centering
    \includegraphics[width=\linewidth]{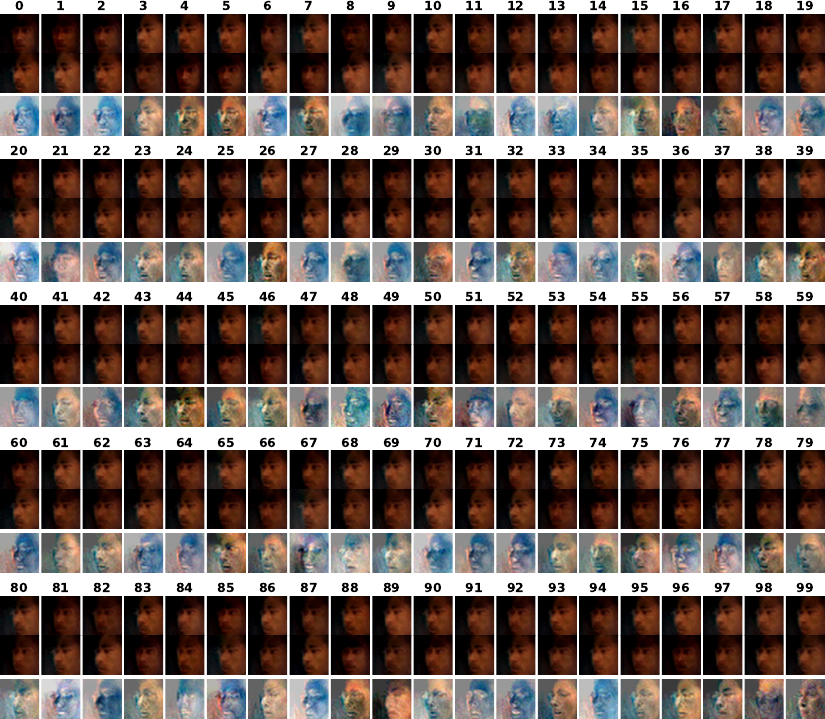}
    \caption{\textbf{Archetypes: plain NF, $\lam=0$ (300 ep, seed 103).} For every latent $i$ (number above each triplet, ordered by decreasing manifold entropy): decoded image at $z_i=+4$ (top), at $z_i=-4$ (middle) and their difference image, normalized per tile (bottom).}
    \label{fig: arch100 nf_e300}
\end{figure}

\begin{figure}[p]
    \centering
    \includegraphics[width=\linewidth]{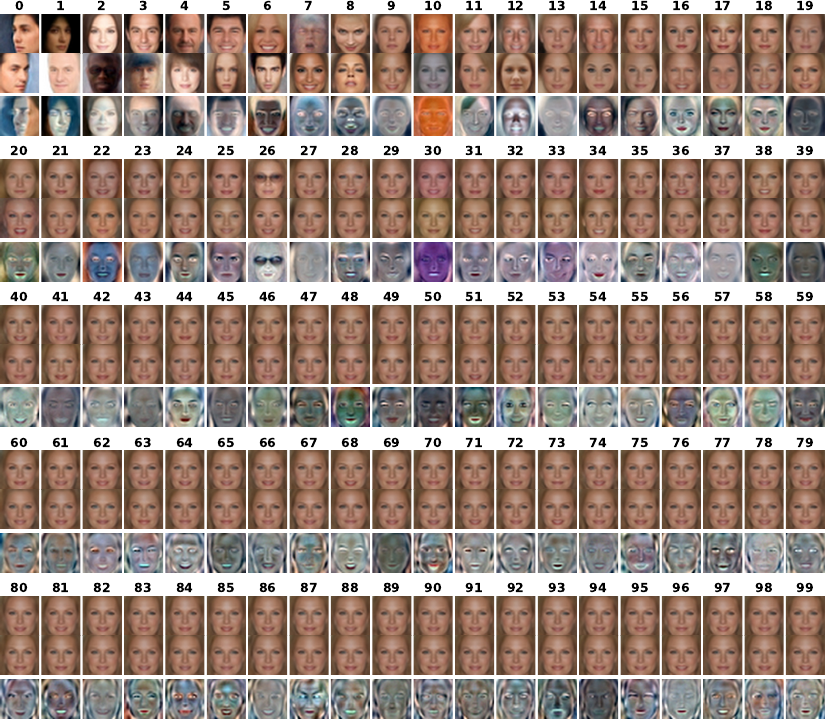}
    \caption{\textbf{Archetypes: EOFlow $\lam=0.01$ (300 ep, seed 103).} For every latent $i$ (number above each triplet, ordered by decreasing manifold entropy): decoded image at $z_i=+4$ (top), at $z_i=-4$ (middle) and their difference image, normalized per tile (bottom).}
    \label{fig: arch100 eo_lam0.01_e300}
\end{figure}

\begin{figure}[p]
    \centering
    \includegraphics[width=\linewidth]{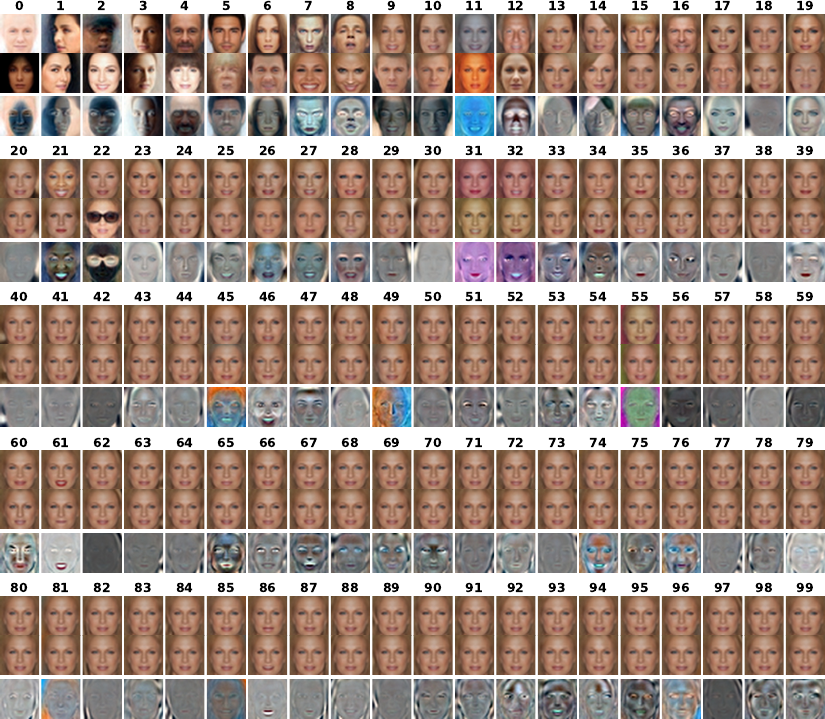}
    \caption{\textbf{Archetypes: EOFlow $\lam=0.1$ (300 ep, seed 103).} For every latent $i$ (number above each triplet, ordered by decreasing manifold entropy): decoded image at $z_i=+4$ (top), at $z_i=-4$ (middle) and their difference image, normalized per tile (bottom).}
    \label{fig: arch100 eo_lam0.1_e300}\label{fig: Combined archetypes lam_TC=0.1}
\end{figure}

\begin{figure}[p]
    \centering
    \includegraphics[width=\linewidth]{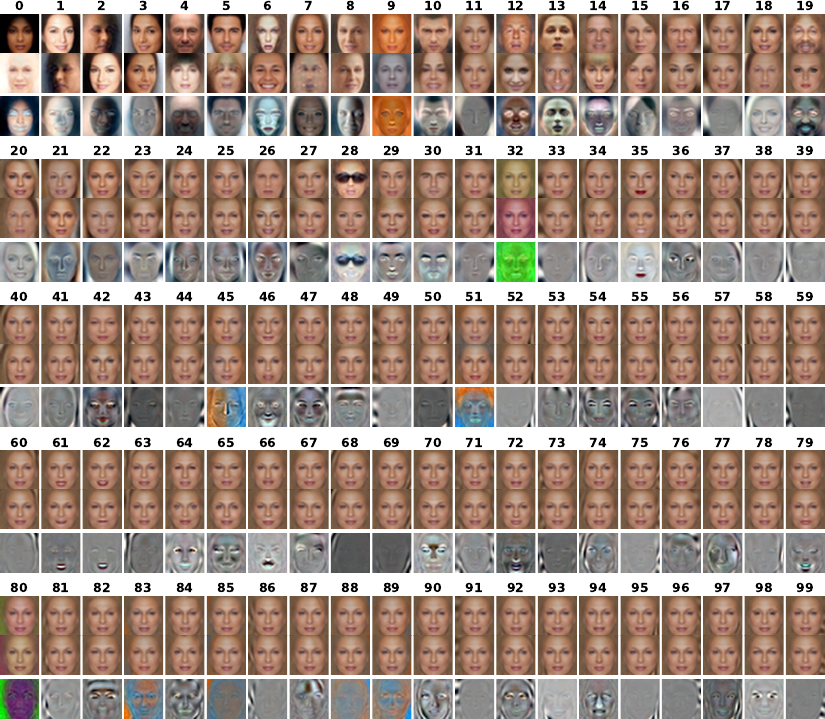}
    \caption{\textbf{Archetypes: EOFlow $\lam=1$ (300 ep, seed 103).} For every latent $i$ (number above each triplet, ordered by decreasing manifold entropy): decoded image at $z_i=+4$ (top), at $z_i=-4$ (middle) and their difference image, normalized per tile (bottom).}
    \label{fig: arch100 eo_lam1_e300}
\end{figure}

\begin{figure}[p]
    \centering
    \includegraphics[width=\linewidth]{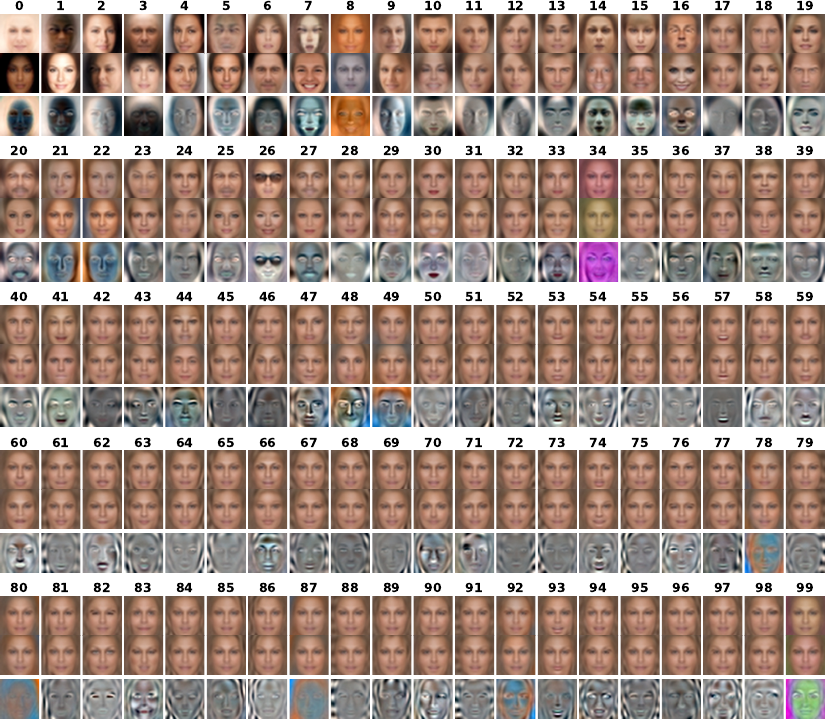}
    \caption{\textbf{Archetypes: EOFlow $\lam=3$ (300 ep, seed 103).} For every latent $i$ (number above each triplet, ordered by decreasing manifold entropy): decoded image at $z_i=+4$ (top), at $z_i=-4$ (middle) and their difference image, normalized per tile (bottom).}
    \label{fig: arch100 eo_lam3_e300}
\end{figure}

\begin{figure}[p]
    \centering
    \includegraphics[width=\linewidth]{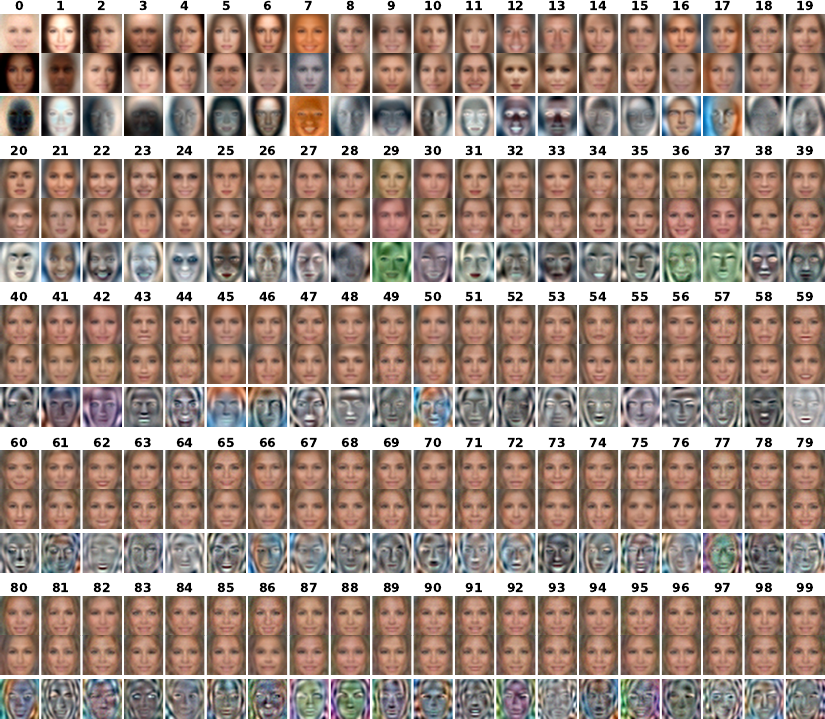}
    \caption{\textbf{Archetypes: EOFlow $\lam=10$ (300 ep, seed 103).} For every latent $i$ (number above each triplet, ordered by decreasing manifold entropy): decoded image at $z_i=+4$ (top), at $z_i=-4$ (middle) and their difference image, normalized per tile (bottom).}
    \label{fig: arch100 eo_lam10_e300}
\end{figure}

\begin{figure}[p]
    \centering
    \includegraphics[width=\linewidth]{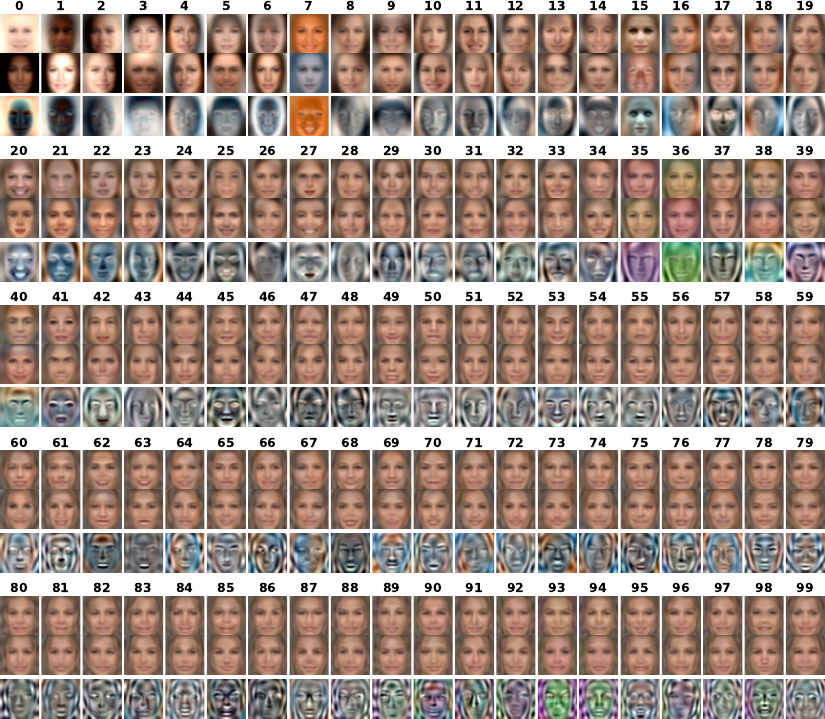}
    \caption{\textbf{Archetypes: PCA.} For every latent $i$ (number above each triplet, ordered by decreasing manifold entropy): decoded image at $z_i=+4$ (top), at $z_i=-4$ (middle) and their difference image, normalized per tile (bottom).}
    \label{fig: arch100 pca}
\end{figure}

\begin{figure}[p]
    \centering
    \includegraphics[width=\linewidth]{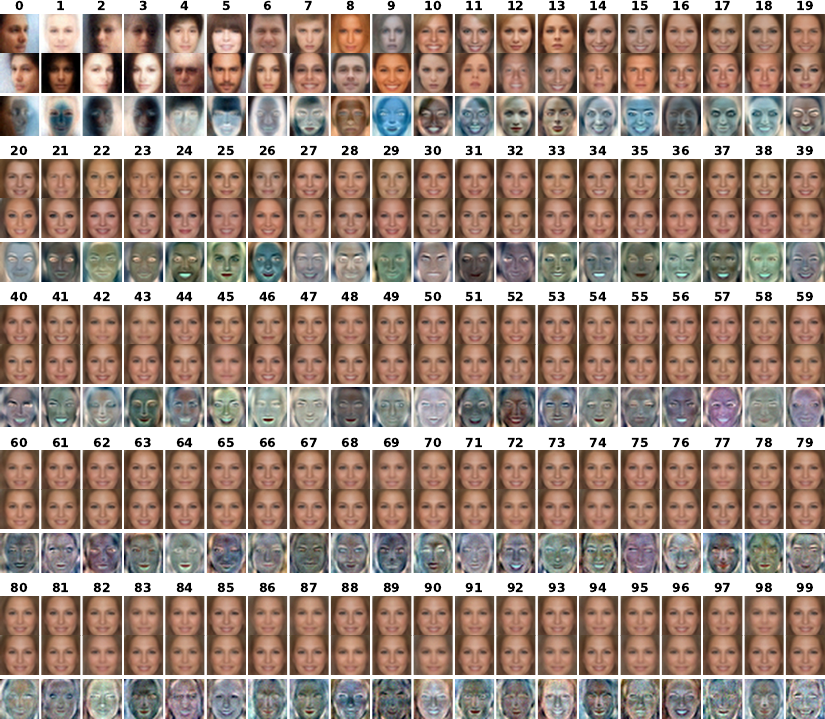}
    \caption{\textbf{Archetypes: NDFlow $\nu=0.1$ (300 ep, seed 103).} For every latent $i$ (number above each triplet, ordered by decreasing manifold entropy): decoded image at $z_i=+4$ (top), at $z_i=-4$ (middle) and their difference image, normalized per tile (bottom).}
    \label{fig: arch100 nd_nu0.1_e300}
\end{figure}

\begin{figure}[p]
    \centering
    \includegraphics[width=\linewidth]{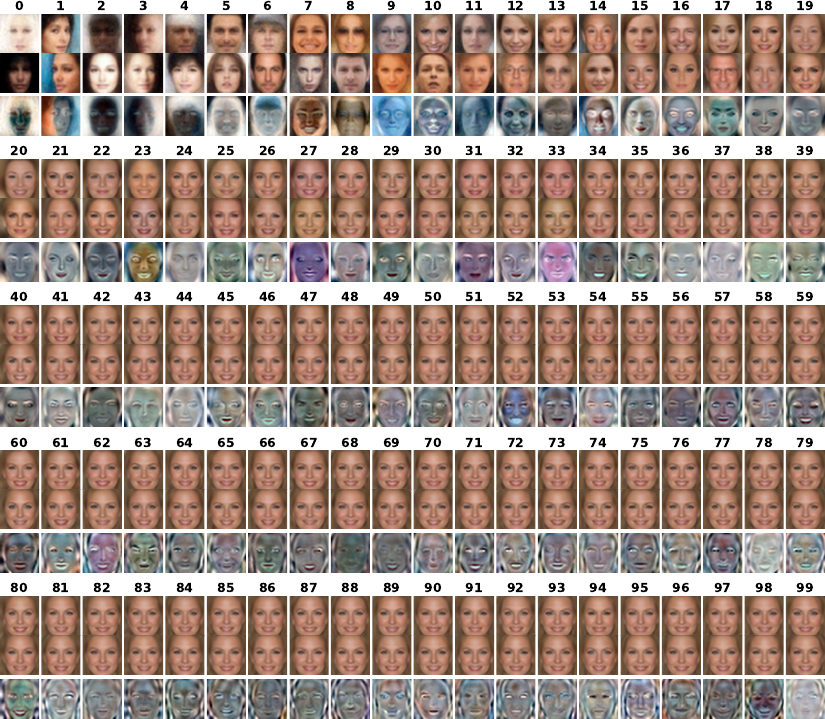}
    \caption{\textbf{Archetypes: NDFlow $\nu=1$ (300 ep, seed 103).} For every latent $i$ (number above each triplet, ordered by decreasing manifold entropy): decoded image at $z_i=+4$ (top), at $z_i=-4$ (middle) and their difference image, normalized per tile (bottom).}
    \label{fig: arch100 nd_nu1_e300}
\end{figure}

\begin{figure}[p]
    \centering
    \includegraphics[width=\linewidth]{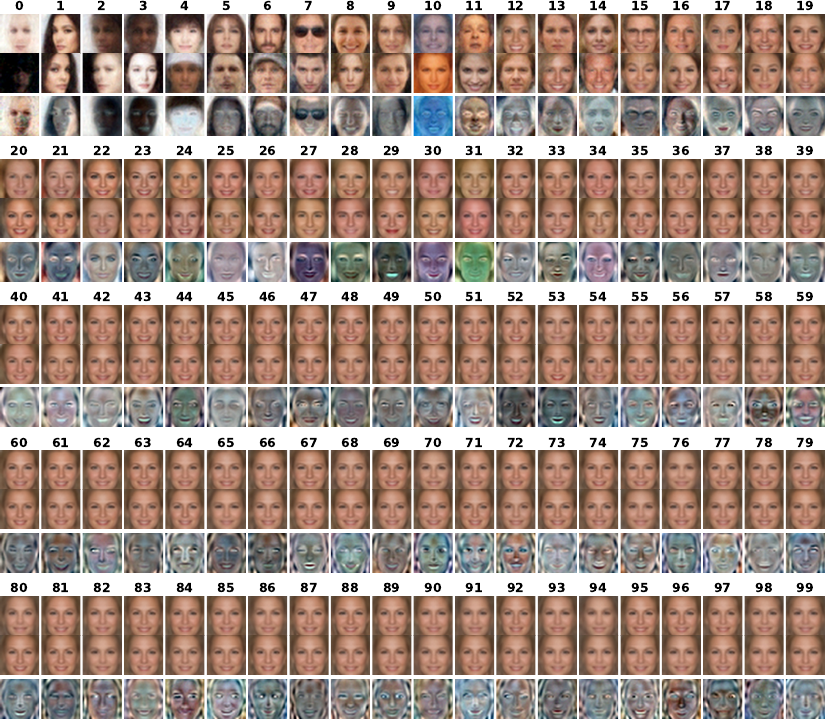}
    \caption{\textbf{Archetypes: NDFlow $\nu=10$ (300 ep, seed 106).} For every latent $i$ (number above each triplet, ordered by decreasing manifold entropy): decoded image at $z_i=+4$ (top), at $z_i=-4$ (middle) and their difference image, normalized per tile (bottom).}
    \label{fig: arch100 nd_nu10_e300}
\end{figure}

\begin{figure}[p]
    \centering
    \includegraphics[width=\linewidth]{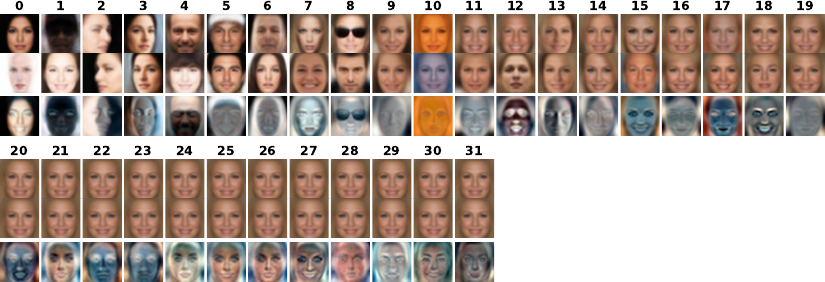}
    \caption{\textbf{Archetypes: FactorVAE $\gamma=10$ (300 ep, seed 101).} For every latent $i$ (number above each triplet, ordered by decreasing manifold entropy): decoded image at $z_i=+4$ (top), at $z_i=-4$ (middle) and their difference image, normalized per tile (bottom).}
    \label{fig: arch100 factor_G10_e300}
\end{figure}

\begin{figure}[p]
    \centering
    \includegraphics[width=\linewidth]{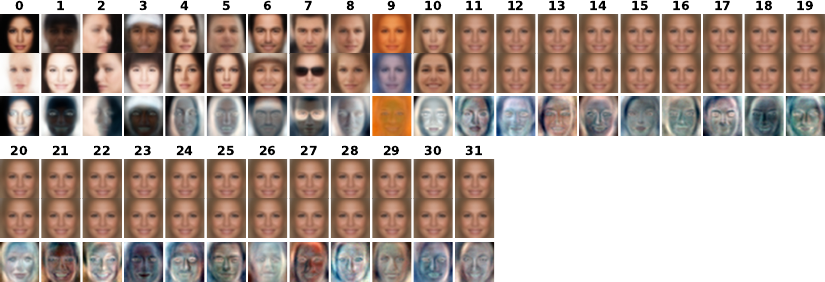}
    \caption{\textbf{Archetypes: $\beta$-TCVAE $\beta=4$ (300 ep, seed 101).} For every latent $i$ (number above each triplet, ordered by decreasing manifold entropy): decoded image at $z_i=+4$ (top), at $z_i=-4$ (middle) and their difference image, normalized per tile (bottom).}
    \label{fig: arch100 btcvae_B4_e300}
\end{figure}

\begin{figure}[p]
    \centering
    \includegraphics[width=\linewidth]{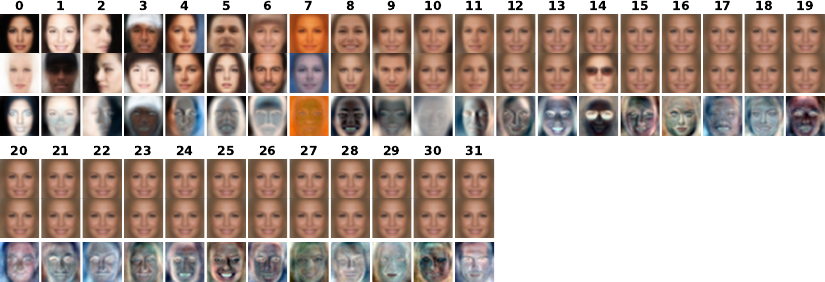}
    \caption{\textbf{Archetypes: $\beta$-VAE $\beta=4$ (300 ep, seed 101).} For every latent $i$ (number above each triplet, ordered by decreasing manifold entropy): decoded image at $z_i=+4$ (top), at $z_i=-4$ (middle) and their difference image, normalized per tile (bottom).}
    \label{fig: arch100 betaH_B4_e300}
\end{figure}

\begin{figure}[p]
    \centering
    \includegraphics[width=\linewidth]{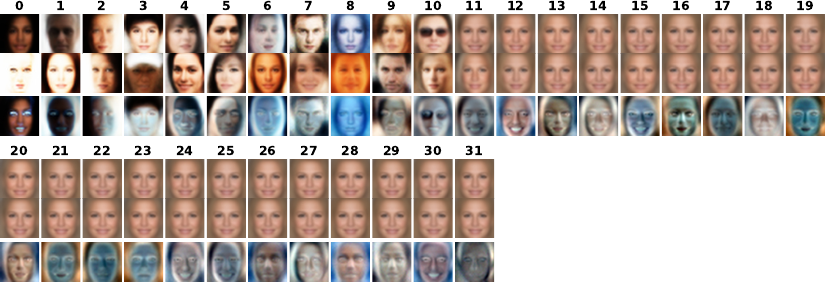}
    \caption{\textbf{Archetypes: AnnealedVAE $C=25$ (100 ep, seed 101).} For every latent $i$ (number above each triplet, ordered by decreasing manifold entropy): decoded image at $z_i=+4$ (top), at $z_i=-4$ (middle) and their difference image, normalized per tile (bottom).}
    \label{fig: arch100 betaB_C25}
\end{figure}

\clearpage

\clearpage

\newpage

\end{document}